\documentclass[twoside,11pt]{article}
\usepackage[abbrvbib, preprint]{jmlr2e}
\pdfoutput=1

\usepackage[utf8]{inputenc} 
\usepackage[T1]{fontenc} 
\usepackage{graphicx}
\usepackage{subfig}
\usepackage{amsmath,amssymb}
\usepackage{amsthm}
\usepackage{mathtools}
\usepackage{enumerate}
\usepackage{algorithm}
\usepackage[noend]{algpseudocode}
\usepackage{thmtools}
\usepackage{setspace}
\usepackage{natbib}
\usepackage{placeins}
\usepackage{tikz}
\usetikzlibrary{backgrounds}
\pgfdeclarelayer{bg1}
\pgfdeclarelayer{bg2}
\pgfsetlayers{bg2, bg1, main}
\usepackage{xcolor}
\usetikzlibrary{arrows.meta, positioning}

\usepackage{geometry} 

\DeclareMathOperator*{\argmax}{arg\,max}
\DeclareMathOperator*{\argmin}{arg\,min}
\DeclareMathAlphabet{\mathbbold}{U}{bbold}{m}{n}

\let\Algorithm\algorithm
\renewcommand\algorithm[1][]{\Algorithm[#1]\setstretch{1.2}}

\newtheorem{assumption}{Assumption}
\newtheorem{theorem}{Theorem}
\newtheorem{example}{Example}
\newtheorem{remark}{Remark}
\newtheorem{proposition}{Proposition}
\newtheorem{definition}{Definition}
\newtheorem{lemma}{Lemma}
\newtheorem{setting}{Setting}
\renewcommand\thmcontinues[1]{Continued}

\newcommand{\isd}{ISD}

\allowdisplaybreaks[1]

\begin{document}
\title{Invariant Subspace Decomposition}
    \author{Margherita Lazzaretto\textsuperscript{1}\thanks{E-mail: \texttt{mala@math.ku.dk}}%\footnotemark
, Jonas Peters\textsuperscript{2} \and Niklas Pfister\textsuperscript{1}\thanks{NP is now at Lakera in Z\"urich but most of this work was done while NP was at the University of Copenhagen.} \\
       \addr \textsuperscript{1}University of Copenhagen, \textsuperscript{2}ETH Z\"{u}rich}

\maketitle

\begin{abstract}%
We consider the task of predicting a response $Y$ from a set of covariates $X$ in settings where the conditional distribution of $Y$ given $X$ changes over time. For this to be feasible, assumptions on how the conditional distribution changes over time are required. Existing approaches assume, for example, that changes occur smoothly over time so that short-term prediction using only the recent past becomes feasible. To additionally exploit observations further in the past, we propose a novel invariance-based framework for linear conditionals, called Invariant Subspace Decomposition (\isd), that splits the conditional distribution into a time-invariant and a residual time-dependent component. 
    As we show, this decomposition can be utilized both for zero-shot and time-adaptation prediction tasks, that is, settings where either no or a small amount of training data is available at the time points we want to predict $Y$ at, respectively.
   We propose a practical estimation procedure, which automatically infers the decomposition using tools from approximate joint matrix diagonalization. Furthermore, we provide finite sample guarantees for the proposed estimator and demonstrate empirically that it indeed improves on approaches that do not use the additional invariant structure.
\end{abstract}

\section{Introduction} \label{sec:intro}

Many commonly studied systems evolve with time, giving rise to heterogeneous data. Indeed, changes over time in the data generating process lead to shifts in the observed data distribution, and make it in general impossible to find a fixed model that consistently describes the system over time.

In this work, we analyze the problem of estimating the relation between a response $Y_t$ and a set of covariates (or predictors) $X_t$, $t\in\mathbb{N}$, both of which are observed over some period of time, with the goal of predicting an unobserved response, when new observations of the covariates become available.
Two types of distribution shifts across time that can arise are (i) variations in the mechanism relating the response to the covariates and (ii) variations in the covariate distribution. Under distribution shifts, successfully predicting an unobserved response from covariates requires assumptions on how the underlying data generating model changes over time.

In regression settings, varying-coefficients models \citep[e.g.,][]{hastie1993varying, fan2008statistical} are a common approach to deal with dynamic system behaviours related to changes in the functional relationship between a response and some covariates. In these works, the model coefficients are usually assumed to change smoothly, and smoothing methods are used in their estimation.
Another approach for online estimation of smoothly changing parameters uses state-space models and filtering solutions \citep[see, for example,][]{durbin2012time}.
In general, smooth changes or other kinds of structured variations such as, for example, step-wise changes, allow us to learn a regression function using previous time points from the recent past. 
To exploit information further in the past, one can look for invariant patterns that persist through time. In this spirit, \citet{meinshausen2015maximin} define the maximin effect as a worst case optimal model that maintains good predictive properties at all observed times. 

Distribution shifts in time-series models are conceptually similar to distribution shifts across domains. In both cases a key problem is to find invariant parts of the observed data distribution and transfer it to the new domain or time-point. 
Unsupervised domain adaptation methods aim to predict an unobserved response in the target domain: for the problem to be tractable, they rely, for example, on the invariance of the conditional distribution of the response given the covariates \citep{sun2017correlation, zhao2019learning} or on independently changing factors of the joint distribution of the covariates and the response \citep{zhang2015multi, stojanov2021domain}. Solutions to the unsupervised domain adaptation problem are often based on aligning the source and target domains by minimizing the discrepancy either between the covariates' distributions \citep{zhang2015multi} or between the covariates' second order moments \citep{sun2017correlation}, or by finding low-dimensional invariant transformations of the variables whose distributions match in the domains of interest \citep{zhao2019learning, stojanov2021domain}.
Other unsupervised domain adaptation approaches such as the one proposed by \citet{bousmalis2016domain} rely on learning disentangled representations where disentangled mean that changes in one representation do not affect the remaining ones. %\rmn{The idea of learning disentangled representations}
This approach is also explored in non-stationary settings, when changes in the covariate distribution happen over time instead of across different domains, for example in the works by \citet{yao2022temporally, li2024and}. The goal of these works is not prediction, but rather the identification of generalizable latent states that could be used for arbitrary downstream tasks.

The field of causality widely explores the concept of invariant and independently changing mechanisms, too. In causal models, changes in the covariate distribution are modeled as interventions, and different interventional settings represent different environments. In this context, \citet{peters2016causal} propose to look for a set of stable causal predictors, that is, a set of covariates for which the conditional distribution of the response given such covariates remains invariant across different environments. 
The same idea is extended by \citet{pfister2019invariant} to a sequential setting in which the different environments are implicitly generated by changes of the covariate distribution over time.
Such invariances can then be used for prediction in previously unseen environments, too \citep[e.g.,][]{Rojas2016, Magliacane2018, Pfister2019stab}.
Other works on causal discovery from heterogeneous or nonstationary data, such as the ones by \citet{huang2020causal, gunther2023causal}, study settings in which the skeleton of the causal graph remains invariant through different contexts or time points, but the causal mechanisms are also allowed to change. Their goal is then to identify the causal graph on observed covariates. To detect whether a mechanism changes, the context or time is included as an additional variable in the graph. As the context or time can be assumed to be exogenous, this often leads to additional identification of causal edges in the graphs.
\citet{feng2022factored} develop and apply similar ideas in a reinforcement learning setting. Their interest is in learning the causal graph and improving the efficiency of non-stationarity adaptation by disentangling the system variations as latent parameters.

In this work, we build on invariance ideas both from the maximin framework proposed by \citet{meinshausen2015maximin} and from the covariates shifts literature, and exploit them for time adaptation. We propose an invariance-based framework, which we call invariant subspace decomposition (\isd), to estimate time-varying regression models in which both the covariate distribution and their relationship with a response can change. In particular, we consider a sequence of independent random vectors $(X_t, Y_t)_{t\in\mathbb{N}}\subseteq \mathbb{R}^p\times\mathbb{R}$ satisfying for all time points $t\in\mathbb{N}$ a linear model of the form
\begin{equation}
    Y_t = X_t^{\top}\gamma_{0,t} + \epsilon_t
    \label{eq:model_def}
\end{equation}
with $\mathbb{E}[\epsilon_t\mid X_t] =0$, where $\gamma_{0,t}$ is the (unknown) \emph{true time-varying parameter}. We allow the covariance matrices $\operatorname{Var}(X_t)$ to change over time, assuming that they are approximately constant in small time windows. Regarding changes in $\gamma_{0,t}$, we only need to assume smoothness during test time (if there is more than one test point)---we provide more details later in Remark~\ref{rem:time_changes}.
Having observed the $X_t$ and $Y_t$ up until some time point $n$, our goal is to learn $\gamma_{0,t^*}$ for some $t^* > n$, which we  then use to predict $Y_{t^*}$ from $X_{t^*}$. 
We propose to do so by considering the explained variance, which, for all $t$, is given for some function $f$ by $\operatorname{\Delta Var}_t(f)\coloneqq \operatorname{Var}(Y_t) - \operatorname{Var}(Y_t-f(X_t))$. Using the explained variance as the objective function provides an intuitive evaluation of the predictive quality of a function $f$: negative explained variance indicates in particular that using $f$ for prediction is harmful, in that it is worse than the best constant function. Under model~\eqref{eq:model_def}, $f$ is a linear function fully characterized by a parameter $\beta\in\mathbb{R}^p$, i.e., $f(X_t)=X_t^{\top}\beta$, and the true time-varying parameter can always be expressed as 
\begin{equation}
\label{eq:gamma0t}
\gamma_{0,t}=\argmax_{\beta\in\mathbb{R}^p}\operatorname{\Delta Var}_t(\beta).
\end{equation}
Key to \isd{} is the decomposition of the explained variance maximization into a time-invariant and a time-dependent part: we show in Theorem~\ref{thm:subspace_partition} that, under some assumptions, 
for all $t\in\mathbb{N}$, 
$\gamma_{0,t}$ can be expressed as
\begin{equation}
    \label{eq:lin_subspace_separation}
    \gamma_{0,t}=
    \underbrace{\argmax_{\beta\in\mathcal{S}^{\operatorname{inv}}}\overline{\operatorname{\Delta Var}}(\beta)}_{=:\beta^{\operatorname{inv}}} + 
    \underbrace{\argmax_{\beta\in\mathcal{S}^{\operatorname{res}}}\operatorname{\Delta Var}_t(\beta)}_{=:\delta^{\operatorname{res}}_t},
\end{equation}
where $\overline{\operatorname{\Delta Var}}(\beta)\coloneqq \frac{1}{n}\sum_{t=1}^n\operatorname{\Delta Var}_t(\beta)$, and $\mathcal{S}^{\operatorname{inv}},\mathcal{S}^{\operatorname{res}}\subseteq\mathbb{R}^p$ are two orthogonal linear subspaces, with $
\mathcal{S}^{\operatorname{inv}} \oplus
\mathcal{S}^{\operatorname{res}}= \mathbb{R}^p$,
assumed to be uniquely determined by the distribution of $(X_1, Y_1),\ldots,(X_n,Y_n)$. 
The decomposition of $\mathbb{R}^p$ into $\mathcal{S}^{\operatorname{inv}}$ and $\mathcal{S}^{\operatorname{res}}$ from observed data is achieved by exploiting ideas from independent subspace analysis \citep[e.g.,][]{gutch2012uniqueness}. 
Unlike works on latent representations, we do not reduce the overall dimensionality of the problem but partition the observed space into two lower dimensional subspaces, $\mathcal{S}^{\operatorname{inv}}$ and $\mathcal{S}^{\operatorname{res}}$. This in particular implies that the two sub-problems in \eqref{eq:lin_subspace_separation} each have less degrees of freedom than the original one. The first sub-problem can be solved using all available heterogeneous observations, which we call \emph{historical data}: its solution $\beta^{\operatorname{inv}}$ is a time-invariant linear parameter that partially describes the dependence of $Y$ on $X$ at all times.
The second sub-problem is a time adaptation problem that tunes the invariant component to a time point of interest: its solution $\delta^{\operatorname{res}}_t$ explains the residuals $Y_t - X_t^{\top}\beta^{\operatorname{inv}}$; the sum $\beta^{\operatorname{inv}}+\delta_t^{\operatorname{res}}$ gives an estimator for the time-varying linear parameter of interest $\gamma_{0,t}$.
In order to estimate the residual component $\delta_t^{\operatorname{res}}$, we assume the model is approximately stationary in a small time window preceding
$t$  and use this local subset of the data, which we call \emph{adaption data}, for estimation.
We distinguish two tasks: (i) the \emph{zero-shot task}, where adaption data is not available and we approximate $\gamma_{0,t}$ by $\beta^{\operatorname{inv}}$ and (ii) the \emph{time-adaption task}, where adaption data is available and we solve both sub-problems to approximate $\gamma_{0,t}$ by $\beta^{\operatorname{inv}}+\delta_t^{\operatorname{res}}$.
A two-dimensional example of $\gamma_{0,t}$ and its \isd{} estimates is shown in Figure~\ref{fig:2d_ex}. The same example is presented again in more detail in Example~\ref{ex:running_ex_2d} below and used as a running example throughout the paper.
 
 \begin{figure}[t]
    \centering
    \includegraphics[width=\textwidth]{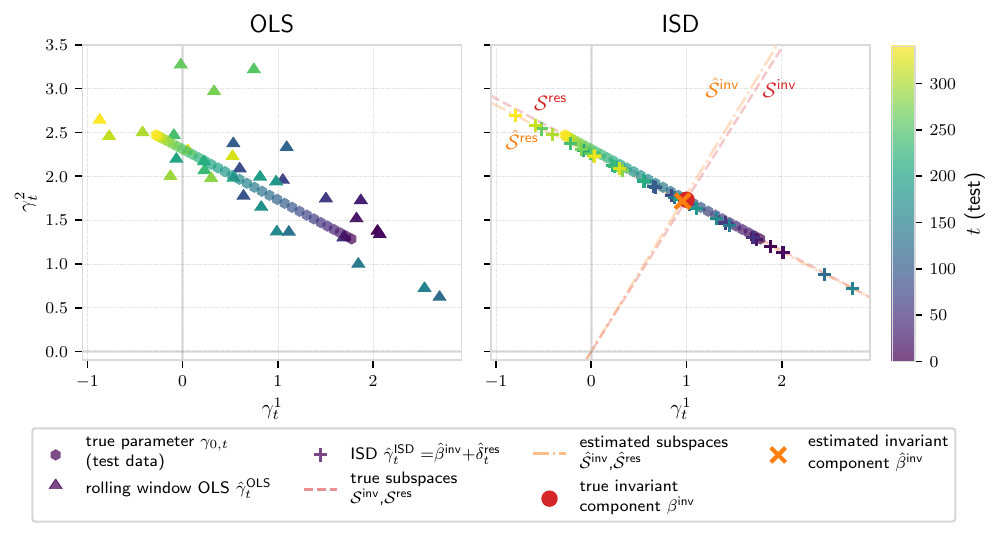}
    \caption{Example of two-dimensional true parameter $\gamma_{0,t}$ varying on a one-dimensional subspace of $\mathbb{R}^2$, and its estimates using \isd{} (right) compared to rolling window OLS (left).
Time is visually encoded using a color map. (Left) True parameter $\gamma_{0,t}$ (hexagons) on $350$ test points and OLS estimates $\hat{\gamma}^{\operatorname{OLS}}_t$ based on rolling windows of size $16$. (Right) Same test data and true parameters $\gamma_{0,t}$, but now we additionally use $1000$ prior time-points as historical data (not shown) to estimate the decomposition of $\mathbb{R}^2$ into the orthogonal subspaces $\mathcal{S}^{\operatorname{inv}}$ and $\mathcal{S}^{\operatorname{res}}$ (dashed lines). Next, we estimate $\beta^{\operatorname{inv}}$ using the historical data and $\hat{\mathcal{S}}^{\operatorname{inv}}$. Then, using the same rolling windows as in the left plot as adaption data, we estimate $\delta^{\operatorname{res}}_t$ using $\hat{\mathcal{S}}^{\operatorname{res}}$. The \isd{} estimates are then given by $\hat{\gamma}_t^{\isd}=\hat{\beta}^{\operatorname{inv}}+\hat{\delta}^{\operatorname{res}}_t$.
All details on the generative model are provided in Example~\ref{ex:ex:running_ex_2d}. 
The subspaces $\mathcal{S}^{\operatorname{inv}}$ and $\mathcal{S}^{\operatorname{res}}$ do not need to be axis aligned, so \isd{} is applicable even in cases where the conditional of $Y_t$ given $X_t$ and all conditionals of $Y_t$ given subsets of $X_t$ vary over time.}
\label{fig:2d_ex} 
\end{figure}

The fundamental assumption we make to apply the \isd{} framework to new data (Assumption~\ref{ass:generalization}) is that the decomposition inferred from available observations generalizes to unseen time points. This guarantees that the estimated invariant component can be meaningfully used for prediction at all new time points, either directly or as part of a two-step estimation that is fine-tuned by solving the time adaptation sub-problem.

The ISD framework allows us to improve on naive methods that directly maximize the explained variance on $\mathbb{R}^p$ using only the most recent available observations and on methods focusing on invariance across environments or time points such as the maximin by \citet{meinshausen2015maximin}, which do not account for time-adaptation. We show in particular in Theorem~\ref{thm:empirical_explained_variance} that isolating an invariant component and reducing the dimensionality of the adaptation problem guarantees lower prediction error or, equivalently, higher explained variance compared to naive methods. 
An example is provided for a simulated setting in Figure~\ref{fig:xplvar_prediction_comparison}, which shows on the left the average explained variance by the invariant component computed at training time on historical data, and on the right the cumulative explained variance by the invariant (zero-shot task) and the adapted invariant component (time-adaptation task) at testing time when new observations become available. For the zero-shot task, we compare the \isd{} invariant component with the standard OLS solution, computed on all training observations, and with the maximin effect ($\hat{\beta}^{\operatorname{mm}}$) by~\citet{meinshausen2015maximin}, which maximizes the worst case explained variance over the available observations. For the time-adaptation task, we compare the ISD  with the standard OLS solution computed on a rolling window over the latest new observations. For both tasks, the \isd{} estimates achieve higher cumulative explained variance at new time-points.

\begin{figure}[t]
    \centering
    \includegraphics[width=\textwidth]{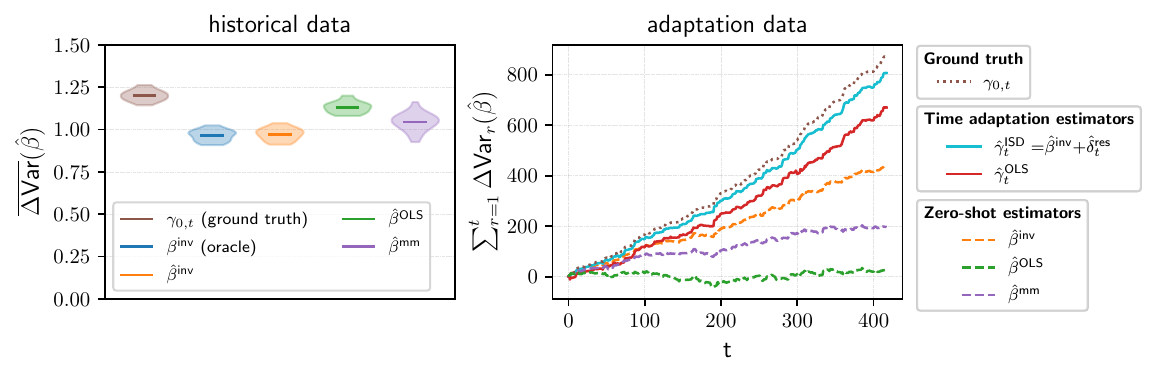} 
    \caption{
    For the data-generating model in Section~\ref{sec:simulations_time_adaptation}, we plot (left) the
    average explained variance (distribution over 20 runs) obtained at training time (historical data) and (right) the cumulative explained variance obtained testing time (adaptation data) (in one of the 20 runs); in this example, the time-varying components in the historical and adaptation data have disjoint support.
    The example considers $p=10$-dimensional predictors and an invariant component of dimension $7$. As baselines, we use (i) the true time-varying parameter $\gamma_{0,t}$, which maximizes the explained variance at all observed time points $t$, and (ii) the oracle invariant component $\beta^{\operatorname{inv}}$.
$6000$ 
historical observations are used to estimate: 
    (iii) the invariant component $\hat{\beta}^{\operatorname{inv}}$ of the \isd{} framework, (iv) the OLS solution $\hat{\beta}^{\operatorname{OLS}}$, (v) the maximin effect $\hat{\beta}^{\operatorname{mm}}$. Starting from $t=0$ after the observed history, windows of  length $3p$ are used to estimate: (vi) the adaptation parameter $\hat{\delta}^{\operatorname{res}}_t$ for $\hat{\beta}^{\operatorname{inv}}$ to obtain the \isd{} estimate $\hat{\gamma}^{\isd}$ and (vii) the rolling window OLS solution $\hat{\gamma}^{\operatorname{OLS}}_t$. While at training time on historical data the \isd{} invariant component $\hat{\beta}^{\operatorname{inv}}$ is the most conservative, with the lowest average explained variance, after a distribution shift (adaptation data) the same component can explain higher variance than other methods based on historical data only ($\hat{\beta}^{\operatorname{OLS}}, \hat{\beta}^{\operatorname{mm}}$), and can be tuned to new time points to improve on estimators based on adaptation data only ($\hat{\gamma}^{\operatorname{OLS}}_t$).
    }
    \label{fig:xplvar_prediction_comparison}
\end{figure}

The remainder of the paper is structured as follows. In Section~\ref{sec:time_inv_sep} we introduce the \isd{} framework, starting from the identification of the invariant and residual subspaces for orthogonal covariates transformation and showing the explained variance separation \eqref{eq:lin_subspace_separation} in Theorem~\ref{thm:subspace_partition}. In Sections~\ref{sec:time_invariant_subspace_effect} and \ref{sec:time_adaptation} we define population estimators for the invariant and residual components, solutions to the two sub-problems in \eqref{eq:lin_subspace_separation}.
We then describe, in Section~\ref{sec:task_description}, the two tasks that \isd{} solves, namely zero-shot prediction and time-adaptation, and provide a characterization of the invariant component as a worst-case optimal parameter.
 In Section~\ref{sec:estimation} we propose an estimation method for \isd, and provide finite sample generalization guarantees for the proposed estimator. 
 Finally, in Section~\ref{sec:simulations}, we illustrate \isd{} based on numerical experiments, both on simulated and on real world data, and validate the theoretical results presented in the paper.

\section{Invariant subspace decomposition}
\label{sec:time_inv_sep}
We formalize the setup described in the introduction in the following setting.
\begin{setting}
    Let $(X_t, Y_t)_{t\in\mathbb{N}}\subseteq \mathbb{R}^p\times\mathbb{R}^1$ be a sequence of independent random vectors satisfying for all $t\in\mathbb{N}$ a linear model as in \eqref{eq:model_def}, with $\gamma_{0,t}\in\mathbb{R}^p$ and  $\mathbb{E}[\epsilon_t\mid X_t]=0$. Assume that for all $t\in\mathbb{N}$ the covariance matrix of the predictors $\Sigma_t\coloneqq \operatorname{Var}(X_t)$ is strictly positive definite. Moreover, for all $n\in\mathbb{N}$, let $[n]\coloneqq\{1,\dots,n\}$ and assume we observe $n$ observations $(X_t, Y_t)_{t\in [n]}$ from model~\eqref{eq:model_def}, which we call \emph{historical data}. Additionally, let $\mathcal{I}^{\operatorname{ad}}\subseteq\mathbb{N}$ be an interval of consecutive time points with $m:=|\mathcal{I}^{\operatorname{ad}}|>p$ and $\min_{t\in\mathcal{I}^{\operatorname{ad}}}t > n$. Assume we observe a second set of observations $(X_t, Y_t)_{t\in \mathcal{I}^{\operatorname{ad}}}$ from the same model but succeeding the historical data, which we call \emph{adaptation data}.  Finally, denote by $t^*>\max_{t\in\mathcal{I}^{\operatorname{ad}}}$ a time point of interest that occurs after the adaptation data and assume that, for all $t\in\mathcal{I}^{\operatorname{ad}}\cup\{t^*\}$,  the quantities $\gamma_{0,t}$, $\Sigma_t$and $\operatorname{Var}(\epsilon_t)$ in model~\eqref{eq:model_def} are constant (in practice, being approximately constant is sufficient).
    \label{set:definition}
\end{setting}
Our goal is to predict $Y_{t^*}$ from $X_{t^*}$. A naive solution is to only consider the adaptation data in $\mathcal{I}^{\operatorname{ad}}$ (see also the analysis in Section~\ref{sec:task_description}): we aim to improve on this approach by additionally exploiting historical data in $[n]$. In particular, depending on whether we only have access to the historical data or we additionally have access to the data in $\mathcal{I}^{\operatorname{ad}}$, we can solve (i) the zero-shot task or (ii) the adaptation task, respectively.
The full setup including the different tasks is visualized in Figure~\ref{fig:datasets}. 
\begin{figure}[b]
\centering
\begin{tikzpicture}[%
  every node/.style={
    font=\scriptsize,
    text height=1ex,
    text depth=.25ex,
  }]
  
  % draw horizontal line
  \draw (0,0) -- (1.5,0);
  \node at (2,0) {$\ldots$};
  \draw (2.5,0) -- (5.5,0);
  \node at (6,0) {$\ldots$};
  \draw (6.5,0) -- (7.5,0);
  \node at (8,0) {$\ldots$};
  \draw[-latex] (8.5,0) -- (12.5,0);
  
  % draw vertical lines
  \foreach \x in {0,1}{
    \draw (\x cm,3pt) -- (\x cm,0pt);
  }
  \foreach \x in {3,4,5}{
    \draw (\x cm,3pt) -- (\x cm,0pt);
  }
  \foreach \x in {7}{
    \draw (\x cm,3pt) -- (\x cm,0pt);
  }
  \foreach \x in {9,10,11}{
    \draw (\x cm,3pt) -- (\x cm,0pt);
  }

  % numbers on axis
  \node[anchor=north] at (0,0) {\tiny $1$};
  \node[anchor=north] at (1,0) {\tiny $2$};
  \node[anchor=north] at (3,0) {\tiny $n-1$};
  \node[anchor=north] at (4,0) {\tiny $n$};
  \node[anchor=north] at (5,0) {\tiny $n+1$};
  \node[anchor=north] at (7,0) {\tiny $t^*-m$};
  \node[anchor=north] at (9,0) {\tiny $t^*-2$};
  \node[anchor=north] at (10,0) {\tiny $t^*-1$};
  \node[anchor=north] at (11,0) {\tiny $t^*$};
  
  \node[anchor=north] at (12,0) {time};
  
  % historical data
  \fill[red!20] (0,0.25) rectangle (4,2);
  \node at (2, 1.5) {historical data};
  \node at (2, 0.75) {$(X_t, Y_t)_{t\in[n]}$};
  
  % adaptation data
  \fill[green!20] (7,0.25) rectangle (10,2);
  \node at (8.5, 1.5) {adaptation data};
  \node at (8.5, 0.75) {$(X^{\operatorname{ad}}_t, Y^{\operatorname{ad}}_t)_{t\in\mathcal{I}^{\operatorname{ad}}}$};

  % prediction point
  \draw[dashed, blue] (11, 0) -- (11, 0.75);
  \node at (11, 1.1) {$X_{t^*}$};

  % tasks
  \node[anchor=west] at (0, 3) {Zero-shot task:};
  \node[anchor=west] at (2.5, 3) {observe $(X_t, Y_t)_{t\in[n]}$ and $X_{t^*}$};
  \node[anchor=west] at (9, 3) {predict $Y_{t^*}$};
  \node[anchor=west] at (0, 2.5) {Adaptation task:};
  \node[anchor=west] at (2.5, 2.5) {observe $(X_t, Y_t)_{t\in[n]}$, $(X^{\operatorname{ad}}_t, Y^{\operatorname{ad}}_t)_{t\in\mathcal{I}^{\operatorname{ad}}}$ and $X_{t^*}$};
  \node[anchor=west] at (9, 2.5) {predict $Y_{t^*}$};
\end{tikzpicture}
    \caption{
    Illustration of the historical and adaptation data and the zero-shot and adaptation tasks.
    }
    \label{fig:datasets}
\end{figure}
The two tasks correspond to the two sub-problems in \eqref{eq:lin_subspace_separation}:  the solution to (i) is a time-invariant parameter $\beta^{\operatorname{inv}}\in\mathbb{R}^p$, defined formally below in Section~\ref{sec:time_invariant_subspace_effect}; the solution to (ii) is an adaptation parameter $\delta^{\operatorname{res}}_t$ formally defined in Section~\ref{sec:time_adaptation} that satisfies $\beta^{\operatorname{inv}}+\delta^{\operatorname{res}}_t=\gamma_{0,t}$.

We start by defining what a time-invariant parameter is. To do so, we first consider the explained variance for a parameter $\beta\in\mathbb{R}^p$ at time $t\in\mathbb{N}$, defined by
\begin{align*}
    \operatorname{\Delta Var}_t(\beta) \coloneqq \operatorname{Var}(Y_t) - \operatorname{Var}(Y_t - X_t^{\top}\beta).
\end{align*}
Under model \eqref{eq:model_def}, the explained variance can be equivalently expressed as
\begin{align}
    \operatorname{\Delta Var}_t(\beta) & =  2\operatorname{Cov}(Y_t, X_t^{\top}\beta) -  \operatorname{Var}(X_t^{\top}\beta) \nonumber\\
    & =  2\operatorname{Cov}(Y_t-X_t^{\top}\beta, X_t^{\top}\beta) + \operatorname{Var}(X_t^{\top}\beta).
   \label{eq:explained_var}
\end{align}
A desirable property for a non-varying parameter $\beta\in\mathbb{R}^p$ is to guarantee that its explained variance remains non-negative at all time points. We call parameters $\beta\in\mathbb{R}^p$ \emph{never harmful} if, for all $t\in\mathbb{N}$, $\operatorname{\Delta Var}_t(\beta)\ge0$. Intuitively, parameters with this property at least partially explain $Y_t$ at all time points $t\in\mathbb{N}$, and are therefore meaningful to use for prediction.\footnote{Under an assumption similar to Assumption~\ref{ass:generalization} below, the maximin framework by \citet{meinshausen2015maximin} allows us to estimate never harmful parameters: we provide a more detailed comparison between our approach and the maximin in Remark~\ref{rem:maximin}.} 
In this paper, we consider the subset of never harmful parameters $\beta\in\mathbb{R}^p$ that satisfy $\operatorname{\Delta Var}_t(\beta) = \operatorname{Var}(X_t^{\top}\beta)$ or equivalently, using \eqref{eq:explained_var}, $\operatorname{Cov}(Y_t-X_t^{\top}\beta, X_t^{\top}\beta)=0$; 
thus, the explained variance of these parameters reflects changes of $\operatorname{Var}(X_t)$ (over $t$)
but is independent of changes in $\gamma_{0,t}$.
\begin{definition}[time-invariance]
\label{def:time_invariance}
    We call a parameter $\beta\in\mathbb{R}^p$ \emph{time-invariant (over $[n]$)} if for all $t\in[n]$,
    \begin{equation}
    \label{eq:invariant_constraint}
         \operatorname{Cov}(Y_t-X_t^{\top}\beta, X_t^{\top}\beta)=0.
    \end{equation}
\end{definition}
We use this definition, in Section~\ref{sec:time_inv_sub}, to define the subspaces $\mathcal{S}^{\operatorname{inv}}$ and $\mathcal{S}^{\operatorname{res}}$ that allow us to obtain the separation of $\gamma_{0,t}$ as in \eqref{eq:lin_subspace_separation}. In particular, as we show in Proposition~\ref{prop:invariance}, the definition of the invariant subspace $\mathcal{S}^{\operatorname{inv}}$ guarantees that the invariant component $\beta^{\operatorname{inv}}$ maximizing the pooled explained variance over $\mathcal{S}^{\operatorname{inv}}$ is always a time-invariant parameter according to Definition~\ref{def:time_invariance}.

While we consider the explained variance here, one can equivalently consider the mean squared error (MSPE).

\begin{remark}[Exchanging explained variance with MSPE]
\label{rem:expl_var_and_mse}
Under model~\eqref{eq:model_def} maximizing the explained variance at time $t\in[n]$ is equivalent to minimizing the MSPE at $t$, $\operatorname{MSPE}_t(\beta)\coloneqq\mathbb{E}[(Y_t - X_t^{\top}\beta)^2]$.  
More precisely, assuming that all variables have mean zero,
\begin{equation*}
   \argmax_{\beta\in\mathbb{R}^p}\operatorname{\Delta Var}_t(\beta)= \argmin_{\beta\in\mathbb{R}^p} \operatorname{Var}(Y_t - X_t^{\top}\beta) = \operatorname{Var}(X_t)^{-1}\operatorname{Cov}(X_t, Y_t)=\argmin_{\beta\in\mathbb{R}^p} \operatorname{MSPE}_t(\beta)
   .
\end{equation*}
\end{remark}

The remaining part of this section is organized as follows. 
Section~\ref{sec:time_inv_sub} explicitly constructs
the spaces $\mathcal{S}^{\operatorname{inv}}$ and $\mathcal{S}^{\operatorname{res}}$ and shows that they can be characterized by joint block diagonalization. In Section~\ref{sec:time_invariant_subspace_effect} we then analyze the time-invariant part of \eqref{eq:lin_subspace_separation}, leading to the optimal time-invariant parameter $\beta^{\operatorname{inv}}$, and show in Proposition~\ref{prop:invariance}~$(iii)$ 
that 
 its solution corresponds to the maximally predictive time-invariant parameter and is an interesting target of inference in its own right.
In Section~\ref{sec:time_adaptation} we analyze the residual part of \eqref{eq:lin_subspace_separation}, leading to the optimal parameter $\delta_t^{\operatorname{res}}$.

\subsection{Invariant and residual subspaces} 
\label{sec:time_inv_sub}

We now construct the two linear subspaces 
$\mathcal{S}^{\operatorname{inv}}$
and
$\mathcal{S}^{\operatorname{res}}$
that allow us to express the true time-varying parameter as the solution to two separate optimizations as in \eqref{eq:lin_subspace_separation}. For a linear subspace $\mathcal{S}\subseteq\mathbb{R}^p$, denote by $\Pi_{\mathcal{S}}$$\in\mathbb{R}^{p\times p}$ the orthogonal projection matrix from $\mathbb{R}^p$ onto $\mathcal{S}$, that is, a symmetric matrix such that $\Pi_{\mathcal{S}}^2=\Pi_{\mathcal{S}}$ and, for all vectors $v\in\mathbb{R}^p$, $\Pi_{\mathcal{S}}v\in\mathcal{S}$ and $(v-\Pi_{\mathcal{S}}v)^{\top}\Pi_{\mathcal{S}}v=0$. 
We call a collection of pairwise orthogonal linear subspaces $\mathcal{S}_1,\ldots,\mathcal{S}_q\subseteq\mathbb{R}^p$ with $\bigoplus_{j=1}^q \mathcal{S}_j=\mathbb{R}^p$ 
an \emph{orthogonal and $(X_t)_{t\in[n]}$-decorrelating partition (of cardinality $q$)}
if for all $i,j\in\{1,\dots,q\}$ with $i\neq j$, and for all $t\in[n]$ it holds that
\begin{equation}
\label{eq:orthogonal_partition}
    \operatorname{Cov}(\Pi_{\mathcal{S}_i}X_t, \Pi_{\mathcal{S}_j}X_t)=0.
\end{equation}
Moreover, we define an orthogonal and $(X_t)_{t\in[n]}$-decorrelating partition as \emph{irreducible} if there is no orthogonal and $(X_t)_{t\in[n]}$-decorrelating partition with strictly larger cardinality.
We will construct $\mathcal{S}^{\operatorname{inv}}$ and $\mathcal{S}^{\operatorname{res}}$ from an irreducible orthogonal and $(X_t)_{t\in[n]}$-decorrelating partition.

Given a (not necessarily irreducible) orthogonal and $(X_t)_{t\in [n]}$-decorrelating partition $\{\mathcal{S}_j\}_{j=1}^q$ the true time-varying parameter can be expressed as the sum of $q$ orthogonal components each lying in one subspace of the partition. These components can be expressed in terms of the covariates and of the response at time $t$ as shown by the following lemma.
\begin{lemma}
    \label{lem:time_var_parameter_separation}
    Let $\{\mathcal{S}_j\}_{j=1}^q$ be an orthogonal and $(X_t)_{t\in[n]}$-decorrelating partition. Then it holds for all $t\in[n]$ that $\gamma_{0,t}=\sum_{j=1}^q \Pi_{\mathcal{S}_j}\gamma_{0,t}$ and for all $j\in\{1,\dots, q\}$ that
    \begin{equation}
    \label{eq:separation_orth_partition1}
    \Pi_{\mathcal{S}_j}\gamma_{0,t}  = \operatorname{Var}\left(\Pi_{\mathcal{S}_j}X_t\right)^{\dagger}\operatorname{Cov}\left(\Pi_{\mathcal{S}_j}X_t,Y_t\right)
\end{equation}
where $(\cdot)^\dagger$ denotes the Moore-Penrose pseudoinverse.
\end{lemma}
Each component is in particular the projection of $\gamma_{0,t}$ on the corresponding subspace $\mathcal{S}_j$, as well as a maximizer of the explained variance in $\mathcal{S}_j$, as shown by Lemma~\ref{lem:subspace_maximization}.
\begin{lemma}
\label{lem:subspace_maximization}
    Let $\mathcal{N}\subseteq\mathbb{N}$ and
    let $\{\mathcal{S}_j\}_{j=1}^q$ be an orthogonal and $(X_t)_{t\in \mathcal{N}}$-decorrelating partition.\footnote{
    This is defined analogously to 
    an orthogonal and $(X_t)_{t\in[n]}$-decorrelating partition.
    }
    Then it holds for all $j\in\{1,\dots, q\}$ and for all $t\in\mathcal{N}$ that
    \begin{equation}
    \label{eq:separation_orth_partition2}
    \argmax_{\beta\in\mathcal{S}_j}\operatorname{\Delta Var}_t(\beta)
    =\operatorname{Var}(\Pi_{\mathcal{S}_j}X_t)^{\dagger}\operatorname{Cov}(\Pi_{\mathcal{S}_j}X_t, Y_t).
\end{equation}
\end{lemma}
Combining Lemmas~\ref{lem:time_var_parameter_separation} and \ref{lem:subspace_maximization} implies that
\begin{equation}
\label{eq:subspace_opt_sum}
    \gamma_{0,t}=\sum_{j=1}^q \argmax_{\beta\in\mathcal{S}_j}\operatorname{\Delta Var}_t(\beta).
\end{equation}
Therefore, any orthogonal and $(X_t)_{t\in[n]}$-decorrelating partition allows us to split the optimization~\eqref{eq:gamma0t}
of the explained variance into separate optimizations over the individual subspaces.
In order to leverage all available observations in at least some of the optimizations by pooling the explained variance as in \eqref{eq:lin_subspace_separation},  we need the optimizer to remain constant over time. More formally, we call a linear subspace $\mathcal{S}\subseteq\mathbb{R}^p$ \emph{opt-invariant (optimum invariant) on $[n]$} if for all $t,s\in[n]$ it holds that
\begin{equation}
   \argmax_{\beta\in\mathcal{S}}\operatorname{\Delta Var}_t(\beta) = \argmax_{\beta\in\mathcal{S}}\operatorname{\Delta Var}_s(\beta).
    \label{eq:time_invariant_irr_subspaces}
\end{equation}
For all terms in \eqref{eq:subspace_opt_sum} corresponding to an opt-invariant subspace, we can---by definition of opt-invariance---use all time-points in the optimization.
For an irreducible orthogonal and $(X_t)_{t\in[n]}$-decorrelating partition $\{\mathcal{S}_j\}_{j=1}^q$, we therefore define the \emph{invariant subspace} $\mathcal{S}^{\operatorname{inv}}$ and the \emph{residual subspace} $\mathcal{S}^{\operatorname{res}}$ by
\begin{equation}
\label{eq:invariant_residual_spaces}
    \mathcal{S}^{\operatorname{inv}} \coloneqq\textstyle\bigoplus_{\substack{j\in\{1,\dots,q\}:\\ \mathcal{S}_j \text{ opt-invariant on }[n]}}\mathcal{S}_j
    \quad\text{and}\quad
    \mathcal{S}^{\operatorname{res}}\coloneqq \textstyle\bigoplus_{\substack{j\in\{1,\dots,q\}:\\ \mathcal{S}_j\, \text{not opt-invariant on }[n]}}\mathcal{S}_j, 
\end{equation}
respectively.
The invariant subspace $\mathcal{S}^{\operatorname{inv}}$ is opt-invariant (see Lemma~\ref{lem:invariant_subspace}). Moreover, by definition, it holds that $\mathcal{S}^{\operatorname{res}}=\left(\mathcal{S}^{\operatorname{inv}}\right)^{\perp}$, where $(\cdot)^{\perp}$ denotes the orthogonal complement in $\mathbb{R}^p$, and that $\{\mathcal{S}^{\operatorname{inv}}, \mathcal{S}^{\operatorname{res}}\}$ is an orthogonal and $(X_t)_{t\in[n]}$-decorrelating partition (see Lemma~\ref{lem:inv_res_partition}). 
To ensure that these two spaces do not depend on the chosen irreducible orthogonal partition, we introduce the following assumption.
\begin{assumption}[uniqueness of the subspace decomposition]
    \label{ass:uniqueness_time_inv_sp}
     Let $\{\mathcal{S}_j\}_{j=1}^q$ and $\{\bar{\mathcal{S}}_j\}_{j=1}^{\bar{q}}$ be two irreducible orthogonal and $(X_t)_{t\in[n]}$-decorrelating partitions. Then, it holds that
    \begin{equation*}
         \textstyle\bigoplus_{\substack{j\in\{1,\dots,q\}:\\ \mathcal{S}_j \text{ opt-invariant on }[n]}}\mathcal{S}_j
         =\textstyle\bigoplus_{\substack{j\in\{1,\dots,\bar{q}\}:\\ \bar{\mathcal{S}}_j\text{ opt-invariant on }[n]}}\bar{\mathcal{S}}_j.
    \end{equation*}
\end{assumption}
Assumption~\ref{ass:uniqueness_time_inv_sp} is for example satisfied if an irreducible orthogonal and $(X_t)_{t\in[n]}$-decor\-relating partition is unique. As we show in Appendix ~\ref{ex:non_unique_orthogonal_partition}, the uniqueness does not always hold (e.g., if for all $t\in[n]$ the multiplicity of some eigenvalues of $\Sigma_t$ is larger than one and shared across all such matrices). 
Assumption~\ref{ass:uniqueness_time_inv_sp} is, however, a mild assumption; for example, the uniqueness of an irreducible orthogonal and $(X_t)_{t\in[n]}$-decorrelating partition is satisfied if there exists at least one $t\in[n]$ such that all eigenvalues of $\Sigma_t$, which we have assumed to be non-zero, are distinct (see Lemma~\ref{lem:jbd_uniqueness} and Proposition~\ref{prop:irreducible_orthogonal_partition} below).
Whenever Assumption~\ref{ass:uniqueness_time_inv_sp} holds, the invariant and residual spaces $\mathcal{S}^{\operatorname{inv}}$ and $\mathcal{S}^{\operatorname{res}}$ do not depend on which irreducible orthogonal partition is used in their construction.

\begin{example}[label=ex:running_ex_2d]
\label{ex:ex:running_ex_2d}
This example describes the setting used to generate Figure~\ref{fig:2d_ex}. Consider a $2$-dimensional covariate $X_t\in\mathbb{R}^2$, and assume that model \eqref{eq:model_def} is defined as follows. We take, for all $t\in[n]$
\begin{align*}
    \gamma_{0,t} = \begin{bmatrix}
        1.5\sqrt{3}+1-\sqrt{3}t/n \\
        t/n - 1.5 + \sqrt{3} 
    \end{bmatrix}
    \quad\text{and}\quad
    \Sigma_t = \frac{1}{4}\begin{bmatrix}
        3\sigma_{1,t}+\sigma_{2,t} & \sqrt{3}(\sigma_{2,t}-\sigma_{1,t})\\
        \sqrt{3}(\sigma_{2,t}-\sigma_{1,t}) & \sigma_{1,t}+3\sigma_{2,t}
    \end{bmatrix},
\end{align*}
where $\sigma_{1, t}$ and $\sigma_{2,t}$ are two fixed sequences sampled as two independent i.i.d.\ samples from a uniform distribution on $[0, 1]$. In this example, we have that an irreducible orthogonal and $(X_t)_{t\in[n]}$-decorrelating partition is given by the two spaces
\begin{equation*}
    \mathcal{S}_1 = \langle\begin{bmatrix}
        0.5\sqrt{3} \\ -0.5
    \end{bmatrix}\rangle \quad \mathcal{S}_2 = \langle\begin{bmatrix}
        0.5 \\ 0.5\sqrt{3}
    \end{bmatrix}\rangle.
\end{equation*}
Indeed, it holds for all $t\in[n]$ that $\operatorname{Cov}(\Pi_{\mathcal{S}_1}X_t, \Pi_{\mathcal{S}_2}X_t) = \Pi_{\mathcal{S}_1}\Sigma_t\Pi_{\mathcal{S}_2} = \mathbf{0}_{4\times 4}$. Moreover, since $\sigma_{1,t}$ and $\sigma_{2,t}$ are the eigenvalues of $\Sigma_t$, the irreducible orthogonal partition defined by $\mathcal{S}_1$ and $\mathcal{S}_2$ is unique (i.e., Assumption~\ref{ass:uniqueness_time_inv_sp} is satisfied by Lemma~\ref{lem:jbd_uniqueness}) if $\sigma_{1,t}\neq \sigma_{2,t}$ for some $t\in[n]$.
It also holds that 
\begin{equation*}
    \Pi_{\mathcal{S}_2} \gamma_{0,t} = \begin{bmatrix}
        1 \\ \sqrt{3}
    \end{bmatrix},
\end{equation*}
which does not depend on $t$ (it can be verified that the same does not hold for $\Pi_{\mathcal{S}_1} \gamma_{0,t}$). As shown in Lemmas~\ref{lem:time_var_parameter_separation} and \ref{lem:subspace_maximization}, it holds that
\begin{align*}
     \argmax_{\beta\in\mathcal{S}_j}\operatorname{\Delta Var}_t(\beta) = \argmax_{\beta\in\mathcal{S}_j} \operatorname{Var}(Y_t)-\operatorname{Var}(Y_t-X_t^{\top}\Pi_{\mathcal{S}_j}\beta) =  \Pi_{\mathcal{S}_j} \gamma_{0,t}.
 \end{align*}
It follows that the subspace $\mathcal{S}_2$ is opt-invariant, whereas $\mathcal{S}_1$ is not, and therefore $\mathcal{S}^{\operatorname{inv}} = \mathcal{S}_2$ and $\mathcal{S}^{\operatorname{res}} =  \mathcal{S}_1$. The two spaces $\mathcal{S}_1$ and $\mathcal{S}_2$ also appear in Figure~\ref{fig:2d_ex}: the true time-varying parameter $\gamma_{0,t}$ does not vary with $t$ in the direction of the vector $[0.5, 0.5\sqrt{3}]^{\top}$ generating $\mathcal{S}_2$; $\mathcal{S}_1$ can be visualized when connecting the circles.
\end{example}

In order for $\mathcal{S}^{\operatorname{inv}}$ and $\mathcal{S}^{\operatorname{res}}$ to be 
useful for prediction on future observations, we assume that the subspace separations we consider remain fixed over time.
 
\begin{assumption}[generalization]
    \label{ass:generalization}
    For all irreducible orthogonal and $(X)_{t\in[n]}$-decorrelating partitions,
     $\mathcal{S}^{\operatorname{inv}}$ and $\mathcal{S}^{\operatorname{res}}$ defined in \eqref{eq:invariant_residual_spaces} satisfy the following two conditions: (i) for all $t\in\mathbb{N}$ it holds that $ \operatorname{Cov}(\Pi_{\mathcal{S}^{\operatorname{inv}}}X_t, \Pi_{\mathcal{S}^{\operatorname{res}}}X_t)=0$ and (ii) $\mathcal{S}^{\operatorname{inv}}$ is opt-invariant on $\mathbb{N}$.
\end{assumption}
 
As shown in the following theorem, Assumption~\ref{ass:generalization} ensures that the two sets satisfy 
a separation of the form \eqref{eq:lin_subspace_separation} for all observed and unobserved time points $t\in\mathbb{N}$, as desired. 
For unobserved time points $t\in\mathbb{N}\setminus[n]$, Assumption~\ref{ass:generalization} does not require
\eqref{eq:orthogonal_partition} to hold for all $i,j\in\{1,\dots,q\}$, but only for all $i,j\in\{1,\dots,q\}$ such that $\mathcal{S}_i\subseteq\mathcal{S}^{\operatorname{inv}}$ and $\mathcal{S}_j\subseteq\mathcal{S}^{\operatorname{res}}$.
\begin{theorem}
\label{thm:subspace_partition}
    Assume Assumption~\ref{ass:generalization} 
    is satisfied. Let $\mathcal{S}^{\operatorname{inv}}$ and $\mathcal{S}^{\operatorname{res}}$ be defined as in \eqref{eq:invariant_residual_spaces} for an arbitrary irreducible orthogonal and $(X_t)_{t\in[n]}$-decorrelating partition. Then, it holds for all $t\in\mathbb{N}$ that
    \begin{equation}
    \label{eq:time_varying_parameter_separation}
        \gamma_{0,t}=\argmax_{\beta\in\mathcal{S}^{\operatorname{inv}}}\overline{\operatorname{\Delta Var}}(\beta) + \argmax_{\beta\in\mathcal{S}^{\operatorname{res}}}\operatorname{\Delta Var}_t(\beta),
    \end{equation}
    where $\overline{\operatorname{\Delta Var}}(\beta)\coloneqq \frac{1}{n}\sum_{t=1}^n\operatorname{\Delta Var}_t(\beta)$. Moreover, if Assumption~\ref{ass:uniqueness_time_inv_sp} is satisfied, the separation in \eqref{eq:time_varying_parameter_separation} is independent of the considered irreducible orthogonal and $(X_t)_{t\in[n]}$-decorrelating partition. 
\end{theorem}
The proof of Theorem~\ref{thm:subspace_partition} can be found in Appendix~\ref{pf:subspace_partition} and relies on the fact that, under Assumption~\ref{ass:generalization}, the invariant and residual subspace form an orthogonal and $(X_t)_{t\in\mathbb{N}}$-decorrelating partition.

\subsubsection{Identifying invariant and residual subspaces using joint block diagonalization}
\label{sec:orthogonal_joint_block_diagonalization}

We can characterize an irreducible orthogonal partition using
joint block diagonalization of the set of covariance matrices $(\Sigma_t)_{t\in[n]}$. Joint block diagonalization of $(\Sigma_t)_{t\in[n]}$ consists of finding an orthogonal matrix $U\in\mathbb{R}^{p\times p}$ such that the matrices $\tilde{\Sigma}_t\coloneqq U^{\top}\Sigma_t U$, $t\in[n]$, are block diagonal and we can choose $q^U$ blocks such that the indices of the corresponding submatrices do not change with $t$ (the entries of the blocks may change with $t$ though). 
Let $q_{\max}^U$ be the largest number of such blocks and let $(\tilde{\Sigma}_{t,1})_{t\in[n]},  \ldots, (\tilde{\Sigma}_{t,q_{\max}^U})_{t \in [n]}$ denote the corresponding \emph{common blocks} with dimensions $p_1, \ldots, p_{q_{\max}^U}$  that are  independent of $t$. 
We call $U$ a \emph{joint block diagonalizer} of $(\Sigma_t)_{t\in[n]}$. Moreover, we call $U$ an \emph{irreducible} joint block diagonalizer if, in addition, for all other joint block diagonalizers $U'\in\mathbb{R}^{p\times p}$ the resulting number of common blocks is at most $q_{\max}^U$.
Joint block diagonalization has been considered extensively in the literature \citep[see, for example,][]{murota2010numerical, nion2011tensor, tichavsky2012algorithms} and various computationally feasible algorithms have been proposed (see Section~\ref{rem:ajbd} for further details).

In our setting, joint block diagonalization can be used to identify the invariant and residual subspaces, since an irreducible joint block diagonalizer $U$ of the covariance matrices $(\Sigma_t)_{t\in[n]}$ corresponds to an irreducible orthogonal partition, as the following proposition shows. 
\begin{proposition}
    \label{prop:irreducible_orthogonal_partition}
    \begin{enumerate}[(i)]
        \item Let $U\in\mathbb{R}^{p\times p}$ be a joint block diagonalizer of $(\Sigma_t)_{t\in[n]}$. 
    For all $j\in\{1,\dots,q^U_{\max}\}$, let $S_j\subseteq\{1,\dots,p\}$ denote the subset of indices corresponding to the $j$-th common block $\tilde{\Sigma}_{t,j}$. Moreover, let $u^k$ denote the $k$-th column of $U$ and, for all $j\in\{1,\dots, q^U_{\max}\}$, define
    \begin{equation*}
        \mathcal{S}_j \coloneqq \operatorname{span}\{u^k \mid k\in S_j\}.
    \end{equation*}
    Then, $\{\mathcal{S}_j\}_{j=1}^{q^U_{\max}}$ is an orthogonal and $(X_t)_{t\in[n]}$-decorrelating partition.
    Moreover, if the joint block diagonalizer is irreducible, then the corresponding orthogonal partition is irreducible.
    \item  The converse is also true. Let $\{\mathcal{S}_j\}_{j=1}^q$ be an orthogonal and $(X_t)_{t\in[n]}$-decorrelating partition.
    Then, there exists a joint block diagonalizer $U\in\mathbb{R}^{p\times p}$ of $(\Sigma_t)_{t\in [n]}$ such that for all $t\in[n]$ the matrix 
    $\tilde{\Sigma}_t\coloneqq U^{\top}\Sigma_t U$ is block diagonal with $q$ diagonal blocks $\widetilde{\Sigma}_{t,j}=(U^{S_j})^{\top}\Sigma_t U^{S_j}$, $j\in\{1,\dots,q\}$ of dimension $|S_j|=\operatorname{dim}(\mathcal{S}_j)$, where $S_j\subseteq \{1,\dots,p\}$ indexes a subset of the columns of $U$. If the orthogonal partition is irreducible, then $U$ is an irreducible joint block diagonalizer. Moreover, $\Pi_{\mathcal{S}_j}=U^{S_j}(U^{S_j})^{\top}$.
    \end{enumerate}
\end{proposition}
If Assumption~\ref{ass:uniqueness_time_inv_sp} is satisfied, any irreducible joint block diagonalizer $U$, via its corresponding irreducible orthogonal and $(X_t)_{t\in[n]}$-decorrelating partition constructed in Proposition~\ref{prop:irreducible_orthogonal_partition}(i), leads to the same (unique) invariant and residual subspaces defined in \eqref{eq:invariant_residual_spaces}.

It is clear that Assumption~\ref{ass:uniqueness_time_inv_sp} is automatically satisfied whenever the joint block diagonalization is unique up to trivial indeterminacies, that is, if for all orthogonal matrices $U, U'\in\mathbb{R}^{p\times p}$ that jointly block diagonalize the set $(\Sigma_t)_{t\in[n]}$ into $q^U_{\max}$ 
irreducible common 
blocks, it holds that $U$ is equal to $U'$ up to block permutations and block-wise isometric transformations. Explicit conditions under which uniqueness of joint block diagonalization is satisfied can be found, for example, in the works by~\citet{de2008decompositions,murota2010numerical}. Intuitively, these conditions are satisfied whenever there is sufficient variability across time in the covariance matrices $(\Sigma_t)_{t\in[n]}$.

Given an irreducible joint block diagonalizer $U$, the invariant and residual subspaces can be identified using Proposition~\ref{prop:irreducible_orthogonal_partition} and Lemma~\ref{lem:time_var_parameter_separation}, by checking for all $j\in\{1,\dots, q^U_{\max}\}$ whether $\Pi_{\mathcal{S}_j}\gamma_{0,t}=U^{S_j}(U^{S_j})^{\top}\gamma_{0,t}$ remains constant ($\mathcal{S}_j\subseteq\mathcal{S}^{\operatorname{inv}}$) or not ($\mathcal{S}_j\subseteq\mathcal{S}^{\operatorname{res}}$) on $[n]$. We denote by $U^{\operatorname{inv}}$ and $U^{\operatorname{res}}$ the submatrices of $U$ formed by the columns that span the invariant and the residual subspace respectively.
\begin{example}[continues=ex:running_ex_2d]
    In Example~\ref{ex:ex:running_ex_2d} so far, we expressed 
    $\mathcal{S}_1$ and $\mathcal{S}_2$ in terms of their generating vectors. These can be retrieved using Proposition~\ref{prop:irreducible_orthogonal_partition} by joint block diagonalizing the matrices $(\Sigma_t)_{t\in[n]}$. In this specific example, an irreducible joint block diagonalizer is given by 
    \begin{equation*}
        U = \begin{bmatrix}
            0.5\sqrt{3} & 0.5 \\
            -0.5 & 0.5\sqrt{3}
        \end{bmatrix},
    \end{equation*}
    which is a (clockwise) rotation matrix of $30$ degrees (see Figure~\ref{fig:2d_ex} and use $S_1 = \{1\}$ and $S_2 = \{2\}$). In particular, we have that $\tilde{\Sigma}_t = U^{\top}\Sigma_t U = \operatorname{diag}(\sigma_{1,t}, \sigma_{2,t})$: therefore, $q^U_{\max}=2$ and each block has dimension $1$.
    Moreover, it holds for all $t\in[n]$ that
    \begin{equation*}
         \Pi_{\mathcal{S}_1}\gamma_{0,t} = U^{S_1}(U^{S_1})^{\top}\gamma_{0,t} = \begin{bmatrix}
            1.5\sqrt{3}-\sqrt{3}t/n \\ t/n -1.5
        \end{bmatrix} \quad \text{and} \quad \Pi_{\mathcal{S}_2} \gamma_{0,t} =  U^{S_2}(U^{S_2})^{\top}\gamma_{0,t}= \begin{bmatrix}
        1 \\ \sqrt{3}
    \end{bmatrix},
    \end{equation*}
    and therefore $\mathcal{S}^{\operatorname{inv}}=\mathcal{S}_2$ and $\mathcal{S}^{\operatorname{res}}=\mathcal{S}_1$.
\end{example}

\subsection{Invariant component}
\label{sec:time_invariant_subspace_effect}

In Theorem~\ref{thm:subspace_partition} we have shown that the true time-varying parameter $\gamma_{0,t}$ can be expressed as the result of two separate optimization problems over the two orthogonal spaces $\mathcal{S}^{\operatorname{inv}}$ and $\mathcal{S}^{\operatorname{res}}$. 
In this section we analyze the result of the first optimization step over the invariant subspace $\mathcal{S}^{\operatorname{inv}}$. To ensure that the space $\mathcal{S}^{\operatorname{inv}}$ is unique, we assume that Assumption~\ref{ass:uniqueness_time_inv_sp} is satisfied throughout Section~\ref{sec:time_invariant_subspace_effect}.
\begin{definition}[Invariant component]
\label{def:time_invariant_subspace_effect}
    We denote the parameter that maximizes the explained variance over the invariant subspace by 
    \begin{equation}
        \beta^{\operatorname{inv}}\coloneqq \argmax_{\beta\in\mathcal{S}^{\operatorname{inv}}}\overline{\operatorname{\Delta Var}}(\beta).
    \label{eq:timeinvariantsubspaceeffect}
    \end{equation}
\end{definition}
The parameter $\beta^{\operatorname{inv}}$ corresponds to the pooled OLS solution obtained by regressing $Y_t$ on the projected
predictors $\Pi_{\mathcal{S}^{\operatorname{inv}}}X_t$, and can be computed using all observations in $[n]$. 
The whole procedure to find $\beta^{\operatorname{inv}}$, by first identifying $\mathcal{S}^{\operatorname{inv}}$ via joint block diagonalization and then using Proposition~\ref{prop:invariance}~$(i)$ below, is summarized in Algorithm~\ref{alg:pop_estimator} (see Section~\ref{sec:pop_isd_algorithm}).
Proposition~\ref{prop:invariance} summarizes some of the properties of $\beta^{\operatorname{inv}}$.
\begin{proposition}[Properties of $\beta^{\operatorname{inv}}$]
    \label{prop:invariance} 
    Under Assumption~\ref{ass:uniqueness_time_inv_sp}, $\beta^{\operatorname{inv}}$ satisfies the following properties:
    \begin{itemize}
        \item[(i)] For all $t\in[n]$, $\beta^{\operatorname{inv}} = \Pi_{\mathcal{S}^{\operatorname{inv}}} \gamma_{0,t} = \Pi_{\mathcal{S}^{\operatorname{inv}}} \overline{\gamma}_{0}$, where $\bar{\gamma}_{0}\coloneqq\frac{1}{n}\sum_{t=1}^n\gamma_{0,t}$.
    \item[(ii)] $\beta^{\operatorname{inv}}$ is time-invariant over $[n]$, see Definition~\ref{def:time_invariance}.
    \item[(iii)] 
    If, in addition, it holds for all $\beta\in\mathbb{R}^p$ time-invariant over $[n]$ that $\beta \in \mathcal{S}^{\operatorname{inv}}$ then
         $\beta^{\operatorname{inv}} = \argmax_{\beta\in\mathbb{R}^p\operatorname{ time-invariant}}\overline{\operatorname{\Delta Var}}(\beta).$
    \end{itemize}
\end{proposition}
Definition~\ref{def:time_invariance} guarantees, for all $t\in[n]$, that the explained variance for all $\beta\in\mathbb{R}^p$ time-invariant over $[n]$ is  $\operatorname{\Delta Var}_t(\beta)=\operatorname{Var}(X^{\top}\beta)=\beta^{\top}\Sigma_t\beta$.
Point $(ii)$ of Proposition~\ref{prop:invariance} therefore implies that, for all $t\in[n]$, $\operatorname{\Delta Var}_t(\beta^{\operatorname{inv}})=(\beta^{\operatorname{inv}})^{\top}\Sigma_t\beta^{\operatorname{inv}}$. Under Assumption~\ref{ass:generalization}, we have that the same holds for all $t\in\mathbb{N}$: $\beta^{\operatorname{inv}}$ is therefore a never harmful parameter and 
for all $t\in\mathbb{N}$ it
is a solution to $\argmax_{\beta\in\mathcal{S}^{\operatorname{inv}}}\operatorname{\Delta Var}_t(\beta)$.
Moreover, Proposition~\ref{prop:invariance}~$(iii)$ implies that, under an additional assumption, the parameter $\beta^{\operatorname{inv}}$ is optimal, i.e., maximizes the explained variance, among all time-invariant parameters. 
In particular, $\beta^{\operatorname{inv}}$ represents an interesting target of inference: it can be used for zero-shot prediction, if no adaptation data is available at time $t$ to solve the second part of the optimization in \eqref{eq:time_varying_parameter_separation} over $\mathcal{S}^{\operatorname{res}}$. 
Using $U^{\operatorname{inv}}$ as defined in Section~\ref{sec:orthogonal_joint_block_diagonalization}, we can express $\beta^{\operatorname{inv}}$ as follows.
\begin{equation}
\label{eq:invariant_component_pop_estimator}
    \beta^{\operatorname{inv}} = U^{\operatorname{inv}} ((U^{\operatorname{inv}})^{\top}\overline{\operatorname{Var}}(X)U^{\operatorname{inv}})^{-1}(U^{\operatorname{inv}})^{\top}\overline{\operatorname{Cov}}(X, Y)
\end{equation}
where $\overline{\operatorname{Var}}(X)\coloneqq\frac{1}{n}\sum_{t=1}^n\operatorname{Var}(X_t)$ and $\overline{\operatorname{Cov}}(X, Y)\coloneqq\frac{1}{n}\sum_{t=1}^n\operatorname{Cov}(X_t, Y_t)$. We derive this expression in Lemma~\ref{lem:invariant_component_pop_estimator}, and we later use it for estimation in Section~\ref{sec:residual_effect_estimation}.

\begin{example}[continues=ex:running_ex_2d]
Considering again Example~\ref{ex:running_ex_2d}, we have that $\mathcal{S}^{\operatorname{inv}} = \mathcal{S}_2$, and therefore the invariant component is given by
\begin{equation*}
    \beta^{\operatorname{inv}}=  \argmax_{\beta\in\mathcal{S}_2}\overline{\operatorname{\Delta Var}}(\beta) = \Pi_{\mathcal{S}_2} \overline{\gamma}_{0} =  \begin{bmatrix}
        1 \\ \sqrt{3}
    \end{bmatrix}.
\end{equation*}
Moreover, we can express $\beta^{\operatorname{inv}}$ under the basis for the irreducible subspaces using the irreducible joint block diagonalizer $U$, obtaining $U^{\top}\beta^{\operatorname{inv}}= \begin{bmatrix} 0,&2\end{bmatrix}^{\top}$: the only non-zero component is indeed the one corresponding to the invariant subspace $\mathcal{S}_2$.
\end{example}

\subsection{Residual component and time adaptation }
\label{sec:time_adaptation}
Under the generalization assumption (Assumption~\ref{ass:generalization}), Theorem~\ref{thm:subspace_partition} implies that we can partially explain the variance of the response $Y_t$ at all---observed and unobserved---time points $t\in\mathbb{N}$ using $\beta^{\operatorname{inv}}$. It also implies that for all $t\in\mathbb{N}$ we can reconstruct the true time-varying parameter by adding
to $\beta^{\operatorname{inv}}$ a residual parameter that maximizes the explained variance at time $t$ over the residual subspace $\mathcal{S}^{\operatorname{res}}$, i.e.,
\begin{equation*}
    \gamma_{0,t} = \beta^{\operatorname{inv}} + \argmax_{\beta\in\mathcal{S}^{\operatorname{res}}} \operatorname{\Delta Var}_t(\beta).
\end{equation*}
In this section we focus on the second optimization step over $\mathcal{S}^{\operatorname{res}}$.
We assume that Assumptions~\ref{ass:uniqueness_time_inv_sp} and \ref{ass:generalization} hold, and for all $t\in\mathbb{N}$ we define the residual component
$
\delta_t^{\operatorname{res}} \coloneqq \argmax_{\beta\in\mathcal{S}^{\operatorname{res}}} \operatorname{\Delta Var}_t(\beta).
$
Using \eqref{eq:separation_orth_partition1}, \eqref{eq:separation_orth_partition2} and \eqref{eq:invariant_residual_spaces}, we can express $\delta_t^{\operatorname{res}}$ as
\begin{align}
\label{eq:adaptation_parameter_2}
    \delta_t^{\operatorname{res}} & = \operatorname{Var}(\Pi_{\mathcal{S}^{\operatorname{res}}}X_t)^{\dagger}\operatorname{Cov}(\Pi_{\mathcal{S}^{\operatorname{res}}}X_t, Y_t) \\
    & = \operatorname{Var}(\Pi_{\mathcal{S}^{\operatorname{res}}}X_t)^{\dagger}\operatorname{Cov}(\Pi_{\mathcal{S}^{\operatorname{res}}}X_t, Y_t-X_t^{\top}\beta^{\operatorname{inv}})\notag.
\end{align}
We can now reduce the number of parameters that need to be estimated by expressing $\delta_t^{\operatorname{res}}$ as an OLS solution with $\operatorname{dim}(\mathcal{S}^{\operatorname{res}})$ parameters.
To see this, we use that, under Assumption~\ref{ass:uniqueness_time_inv_sp},  Proposition~\ref{prop:irreducible_orthogonal_partition} allows us to express the space $\mathcal{S}^{\operatorname{res}}$ in terms of an irreducible joint block diagonalizer $U$ corresponding to the irreducible orthogonal partition $\{\mathcal{S}_j\}_{j=1}^{q^U_{\max}}$ as
\begin{equation*}
    \mathcal{S}^{\operatorname{res}} = \operatorname{span}\{u^k\mid \exists j\in \{1,\dots,q^U_{\max}\}:\, \mathcal{S}_j \text{ not opt-invariant} \text{ and }k\in S_j\}.
\end{equation*}
Moreover, the orthogonal projection matrix onto $\mathcal{S}^{\operatorname{res}}$ is given by $\Pi_{\mathcal{S}^{\operatorname{res}}}=U^{\operatorname{res}}(U^{\operatorname{res}})^{\top}$. Hence, using Lemma~\ref{lem:pseudo_inv_projection} %\rmm{\ref{lem:time_var_parameter_separation}}\niklas{+1}
we get that
\begin{align}
    \delta_t^{\operatorname{res}} 
    & = U^{\operatorname{res}}\operatorname{Var}((U^{\operatorname{res}})^{\top}X_t)^{-1} \operatorname{Cov}((U^{\operatorname{res}})^{\top}X_t, Y_t-X_{t}^{\top}\beta^{\operatorname{inv}}), \label{eq:adaptation_parameter_estimator}
\end{align}
where $\operatorname{Var}((U^{\operatorname{res}})^{\top}X_t)^{-1} \operatorname{Cov}((U^{\operatorname{res}})^{\top}X_t, Y_t-X_{t}^{\top}\beta^{\operatorname{inv}})$ is the population ordinary least squares solution obtained by regressing $Y_t-X_{t}^{\top}\beta^{\operatorname{inv}}$ on $(U^{\operatorname{res}})^{\top}X_t\in\mathbb{R}^{\operatorname{dim}(\mathcal{\mathcal{S}^{\operatorname{res}}})}$.

\begin{example}[continues=ex:running_ex_2d]
    In Example~\ref{ex:ex:running_ex_2d}, we have that for all $t\in[n]$ the residual component $\delta^{\operatorname{res}}_t$ is given by 
    \begin{equation*}
        \delta^{\operatorname{res}}_t = \argmax_{\beta\in\mathcal{S}_1}\operatorname{\Delta Var}_t(\beta) = \Pi_{\mathcal{S}_1}\gamma_{0,t} = \begin{bmatrix}
            1.5\sqrt{3}-\sqrt{3}t/n \\ t/n -1.5
        \end{bmatrix}.
    \end{equation*}
    Moreover, we can express $\delta^{\operatorname{res}}_t$ under the irreducible orthogonal partition basis (given by the two vectors generating $\mathcal{S}_1$ and $\mathcal{S}_2$), obtaining $U^{\top}\delta^{\operatorname{res}}_t = \begin{bmatrix}
        3-2t/n, & 0
    \end{bmatrix}^{\top}$, which indeed only has one degree of freedom in the first component. This component takes the following values:  for all $t\in[n]$ we have $3-2t/n\in[1, 3]$.  
\end{example}

\subsection{Population \isd{} algorithm}
\label{sec:pop_isd_algorithm}
We call the procedure to identify $\beta^{\operatorname{inv}}$ and $\delta^{\operatorname{res}}_t$ \emph{invariant subspace decomposition} (\isd). By construction, the result of \isd{} in its population version is equal to the true time-varying parameter at time $t\in\mathbb{N}$, i.e.,
\begin{equation}
    \gamma_{0,t} =\beta^{\operatorname{inv}}+\delta_t^{\operatorname{res}}. \label{eq:adapted_iss_effect}
\end{equation}
The full \isd{} procedure is summarized in Algorithm~\ref{alg:pop_estimator}. 
In the algorithm, the invariant and residual subspaces are identified through joint block diagoanlization as described at the end of Section~\ref{sec:orthogonal_joint_block_diagonalization}. The number of subspaces $q$ in the irreducible orthogonal and $(X_t)_{t\in[n]}$-decorrelating partition is inferred by the joint block diagonalization algorithm.
If Assumption~\ref{ass:uniqueness_time_inv_sp} is not satisfied, then the decomposition in \eqref{eq:adapted_iss_effect} and therefore the output of Algorithm~\ref{alg:pop_estimator} depends on the irreducible orthogonal and $(X_t)_{t\in[n]}$-decorrelating partition used. 

In Appendix~\ref{sec:non_orthogonal_subspaces} we show that a decomposition of the true time-varying parameter similar to \eqref{eq:adapted_iss_effect} is also obtained when considering $(X_t)_{t\in[n]}$-decorrelating partitions that are not orthogonal, i.e., such that \eqref{eq:orthogonal_partition} holds but the subspaces in the partition are not necessarily pairwise orthogonal. In this case, in particular, we can still find an invariant and residual parameter as in \eqref{eq:invariant_component_pop_estimator} and \eqref{eq:adaptation_parameter_estimator}, respectively, where the matrix $U$ is now a non-orthogonal irreducible joint block diagonalizer.
\begin{algorithm}[h]
    \caption{Population \isd}\label{alg:pop_estimator}
    \begin{algorithmic}[1]
    \Require distributions of $(X_t, Y_t)_{t\in[n]}$
    \Ensure $\beta^{\operatorname{inv}}$, $\delta^{\operatorname{res}}_n$
    \State $(\Sigma_t)_{t\in[n]} \gets \{\operatorname{Var(X_t)}\}_{t\in[n]}$ 
    \State $\{\gamma_{0,t}\}_{t\in[n]}\gets \{\operatorname{Var(X_t)}^{-1}\operatorname{Cov}(X_t, Y_t)\}_{t\in[n]}$
    \State $U,\{S_j\}_{j=1}^q \gets \operatorname{irreducibleJointBlockDiagonalizer}((\Sigma_t)_{t\in[n]})$ \Comment{{\color{gray}Prop.~\ref{prop:irreducible_orthogonal_partition}}}
    \State $\rhd$ {\color{gray} Find the opt-invariant subspaces to identify $\mathcal{S}^{\operatorname{inv}}, \mathcal{S}^{\operatorname{res}}$}
    \State $S^{\operatorname{inv}}, S^{\operatorname{res}} \gets \emptyset $
    \For{$j=1,\dots,q$} 
        \State $\Pi_{\mathcal{S}_j} \gets U^{S_j}(U^{S_j})^{\top}$
        \If{$\Pi_{\mathcal{S}_j}\gamma_{0,t}$ constant in $[n]$ } 
        $S^{\operatorname{inv}}\gets S^{\operatorname{inv}}\cup S_j$ \Comment{{\color{gray}see Prop.~\ref{prop:invariance}}}
        \Else  $\quad S^{\operatorname{res}}\gets S^{\operatorname{res}}\cup S_j$
        \EndIf
    \EndFor
    \State $\rhd$ {\color{gray}Define estimators for the invariant and residual component of $\gamma_{0,t}$:}
    \State $\Pi_{\mathcal{S}^{\operatorname{inv}}} \gets U^{\operatorname{inv}}(U^{\operatorname{inv}})^{\top}$
    \State $\beta^{\operatorname{inv}}\gets \frac{1}{n}\sum_{t=1}^n\operatorname{Var}(\Pi_{\mathcal{S}^{\operatorname{inv}}}X_t)^{\dagger}\operatorname{Cov}(\Pi_{\mathcal{S}^{\operatorname{inv}}}X_t, Y_t )$ \Comment{{\color{gray}see Prop.~\ref{prop:invariance}, Lemmas~\ref{lem:time_var_parameter_separation},\ref{lem:subspace_maximization}}}
    \State $\delta_{n}^{\operatorname{res}}\gets U^{\operatorname{res}}\operatorname{Var}((U^{\operatorname{res}})^{\top}X_{n})^{-1} \operatorname{Cov}((U^{\operatorname{res}})^{\top}X_{n}, Y_{n}-X_{n}^{\top}\beta^{\operatorname{inv}})$  \Comment{{\color{gray}see Eq.~\eqref{eq:adaptation_parameter_estimator}}}
    \end{algorithmic}
\end{algorithm}

\begin{remark}[Changes of $\gamma_{0,t}$ and $\Sigma_t$ over time]
\label{rem:time_changes}
For the ISD procedure to work in practice, we assume that the covariance matrices $\Sigma_t$ are approximately constant in small time windows (e.g., smoothly changing or piece-wise constant). This is required to obtain accurate estimates of $\Sigma_t$, which we then (approximately) joint block diagonalize (see Section~\ref{sec:approx_jbd_estimation}). Similarly, we assume that $\gamma_{0,t}$ is approximately constant in the adaptation window, in order to perform the adaptation step (see Setting~\ref{set:definition}). In the historical data, however, $\gamma_{0,t}$ can vary arbitrarily fast in the residual subspace ($\beta^{\operatorname{inv}}$ remains constant, of course). 
 Intuitively, even if $\gamma_{0,t}$ does not change in a structured way, its projection $\Pi_{\mathcal{S}^{\operatorname{inv}}}\gamma_{0,t}$ on the invariant subspace remains constant across time points. We provide an example in which $\gamma_{0,t}$ is  
quickly varying in Appendix~\ref{sec:fast_gamma_changes}.
\end{remark}

\section{Analysis of the two ISD tasks: zero-shot generalization and time adaptation}
\label{sec:task_description}
We now analyze the two prediction tasks that can be solved with the \isd{} framework introduced in the previous sections, namely (i) the \emph{zero-shot task} and (ii) the \emph{adaptation task}.
We consider Setting~\ref{set:definition} and assume for (i) that only the historical data is available, while for (ii) we assume to also have access to the adaptation data.
Throughout the remainder of this section we assume that the invariant and residual subspaces are known (they can be computed from the joint distributions), which simplifies the theoretical analysis.

\subsection{Zero-shot task}
\label{sec:zero_shot_learning}
We start by analyzing the zero-shot task, in which no adaptation data are observed, but only historical data and $X_{t^*}$.
Under the generalization assumption (Assumption~\ref{ass:generalization}), $\mathcal{S}^{\operatorname{inv}}$ is opt-invariant on $\mathbb{N}$ and we can characterize all possible models defined as in \eqref{eq:model_def} by the possible variations of the true time-varying parameter in the residual space $\mathcal{S}^{\operatorname{res}}$, or, equivalently, by all possible values of $\gamma_{0,t}-\beta^{\operatorname{inv}}\in\mathcal{S}^{\operatorname{res}}$. We then obtain that $\beta^{\operatorname{inv}}$ is worst case optimal in the following sense. 
\begin{theorem} 
\label{thm:inv_eff_maximin_optimal}
Under Assumptions~\ref{ass:uniqueness_time_inv_sp} and \ref{ass:generalization} it holds for all $t\in\mathbb{N}$ that
    \begin{equation*}
    \beta^{\operatorname{inv}} = \argmax_{\beta\in\mathbb{R}^p} \inf_{\substack{\gamma_{0,t}\in\mathbb{R}^p:\\ \gamma_{0,t}-\beta^{\operatorname{inv}}\in\mathcal{S}^{\operatorname{res}}}} \operatorname{\Delta Var}_t(\beta).
\end{equation*}
\end{theorem}
We further obtain from  Proposition~\ref{prop:invariance}~$(iii)$ that, under an additional condition, $\beta^{\operatorname{inv}}$ is worst case optimal among all time-invariant parameters. This characterization of $\beta^{\operatorname{inv}}$ allows for a direct comparison with the maximin effect by \citet{meinshausen2015maximin}, which we report in more detail in Remark~\ref{rem:maximin}. 
In absence of further information on $\delta_{t^*}^{\operatorname{res}}$, Theorem~\ref{thm:inv_eff_maximin_optimal} suggests to use an estimate of $\hat{\beta}^{\operatorname{inv}}$ of $\beta^{\operatorname{inv}}$ to predict $Y_{t^*}$, i.e., $\hat{Y}_{t^*} = X_{t^*}^{\top}\hat{\beta}^{\operatorname{inv}}$.
\begin{remark}[Relation to maximin]
\label{rem:maximin}
    The maximin framework introduced by \citet{meinshausen2015maximin} considers a linear regression model with varying coefficients, where the variations do not necessarily happen in a structured way, e.g., in time. 
    Translated to our notation and restricting to time-based changes, their model can be expressed for $t\in[n]$ as
    \begin{equation*}
        Y_t = X_t^{\top}\gamma_{0,t} + \epsilon_t,
    \end{equation*}
    where $\mathbb{E}[\epsilon_t\mid X_t]=0$ and the covariance matrix 
    $\Sigma \coloneqq\operatorname{Var}(X_t)$
    does not vary with $t$. 
    The maximin estimator maximizes the explained variance in the worst (most adversarial) scenario that is covered in the training data;  it is defined as
    \begin{equation*}
        \beta^{\operatorname{mm}}\coloneqq \argmax_{\beta\in\mathbb{R}^p}\min_{t\in[n]}\operatorname{\Delta Var}_t(\beta).
    \end{equation*}
    The maximin estimator guarantees that, for all $t\in[n]$, $\operatorname{Cov}(Y_t-X_t^{\top}\beta^{\operatorname{mm}}, X_t^{\top}\beta^{\operatorname{mm}})\ge 0$ and therefore $\operatorname{\Delta Var}_t(\beta^{\operatorname{mm}})\ge 0$ \citep[see][Equation (8)]{meinshausen2015maximin}. This means that the maximin relies on a weaker notion of invariance that only requires the left-hand side of \eqref{eq:invariant_constraint} to be non-negative instead of zero. This implies that $\beta^{\operatorname{inv}}$ is in general more conservative than $\beta^{\operatorname{mm}}$ in the sense that, for all $t\in[n]$, $\operatorname{\Delta Var}_t(\beta^{\operatorname{mm}})\ge\operatorname{\Delta Var}_t(\beta^{\operatorname{inv}}).$
By finding the invariant and residual subspaces, we determine the domain in which $\gamma_{0,t}$ varies and assume (Assumption~\ref{ass:generalization}) that this does not change even at unobserved time points. The parameter $\beta^{\operatorname{inv}}$ is then worst case optimal over this domain, and guarantees that the explained variance remains positive even for scenarios that are more adversarial than the ones observed in $[n]$, i.e., such that $\operatorname{\Delta Var}_s(\beta^{\operatorname{mm}})<\min_{t\in[n]}\operatorname{\Delta Var}_t(\beta^{\operatorname{mm}})$ for some  $s\in\{n+1, n+2, \dots\}$ (see, for example, the results of the experiment described in Section~\ref{sec:invariant_exp} and shown in Figure~\ref{fig:xplvar_prediction_comparison}). 
Furthermore, the decomposition allows for time adaptation (see Section~\ref{sec:time_adaptation}), which would not be possible starting from the maximin effect. 
\end{remark}
 
\subsection{Adaptation task}
\label{sec:time_adaptation_est}
We now consider the adaptation task in which, additionally to the historical data, adaptation data are available. 
Adaptation data can be used to define an estimator for the residual component $\delta_{t^*}^{\operatorname{res}}$, which we denote by $\hat{\delta}_{t^*}^{\operatorname{res}}$. This, together with an estimator $\hat{\beta}^{\operatorname{inv}}$ for $\beta^{\operatorname{inv}}$ fitted on the historical data gives us an estimator $\hat{\gamma}^{\operatorname{ISD}}_{t^*}\coloneqq \hat{\beta}^{\operatorname{inv}} + \hat{\delta}_{t^*}^{\operatorname{res}}$ for the true time-varying parameter. We compare this estimator with a generic estimator $\hat{\gamma}_{t^*}$ for $\gamma_{0,t^*}$ which uses only the adaptation data.

To do so, we consider the minimax lower bound provided by \citet{mourtada2022exact}[Theorem 1] for the expected squared prediction error of an estimator $\hat{\gamma}$ computed using $n$ i.i.d.\ observations of a random covariate vector $X\in\mathbb{R}^p$ and of the corresponding response $Y$. It is given by
\begin{equation}
\label{eq:minimax_lower_bound}
    \inf_{\hat{\gamma}}\sup_{P\in\mathcal{P}}  \mathbb{E}_P[(X^{\top}(\hat{\gamma}-\gamma))^2] \geq \sigma^2\tfrac{p}{n},
\end{equation}
where $\mathcal{P}$ is the class of distributions over $(X, Y)$ such that $Y=X^{\top}\gamma+\epsilon$, $\mathbb{E}[\epsilon\mid X]=0$ and $\mathbb{E}[\epsilon^2\mid X]<\sigma^2$.
It follows from \eqref{eq:minimax_lower_bound} that, for a generic estimator $\hat{\gamma}_{t^*}$ of $\gamma_{0,t^*}$ based on the adaptation data alone, we can expect at best
to achieve a prediction error of $\sigma_{\operatorname{ad}}^2\frac{p}{m}$, where  $\sigma_{\operatorname{ad}}^2$ is the variance of $\epsilon_t$ for $t\in\mathcal{I}^{\operatorname{ad}}$ (which is assumed to be constant).
We can improve on this if we allow the estimator $\hat{\gamma}_{t^*}$ to also depend on the historical data. To see this, observe that we can always decompose $\hat{\gamma}_{t^*}=\hat{\beta}+\hat{\delta}_{t^*}$ with $\hat{\beta}\in\mathcal{S}^{\operatorname{inv}}$ and $\hat{\delta}_{t^*}\in\mathcal{S}^{\operatorname{res}}$. Moreover, under Assumption~\ref{ass:generalization} we can split the expected prediction error of $\hat{\gamma}_{t^*}$ at $t^*$ accordingly as
\begin{equation*}
    \mathbb{E}[(X_{t^*}^{\top}(\hat{\gamma}_{t^*}-\gamma_{0,t^*}))^2]=\mathbb{E}[((\Pi_{\mathcal{S}^{\operatorname{inv}}}X_{t^*})^{\top}(\hat{\beta}-\beta^{\operatorname{inv}}))^2] + \mathbb{E}[((\Pi_{\mathcal{S}^{\operatorname{res}}}X_{t^*})^{\top}(\hat{\delta}_{t^*}-\delta^{\operatorname{res}}_{t^*}))^2].
\end{equation*}
Then, $\hat{\beta}$ represents an estimator for $\beta^{\operatorname{inv}}$ and can be computed on historical data\footnote{In principle, an estimator for $\beta^{\operatorname{inv}}$ could be computed using both historical and adaptation data. However, the \isd{} procedure is motivated by scenarios in which the size of historical data is very large (and $n\gg m$): this means that it could be computationally costly to update the estimate for the invariant component every time new adaptation data are available, without a significant gain in estimation accuracy. For this reason, we only consider estimators for $\beta^{\operatorname{inv}}$ that use historical data.}, whereas $\hat{\delta}_{t^*}$ estimates $\delta^{\operatorname{res}}_t$ and is based on adaptation data: by decomposing $\hat{\gamma}_{t^*}$ in this way, the best prediction error we can hope for is of the order $\frac{\operatorname{dim}(\mathcal{S}^{\operatorname{inv}})}{n}+\frac{\operatorname{dim}(\mathcal{S}^{\operatorname{res}})}{m}$.
In Section~\ref{sec:estimation}, we prove that $\hat{\gamma}^{\operatorname{\isd}}_{t^*}$ indeed achieves this bound in Theorem~\ref{thm:empirical_explained_variance}. If the invariant subspace is non-degenerate and therefore $\operatorname{dim}(\mathcal{S}^{\operatorname{res}})<p$, and $n$ is sufficiently large, this implies that $\hat{\gamma}^{\operatorname{\isd}}_{t^*}$ has better finite sample performance than estimators based on the adaptation data alone.

\section{\isd{} estimator and its finite sample generalization guarantee}
\label{sec:estimation}
We now construct an empirical estimation procedure for the 
 \isd{} framework, based on the results of Section~\ref{sec:time_invariant_subspace_effect} and on Algorithm~\ref{alg:pop_estimator} described in Section~\ref{sec:time_adaptation}. 
Throughout this section, we assume that Assumptions~\ref{ass:uniqueness_time_inv_sp} and \ref{ass:generalization} are satisfied. 

We assume that we observe both historical and adaptation data as in the adaptation task. We use the historical data to first estimate the decomposition of $\mathbb{R}^p$ into $\mathcal{S}^{\operatorname{inv}}$ and $\mathcal{S}^{\operatorname{res}}$, employing a joint block diagonalization algorithm (Section~\ref{sec:subspace_decomposition_est}). We then use the resulting decomposition to estimate $\beta^{\operatorname{inv}}$. Finally, we use the adaptation data to construct an estimator for $\delta^{\operatorname{res}}_t$ in Section~\ref{sec:residual_effect_estimation}. In Section~\ref{sec:finite_sample_generalization_guarantees} we then show the advantage of separating the optimization as in \eqref{eq:lin_subspace_separation} to estimate $\gamma_{0,t^*}$ at the previously unobserved time point $t^*$, by providing finite sample guarantees for the \isd{} estimator.

 \subsection{Estimating the subspace decomposition}
 \label{sec:subspace_decomposition_est}
  \subsubsection{Approximate joint block diagonalization}
  \label{sec:approx_jbd_estimation}
We first need to find a good estimator for the covariance matrices $\Sigma_t$. Since only one 
observation $(X_t, Y_t)$ is available at each time step $t\in[n]$, some further assumptions are needed 
about how $\Sigma_t$ varies over time.
Here, we assume that $\Sigma_t$ varies smoothly with $t\in[n]$ and is therefore approximately constant in small time windows.
We can then consider a rolling window approach, i.e., consider $K$ windows in $[n]$ of length $w\ll n$ over which the constant approximation is deemed valid, and for the $k$-th time window, $k\in\{1,\dots,K\}$, take the sample covariance $\hat{\Sigma}_k$ as an estimator for $\Sigma_t$ in such time window.
 
 Given the set of estimated covariance matrices $\{\hat{\Sigma}_k\}_{k=1}^K$, we now need to estimate an orthogonal transformation $\hat{U}$ that approximately joint block diagonalizes them. 
  We provide an overview of joint block diagonalization methods in Section~\ref{rem:ajbd} in the appendix. In our simulated settings, we solve the approximate joint block diagonalization (AJBD) problem via approximate joint diagonalization (AJD), since we found this approach to represent a good trade off between computational complexity and accuracy. More in detail, similarly to what is proposed by~\citet{tichavsky2012algorithms}, we start from the output matrix $V\in\mathbb{R}^{p\times p}$ of the \texttt{uwedge} algorithm by~\citet{tichavsky2008fast}, which solves AJD for the set $\{\hat{\Sigma}_k\}_{k=1}^K$, i.e., it is such that for all $k\in\{1,\dots,K\}$ the matrix $V^{\top}\hat{\Sigma}_kV$ is approximately diagonal. We then use the off-diagonal elements of the approximately diagonalized covariance matrices to identify the common diagonal blocks.
As described more in detail in Section~\ref{rem:ajbd}, this is achieved by finding an appropriate permutation matrix $P\in\mathbb{R}^{p\times p}$ for the columns of $V$ such that for all $k\in\{1,\dots,K\}$ the matrix $(VP)^{\top}\hat{\Sigma}_k(VP)$ is approximately joint block diagonal.
The estimated irreducible joint block diagonalizer is then given by $\hat{U}=VP$. We denote by $q^{\hat{U}}_{\max}$ the number of estimated diagonal blocks and  by $\{\hat{\mathcal{S}_j}\}_{j=1}^{q^{\hat{U}}_{\max}}$ the estimated irreducible orthogonal and $(X_t)_{t\in[n]}$-decorrelating partition.

\subsubsection{Estimating the invariant and residual subspaces}
\label{sec:subspaces_estimation}
We now estimate the invariant and residual subspaces using the estimated irreducible joint block diagonalizer $\hat{U}$. 
To do so, we first estimate the true time-varying parameter $\gamma_{0,t}$ using similar considerations as the ones made  in Section~\ref{sec:approx_jbd_estimation}.
We assume for example that $\gamma_{0,t}$ is approximately constant in small windows (for simplicity, we consider the same $K$ windows defined in Section~\ref{sec:approx_jbd_estimation}). This assumption is helpful to define the estimation procedure, but is not strictly necessary for the \isd{} framework to work. We show in Appendix~\ref{sec:fast_gamma_changes} that the procedure can still work if this assumption is violated.
We then compute the regression coefficient $\hat{\gamma}_{k}$ of $\mathbf{Y}_k$ on $\mathbf{X}_k$, where $\mathbf{Y}_k\in\mathbb{R}^{w\times 1}$ and $\mathbf{X}_k\in\mathbb{R}^{w\times p}$ are the observations in the $k$-th time window\footnote{We have so far omitted the intercept in our linear model, but it can be included by adding a constant term to $X_t$ when estimating the linear parameters. We explicitly show how to take the intercept into account in Section~\ref{sec:isd_algorithm_pseudocode} in the appendix.}.
 We use the estimates $\hat{\gamma}_k$ to determine which of the subspaces identified by $\hat{U}$ are opt-invariant. For all $j\in\{1,\dots,q^{\hat{U}}_{\max}\}$, we take the sets of indices $S_j$ as defined in Proposition~\ref{prop:irreducible_orthogonal_partition}, and consider the estimated orthogonal projection matrices $\Pi_{\hat{\mathcal{S}}_j} = \hat{U}^{S_j}(\hat{U}^{S_j})^{\top}$ from $\mathbb{R}^p$ onto the $j$-th subspace $\hat{\mathcal{S}_j}$ of the estimated irreducible orthogonal partition. 
It follows from Lemma~\ref{lem:time_var_parameter_separation} that we can find the opt-invariant subspaces by checking for all $j\in\{1,\dots,q^{\hat{U}}_{\max}\}$ whether $\Pi_{\hat{\mathcal{S}}_j}\hat{\gamma}_{k}$ remains approximately constant for $k\in\{1,\dots, K\}$.
To do so, let $\hat{\gamma}\coloneqq (\sum_{k=1}^K\operatorname{Var}(\hat{\gamma}_k)^{-1})^{-1}\sum_{k=1}^K \operatorname{Var}(\hat{\gamma}_k)^{-1}\hat{\gamma}_{k}$ be the average of the estimated regression coefficients inversely weighted by their variance. We further use that, by Lemma~\ref{lem:opt_invariant_projection}, 
if $\hat{\mathcal{S}}_j$
is opt-invariant on $[n]$ then the weighted average of the (approximately constant) projected regression coefficient $\Pi_{\hat{\mathcal{S}}_j}\hat{\gamma}_{k}$, i.e., $\Pi_{\hat{\mathcal{S}}_j}\hat{\gamma}$, approximately satisfies the time-invariance constraint \eqref{eq:invariant_constraint}. 
This approach is motivated by Proposition~\ref{prop:invariance}~$(iii)$. If the corresponding
assumption cannot be assumed to hold, other methods can alternatively be used to determine whether $\Pi_{\hat{\mathcal{S}}_j}\hat{\gamma}_{k}$ is constant, e.g., checking its gradient or variance.
In \eqref{eq:invariant_constraint}, we can equivalently consider the correlation in place of the covariance, that is, $\operatorname{corr}(Y_t-X_t^{\top}\beta, X_t^{\top}\beta)=0$: an estimate of this correlation allows us to obtain a normalized measure of \eqref{eq:invariant_constraint} that is comparable across different experiments. Formally, we consider for all $k\in\{1, \dots, K\}$ and for all $j\in\{1,\dots,q^{\hat{U}}_{\max}\}$
\begin{equation*}
    \hat{c}_k^j \coloneqq \widehat{\operatorname{Corr}}(\mathbf{Y}_k - \mathbf{X}_k(\Pi_{\hat{\mathcal{S}}_j}\hat{\gamma}), \mathbf{X}_k(\Pi_{\hat{\mathcal{S}}_j}\hat{\gamma}))
\end{equation*}
and check, for all $j\in\{1,\dots,q^{\hat{U}}_{\max}\}$, whether
\begin{equation}
\label{eq:opt-invariant_test}
    \frac{1}{K}\sum_{k=1}^K \left| \hat{c}_k^j \right| \le \lambda
\end{equation}
 for some small threshold $\lambda\in [0, 1]$. The threshold $\lambda$ can be chosen, for example, using cross-validation (more details are provided in Section~\ref{rem:constant_threshold_selection} in the appendix). 
An estimator of the invariant and residual subspaces is then given by
\begin{equation}
\label{eq:time-invariant_subspace_effect_estimator}
    \hat{\mathcal{S}}^{\operatorname{inv}} = \bigoplus_{\substack{j\in\{1,\dots,q^{\hat{U}}_{\max}\}:\\ \eqref{eq:opt-invariant_test} \text{ is satisfied}}} \hat{\mathcal{S}}_j,  \quad \hat{\mathcal{S}}^{\operatorname{res}} = \bigoplus_{\substack{j\in\{1,\dots,q^{\hat{U}}_{\max}\}:\\ \eqref{eq:opt-invariant_test} \text{ is not satisfied}}} \hat{\mathcal{S}}_j,
\end{equation}
where we approximate opt-invariance with the inequality \eqref{eq:opt-invariant_test} being satisfied.

\subsection{Estimating the invariant and residual components}
\label{sec:residual_effect_estimation}
Let $\hat{U}^{\operatorname{inv}}$ and $\hat{U}^{\operatorname{res}}$ be the submatrices of $\hat{U}$ whose columns span $\hat{\mathcal{S}}^{\operatorname{inv}}$ and $\hat{\mathcal{S}}^{\operatorname{res}}$, respectively. We propose to estimate $\beta^{\operatorname{inv}}$ using the following plug-in estimator for~\eqref{eq:invariant_component_pop_estimator}
\begin{align}
        \hat{\beta}^{\operatorname{inv}} & \coloneqq \hat{U}^{\operatorname{inv}} ((\hat{U}^{\operatorname{inv}})^{\top}\mathbf{X}^{\top}\mathbf{X}\hat{U}^{\operatorname{inv}}) ^{-1} (\hat{U}^{\operatorname{inv}})^{\top}\mathbf{X}^{\top}\mathbf{Y}.
    \label{eq:invariant_eff_estimator}
\end{align}

We consider now the new observation of the covariates $X_{t^*}$ at time $t^*$ and the adaptation data $(X_t, Y_t)_{t\in\mathcal{I}^{\operatorname{ad}}}$ introduced in Setting~\ref{set:definition}, and denote by $\mathbf{X}^{\operatorname{ad}}\in\mathbb{R}^{m\times p}$ and $\mathbf{Y}^{\operatorname{ad}}\in\mathbb{R}^{m\times 1}$ the matrices containing this adaptation data.
Similarly to $\hat{\beta}^{\operatorname{inv}}$, using \eqref{eq:adaptation_parameter_estimator} we obtain the following plug-in estimator for $\delta^{\operatorname{res}}_{t^*}$
\begin{align}
    \hat{\delta}^{\operatorname{res}}_{t^*} 
     & \coloneqq \hat{U}^{\operatorname{res}}((\hat{U}^{\operatorname{res}})^{\top}(\mathbf{X}^{\operatorname{ad}})^{\top}\mathbf{X}^{\operatorname{ad}}\hat{U}^{\operatorname{res}})^{-1} (\hat{U}^{\operatorname{res}})^{\top}(\mathbf{X}^{\operatorname{ad}})^{\top}(\mathbf{Y}^{\operatorname{ad}}-\mathbf{X}^{\operatorname{ad}}\hat{\beta}^{\operatorname{inv}}). \label{eq:residual_effect_estimator}
\end{align}
We can now define the \isd{} estimator for the true time-varying parameter at $t^*$ as
\begin{equation}
    \hat{\gamma}^{\operatorname{\isd}}_{t^*} \coloneqq \hat{\beta}^{\operatorname{inv}} + \hat{\delta}^{\operatorname{res}}_{t^*}. \label{eq:true_parameter_empirical_estimator}
\end{equation}
A prediction of the response $Y_{t^*}$ is then given by $\hat{Y}_{t^*}=X_{t^*}^{\top} \hat{\gamma}_{t^*}^{\operatorname{\isd}}$.

We provide the pseudocode summarizing the whole \isd{} estimation procedure in the Appendix~\ref{sec:isd_algorithm_pseudocode}.

\begin{example}[continues=ex:running_ex_2d]

    Assume that in  Example~\ref{ex:ex:running_ex_2d} the true time-varying parameter at time points $t\in\{n+1, n+2, \dots\}$ is given by 
    \begin{equation*}
        \gamma_{0,t}^{\operatorname{ad}} = \begin{bmatrix}
            1+ 0.5\sqrt{3}-1.5\sqrt{3}\frac{t-n}{T-n}\sin^2 (\frac{t-n}{T-n} +1) \\
            \sqrt{3}-0.5 +1.5\frac{t-n}{T-n}\sin^2 (\frac{t-n}{T-n} +1)
        \end{bmatrix}
    \end{equation*}
    for some $T\in\mathbb{N}$.
    We can verify that $\Pi_{\mathcal{S}_2}\gamma_{0,t}^{\operatorname{ad}}=\Pi_{\mathcal{S}_2}\gamma_{0,t} = \beta^{\operatorname{inv}}$, and therefore the invariant and residual subspaces defined on $[n]$ generalize to $\{n+1, n+2, \dots \}$ and Assumption~\ref{ass:generalization} is satisfied. 
    Moreover, we obtain that for $t\in\{n+1,\dots\}$ the residual component expressed under the irreducible orthogonal partition basis is
    \begin{equation*}
        U^{\top}\delta^{\operatorname{res}}_t = \begin{bmatrix}
        1-3\frac{t-n}{T-n}\sin^2 (\frac{t-n}{T-n} +1) \\ 0
    \end{bmatrix}.
    \end{equation*}
    The first entry now takes values in $[-1.5, 1]$, which is disjoint from the range of values of the first coordinate observed in $[n]$ (which was $[1, 3]$). 
    
    We now take $T=350$ and consider an online setup in which we sequentially observe $X_t^*$ at a new time point $t^*\in\{n+1,\ldots\}$, 
    and assume that $Y_t$ is observed until $t=t^*-1$. We take as historical data the observations on $[n]$, and consider as adaptation data the observations in windows $\mathcal{I}^{\operatorname{ad}} = \{t^*-m,\ldots, t^*-1\}$ of length $m=16$. After estimating $\beta^{\operatorname{inv}}$ on historical data, we use the adaptation data to estimate $\delta^{\operatorname{res}}_{t^*}$ and the OLS solution $\gamma^{\operatorname{OLS}}_{t^*}$. We repeat this online step $350$ times (each time increasing $t^*$ by one). Figure~\ref{fig:2d_ex} shows the results of this experiment.
\end{example}
\subsection{Finite sample generalization guarantee}
\label{sec:finite_sample_generalization_guarantees}
Considering the setting described in Section~\ref{sec:time_adaptation}, we now compare the \isd{} estimator in \eqref{eq:true_parameter_empirical_estimator} with the OLS estimator computed on $\mathcal{I}^{\operatorname{ad}}$
, i.e., $\hat{\gamma}_{t^*}^{\operatorname{OLS}} \coloneqq ((\mathbf{X}^{\operatorname{ad}})^{\top}\mathbf{X}^{\operatorname{ad}})^{-1} (\mathbf{X}^{\operatorname{ad}})^{\top}\mathbf{Y}^{\operatorname{ad}}$. We assume that we are given the (oracle) subspaces $\mathcal{S}^{\operatorname{inv}}$ and $\mathcal{S}^{\operatorname{res}}$, and consider the expected explained variance at $t^*$ of $\hat{\gamma}_{t^*}^{\operatorname{\isd}}$ and $\hat{\gamma}_{t^*}^{\operatorname{OLS}}$.
More in detail, the expected explained variance at $t^*$ of an arbitrary estimator $\hat{\gamma}$ of $\gamma_{0,t^*}$ is given by 
\begin{equation}
\label{eq:expected_explained_variance}
    \mathbb{E}[\operatorname{\Delta Var}_{t^*}(\hat{\gamma})] \coloneqq \mathbb{E}[\operatorname{Var}(Y_{t^*})-\operatorname{Var}(Y_{t^*}-X_{t^*}^{\top}\hat{\gamma}\mid \hat{\gamma})].
\end{equation}
Evaluating~\eqref{eq:expected_explained_variance} allows us to obtain a measure of the prediction accuracy of $\hat{\gamma}$: the higher the expected explained variance, the more predictive the estimator is. This also becomes clear by isolating in \eqref{eq:expected_explained_variance} the term $\mathbb{E}[\operatorname{Var}(Y_{t^*}-X_{t^*}^{\top}\hat{\gamma}\mid \hat{\gamma})]$, which represents the mean square prediction error obtained by using $\hat{\gamma}$ (see Remark~\ref{rem:expl_var_and_mse}). Our goal is to show that the explained variance by $\hat{\gamma}_{t^*}^{\operatorname{\isd}}$ is always greater than that of the OLS estimator.  
\begin{theorem}
\label{thm:empirical_explained_variance}
    Assume Assumption~\ref{ass:generalization} and that, in model \eqref{eq:model_def}, $\gamma_{0,t}$ and the variances of $X_t$ and $\epsilon_t$ do not change with respect to $t\in\mathcal{I}^{\operatorname{ad}}\cup t^*$, and denote them by $\gamma_{0,t^*}$, $\Sigma_{t^*}$ and $\sigma_{\operatorname{ad}}^{2}$, respectively. 
    Moreover, let $c,\sigma_{\epsilon,\max}^{2}>0$ be constants such that for all $n\in\mathbb{N}$ it holds that $\sigma_{\epsilon,\max}^{2}\geq \max_{t\in[n]}\operatorname{Var}(\epsilon_t)$ and for all $m\in\mathbb{N}$ with $m\geq p$, $\lambda_{\min}(\frac{1}{m}(\mathbf{X}^{\operatorname{ad}})^{\top}\mathbf{X}^{\operatorname{ad}})\ge c$ almost surely, where $\lambda_{\min}(\cdot)$ denotes the smallest eigenvalue.
    Further assume that the invariant and residual subspaces $\mathcal{S}^{\operatorname{inv}}$ and $\mathcal{S}^{\operatorname{res}}$ are known. Then, there exist $C_{\operatorname{inv}}, C_{\operatorname{res}} > 0$ constants such that for all $n, m\in\mathbb{N}$ with $n,m\geq p$ it holds that
    \begin{align*}
        \operatorname{MSPE}(\hat{\gamma}_{t^*}^{\operatorname{\isd}}) & \coloneqq\mathbb{E}[(X_{t^*}(\hat{\gamma}_{t^*}^{\operatorname{\isd}}-\gamma_{0,t^*}))^2] 
        \le\sigma_{\epsilon,\max}^2\tfrac{\operatorname{dim}(\mathcal{S}^{\operatorname{inv}}) }{n} C_{\operatorname{inv}} + \sigma_{\operatorname{ad}}^2\tfrac{\operatorname{dim}(\mathcal{S}^{\operatorname{res}}) }{m} C_{\operatorname{res}}.
    \end{align*}
     Furthermore, for all $n, m\in\mathbb{N}$ with $n,m\geq p$ it holds that
    \begin{equation*}
   \operatorname{MSPE}(\hat{\gamma}_{t^*}^{\operatorname{OLS}})-\operatorname{MSPE}(\hat{\gamma}_{t^*}^{\operatorname{\isd}}) \ge \sigma_{\operatorname{ad}}^2\tfrac{\operatorname{dim}(\mathcal{S}^{\operatorname{inv}})}{m} - \sigma_{\epsilon,\max}^2\tfrac{ \operatorname{dim}(\mathcal{S}^{\operatorname{inv}}) }{n}C_{\operatorname{inv}}.
\end{equation*}
\end{theorem}
From the proof of Theorem~\ref{thm:empirical_explained_variance} it follows that $C_{\operatorname{res}}$ 
can be chosen close to $1$ if we only consider sufficiently large $m$. Moreover, because $\operatorname{MSPE}(\hat{\gamma}_{t^*}^{\operatorname{OLS}})-\operatorname{MSPE}(\hat{\gamma}_{t^*}^{\operatorname{\isd}}) =  \mathbb{E}[\operatorname{\Delta Var}_{t^*}(\hat{\gamma}^{\operatorname{\isd}}_{t^*})] - \mathbb{E}[\operatorname{\Delta Var}_{t^*}(\hat{\gamma}^{\operatorname{OLS}}_{t^*})]$, Theorem~\ref{thm:empirical_explained_variance} implies that if $\operatorname{dim}(\mathcal{S}^{\operatorname{inv}})\ge 1$ and $n$ sufficiently large then
\begin{equation*}
     \mathbb{E}[\operatorname{\Delta Var}_{t^*}(\hat{\gamma}^{\operatorname{\isd}}_{t^*})] > \mathbb{E}[\operatorname{\Delta Var}_{t^*}(\hat{\gamma}^{\operatorname{OLS}}_{t^*})].
\end{equation*}
The first term in the difference between the expected explained variances (or MSPEs) depends on the dimensions of the invariant subspace and of the time-adaptation window $\mathcal{I}^{\operatorname{ad}}$:
 there is a higher gain in using $\hat{\gamma}_{t^*}^{\operatorname{\isd}}$ instead of $\hat{\gamma}_{t^*}^{\operatorname{OLS}}$ if the dimension of $\mathcal{S}^{\operatorname{inv}}$ is large and only a small amount of time-points are available are available in the adaptation data.

\section{Experiments}
\label{sec:simulations}
To show the effectiveness of the \isd{} framework we report the results of two simulation experiments and one real data experiment. 
The first simulation experiment evaluates the estimation accuracy of the invariant component $\hat{\beta}^{\operatorname{inv}}$ 
for increasing sample size $n$ of the historical data. The second simulation experiment
compares the predictive accuracy of $\hat{\gamma}^{\operatorname{\isd}}_t$ and $\hat{\gamma}^{\operatorname{OLS}}_t$ for different sizes $m$ of the adaptation dataset, to empirically investigate the dependence of the MSPE difference on the size of the adaptation data shown in Theorem~\ref{thm:empirical_explained_variance}.
 
In both simulation experiments we let the dimension of the covariates be $p=10$, and $\operatorname{dim}(\mathcal{S}^{\operatorname{inv}})=7$, $\operatorname{dim}(\mathcal{S}^{\operatorname{res}})=3$, and generate data as follows. We sample a random orthogonal matrix $U$, and sample the covariates $X_t$ from a normal distribution with zero mean and covariance matrix $U\tilde{\Sigma}_tU^{\top}$, where $\tilde{\Sigma}_t$ is a block-diagonal matrix with four blocks of dimensions $2, 4, 3$ and $1$, and random entries that change $10$ times in the observed time horizon $n$. We take as true time-varying parameter the rotation by $U$ of the parameter with constant entries equal to $0.2$ (we set these entries all to the same value for simplicity, but they need not to be equal in general)
corresponding to the blocks of sizes $4$ and $3$, and time-varying entries---corresponding to the blocks of sizes $2$ and $1$---equal to $(1-1.5t \sin^2 (it/n+i))/n$, where $i\in \{2, ..., 8\}$ is the entry index (the values of these coefficients range between $-0.25$ and $1$). The noise terms $\epsilon_t$ are sampled 
i.i.d.\ from a normal distribution with zero mean and variance $\sigma_{\epsilon_t}^2=0.64$. 
The dimensions of the subspaces, $\mathcal{S}^{\operatorname{inv}}$ and $\mathcal{S}^{\operatorname{res}}$, and the true time-varying parameter $\gamma_{0,t}$ are chosen to ensure that, in the historical data, $X_t^{\top}\beta^{\operatorname{inv}}$ and $X_t^{\top}\delta_t^{\operatorname{res}}$ explain approximately half of the variance of $Y_t$ each. This choice allows for better visualization in the described experiments. 

In Section~\ref{sec:light_tunnel} we repeat the same experiments on real data coming from a controlled physical system (light tunnel) developed by~\citet{gamella2025causal}. The data used in this experiment will be soon available at \url{https://github.com/juangamella/causal-chamber}.
The code for the presented experiments 
is available at \url{https://github.com/mlazzaretto/Invariant-Subspace-Decomposition}. The implementation of the \texttt{uwedge} algorithm is taken from the Python package \url{https://github.com/sweichwald/coroICA-python} developed by~\citet{pfister2019robustifying}.

\subsection{Invariant decomposition and zero-shot prediction}
\label{sec:invariant_exp}
We first estimate the time invariant parameter $\beta^{\operatorname{inv}}$ for different sample sizes $n$ of the historical data. We consider $n\in\{500, 1000, 2500, 4000, 6000\}$, and repeat the experiment $20$ times for each $n$.  
To 
compute 
$\hat \beta^{\operatorname{inv}}$, we use 
$K=25$ equally distributed windows of length $n/8$ (see Section~\ref{sec:estimation}). Figure~\ref{fig:mse_inv} shows that the mean squared error (MSE) $\|\beta^{\operatorname{inv}}-\hat{\beta}^{\operatorname{inv}}\|_2^2$ converges to zero for increasing  values of $n$. 

We then consider a separate time window of $250$ observations in which the value of the time-varying coefficients (before the transformation using $U$) is set to $-1$. 
\begin{figure}[!bp]
    \centering
    \includegraphics[width=0.5\textwidth]{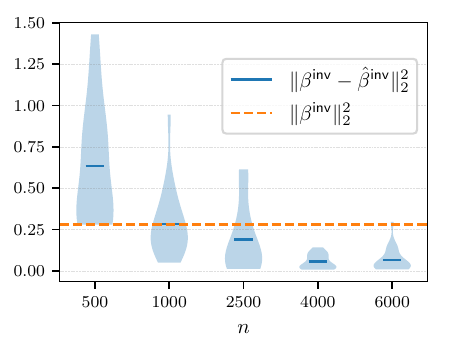}\\
    \caption{MSE of $\hat{\beta}^{\operatorname{inv}}$ for increasing size of the historical data $n$ (see Section~\ref{sec:invariant_exp}). For larger values of $n$, the estimation of the invariant subspace decomposition becomes more precise and leads to smaller errors in the estimated invariant component $\hat{\beta}^{\operatorname{inv}}$.}
    \label{fig:mse_inv}
\end{figure}
We use these observations to test the zero-shot predictive capability of the estimated invariant component, i.e., they can be seen as realizations of the variable $X_{t^*}$ introduced in Setting~\ref{set:definition} (we refer to this window as test data).
We compare the predictive performance of the parameter $\hat{\beta}^{\operatorname{inv}}$ on the historical and test data with that of the oracle invariant parameter $\beta^{\operatorname{inv}}$, the maximin effect $\hat{\beta}^{\operatorname{mm}}$ \citep[computed using the magging estimator proposed 
by~][]{buhlmann2015magging}, and the OLS solution $\hat{\beta}^{\operatorname{OLS}}$, both computed using the historical data. We show in Figure~\ref{fig:sim_inv} the results in terms of the $R^2$ coefficient, given by $R^2 = \frac{\sum_{t=1}^n (\widehat{\operatorname{Var}}(Y_t)-\widehat{\operatorname{Var}}(Y_t-X_t^{\top}\hat{\beta}))}{\sum_{t=1}^n \widehat{\operatorname{Var}}(Y_t)}$.

Figure~\ref{fig:sim_inv} shows that the $R^2$ coefficient of the oracle invariant component $\beta^{\operatorname{inv}}$ remains positive even for values of $\gamma_{0,t}$ in the test data that lie outside of the observed support in the historical data; for increasing $n$ the same holds for the estimated $\hat{\beta}^{\operatorname{inv}}$.
Using $\hat{\beta}^{\operatorname{OLS}}$ or $\hat{\beta}^{\operatorname{mm}}$ leads instead to negative explained variance in this experiment. 
\begin{figure}[t]
    \centering
    \includegraphics[width=\textwidth]{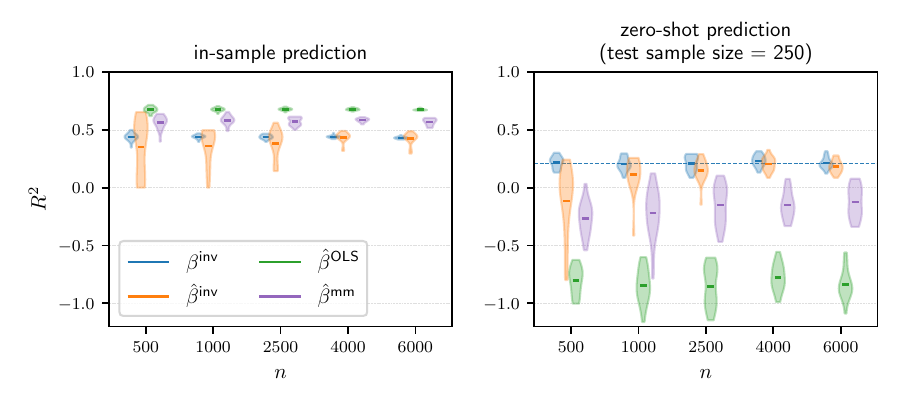}
    \caption{Normalized explained variance ($R^2$) by $\hat{\beta}^{\operatorname{inv}}$ and comparison with $\beta^{\operatorname{inv}}$, $\hat{\beta}^{\operatorname{mm}}$ and $\hat{\beta}^{\operatorname{OLS}}$: training (historical data, left) and zero-shot generalization (test data, right), for different sizes $n$ of the historical data (see Section~\ref{sec:invariant_exp}). The dashed line indicates the population value of the (normalized) explained variance by $\beta^{\operatorname{inv}}$.}
    \label{fig:sim_inv}
\end{figure}

The main limitation of the \isd{} method lies in the estimation of the invariant and residual subspaces. As outlined in Section~\ref{sec:estimation}, this process consists of two main steps, approximate joint block diagonalization and selection of the invariant blocks, both of which are in practice sensitive to noise. In both steps, we implement our estimator to be as conservative as possible, that is, such that it does not on average overestimate the number of common diagonal blocks or the dimension of $\mathcal{S}^{\operatorname{inv}}$, to avoid including part of the residual subspace into $\hat{\mathcal{S}}^{\operatorname{inv}}$. This behavior is however hard to avoid if the size of the historical dataset is not sufficiently large, therefore requiring large values of $n$ for the \isd{} framework to work effectively, see Figures~\ref{fig:mse_inv} and \ref{fig:sim_inv}.

\subsection{Time adaptation}
\label{sec:simulations_time_adaptation}

In the same setting, we now fix the size of the historical dataset to $n=6000$, which we use to estimate $\hat{\beta}^{\operatorname{inv}}$, and consider a test dataset in which the time-varying coefficients (before the transformation using $U$) undergo two shifts and take values $-0.5$ and $-2$ on two consecutive time windows, each containing $1000$ observations.
We assume that the test data are observed sequentially, and take as adaptation data a rolling window of length $m$ contained in the test data and shifting by one time point at the time. We use these sequential adaptation datasets to estimate the residual parameter $\hat{\delta}^{\operatorname{res}}_t$ and the OLS solution $\hat{\gamma}^{\operatorname{OLS}}_t$. We then compute the squared prediction error of $\hat{\gamma}^{\operatorname{\isd}}_t=\hat{\beta}^{\operatorname{inv}}+\hat{\delta}^{\operatorname{res}}_t$ and $\hat{\gamma}^{\operatorname{OLS}}_t$ on the next data point $X_{t+1}$, i.e., $(X_{t+1}^{\top}(\gamma_{0,t+1}-\hat{\gamma}_t))^2$ and approximate the corresponding MSPE using a Monte-Carlo approximation with $1000$ draws from $X_{t+1}$ (these correspond to the $1000$ sequential observations in each of the two windows of the test data).
We repeat the simulation $20$ times for different sizes of the adaptation window, 
 $m\in\{1.5p, 2p, 5p, 10p\}$, and plot 
the obtained MSPE values for $\hat{\gamma}^{\operatorname{OLS}}_t$ against $\hat{\gamma}^{\operatorname{\isd}}_t$. The result is shown in Figure~\ref{fig:sim_time}, and empirically supports Theorem~\ref{thm:empirical_explained_variance}, in particular that the difference in the MSPEs of the OLS and \isd{} estimators is proportional to the ratio $\frac{\operatorname{dim}(\mathcal{S}^{\operatorname{inv}})}{m}$ and is therefore larger for small values of $m$ and shrinks for increasing size of the adaptation data (additional details on this simulation are provided in Section~\ref{sec:simulations_time_adaptation_appendix} in the Appendix).
\begin{figure}[!bp]
    \centering
    \includegraphics[width=\textwidth]{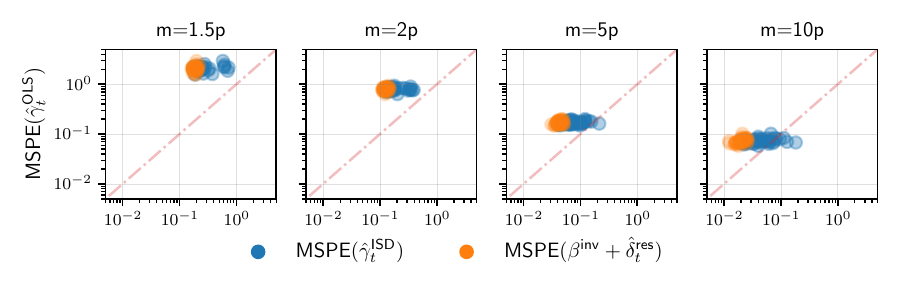}
    \caption{MSPE comparison: $\hat{\gamma}^{\operatorname{\isd}}_t$ (blue dots) vs.\ $\hat{\gamma}^{\operatorname{OLS}}_t$ for $p=10$ and various adaptation window lengths $m$ (see Section~\ref{sec:simulations_time_adaptation}). The \isd{} estimator achieves lower MSPE than the OLS for smaller sizes $m$ of the adaptation window, while the two estimators become comparable for increasing $m$. The orange dots show the MSPE of $\beta^{\operatorname{inv}}\text{(oracle)} + \hat{\delta}^{\operatorname{res}}_t$ vs.\ $\hat{\gamma}^{\operatorname{OLS}}_t$: if the subspace decomposition is known, then the \isd{} always achieves lower MSPE than the OLS.
    }
    \label{fig:sim_time}
\end{figure}
The \isd{} framework is particularly helpful in scenarios in which the size of the available adaptation window is small (two first plots from the left in Figure~\ref{fig:sim_time}). 
Indeed, from Theorem~\ref{thm:empirical_explained_variance} it also follows that the larger the dimension of the invariant subspace the greater the advantage in using the \isd{} framework for prediction rather than a naive OLS approach.
A further benefit of the \isd{} estimator is that it allows us to estimate $\hat{\delta}^{\operatorname{res}}_t$ for small lengths $m$ of the adaptation window where $\operatorname{dim}(\mathcal{S}^{\operatorname{res}})<m<p$ and OLS is not feasible.

We run a similar experiment to show (Figure~\ref{fig:sim_time_regret}) the average cumulative explained variance on the adaptation data over $20$ runs, both by estimators computed only on the historical data and estimators that use the adaptation data. For visualization purposes, we now consider the time-varying coefficients (before the transformation using $U$) equal to $(0.5-t \sin^2 (it/n+i))/n$, where $i\in \{2, ..., 8\}$ is the coefficient index, in the historical data, and constantly equal to $-0.3$, $-0.65$ and $-1$ in three consecutive time windows of size $150$ on the time points after the historical data. We estimate $\hat{\delta}^{\operatorname{res}}_t$ and $\hat{\gamma}^{\operatorname{OLS}}_t$ on a rolling adaptation window of size $m=3p$. The plot in Figure~\ref{fig:sim_time_regret} shows that, on average, the \isd{} framework, by exploiting invariance properties in the observed data, allows us to accurately explain the variance of the response by using small windows for time adaptation, significantly improving on the OLS solution in the same time windows.
\begin{figure}[!t]
    \centering
    \includegraphics[width=.8\textwidth]{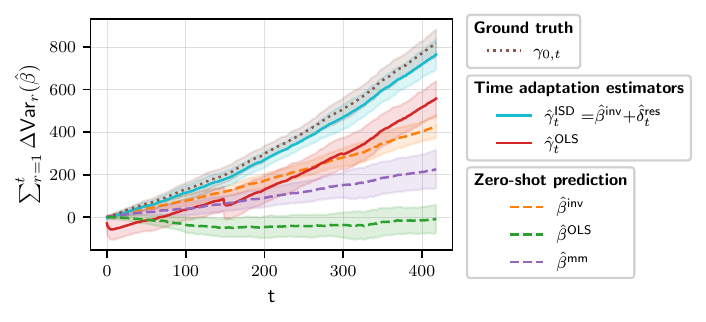}
    \caption{Average cumulative explained variance on the adaptation data and one standard deviation intervals (over $20$ runs) by the true time-varying parameter $\gamma_{0,t}$ and various estimators (see Section~\ref{sec:simulations_time_adaptation}). When few data points are available for adaptation at time $t$, the explained variance of the \isd{} estimator is significantly higher than that of the rolling window OLS, and improves on invariance-based estimators such as the invariant component $\hat{\beta}^{\operatorname{inv}}$ or the maximin $\hat{\beta}^{\operatorname{mm}}$. Due to a shift of $\gamma_{0,t}$ in the residual subspace, the OLS computed on historical data can perform worse than the zero function. 
    }
    \label{fig:sim_time_regret}
\end{figure}

\subsection{Real data example}
\label{sec:light_tunnel}
We now present a toy example that applies \isd{} to real data. The data used for this experiment are collected using a controlled physical system developed by~\citet{gamella2025causal}. The system, shown in Figure~\ref{fig:light_tunnel}, consists of a light tunnel with a light source $X_{\operatorname{RGB}}$ 
whose emitted light
passes through two polarizers with relative angle $\theta$ between them and is captured by a sensor placed at the end of the tunnel, behind the polarizers. At the end of the tunnel there are two additional LED light sources, $X_{L_{31}}$ and $X_{L_{32}}$, whose emitted light is unaffected by the polarizers. The sensor $Y := \tilde{I}_3$ measures the overall infrared light at the end of the tunnel, which is affected by the intensity of the RGB source and by the two LEDs. As described by~\citet{gamella2025causal}, the effect of $X_{\operatorname{RGB}}$ on $Y$ is linear and depends on $\theta$, more precisely, it holds that $Y\propto \operatorname{cos}^2(\theta)X_{\operatorname{RGB}}$
The dependence of $Y$ on the two LEDs is instead independent of $\theta$. 
We take $Y$ as our response, consider the covariates vector $[X_{\operatorname{RGB}}, X_{L_{31}}, X_{L_{32}}]^{\top}\in\mathbb{R}^3$
and assume that the angle $\theta$ is unknown. 
Since we control the three light sources independently, and we expect the dependence of the infrared light $Y$ on the two LEDs to remain the same across time, we hope to detect a nontrivial invariant subspace related to the two LED covariates.
\begin{figure}[t]
    \centering
    \includegraphics[width=.8\linewidth]{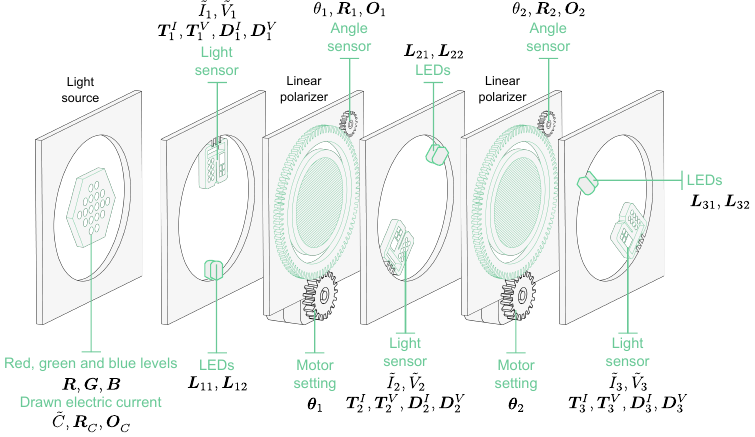}
    \caption{Illustration of the light tunnel, see Section~\ref{sec:light_tunnel}. The figure is taken from \citet{gamella2025causal} (published under a Creative Commons Attribution 4.0 International License (\url{https://creativecommons.org/licenses/by/4.0/})). The variables of interest in our experiment are the $RGB$ values of the light source, the LEDs intensities $L_{31}$ and $L_{32}$, and the measurement of the light sensor at the end of the tunnel $\tilde{I}_3$. %\jonas{[why is this not tilde?]}.}
    }
    \label{fig:light_tunnel}
\end{figure}

The 
available dataset contains $8000$ observations, collected under changing values of the angle $\theta$. The historical dataset contains the first $7000$ observations, and the test dataset the remaining $1000$. Figure~\ref{fig:real_data} shows the dependence of the response on the three covariates, as well as the values of the response and of $\operatorname{cos}^2(\theta)$ through time. 
\begin{figure}[t]
    \centering
    \includegraphics[width=\linewidth]{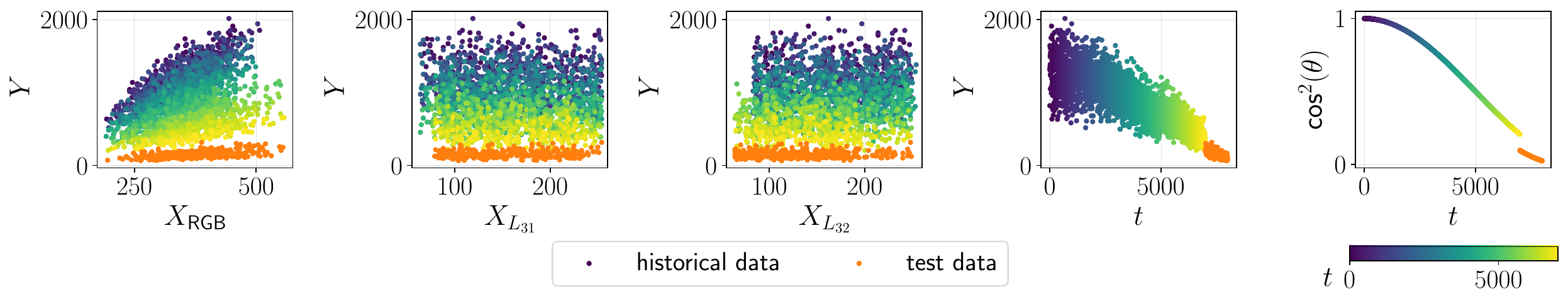}
    \includegraphics[width=\linewidth]{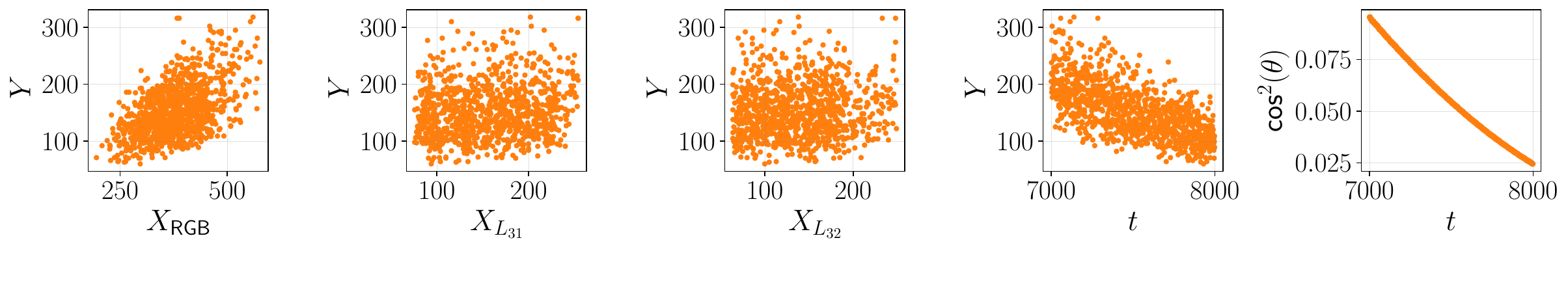}
    \caption{
    Dataset of the experiment discussed in Section~\ref{sec:light_tunnel}.
    The top four figures from the left show the dependence of the infrared measurement $Y$ on the covariates and on time $t$ (encoded by the colormap) for historical data.
    The angle of the polarizers changes over time (top right) and thus has an influence on the linear relationship between $Y$ and $X_{\operatorname{RGB}}$ (top left).
    For this experiment we assume that the angle $\theta$ is unknown.
    The second row shows the same quantities during test time, where the polarization angle is much smaller (bottom right). 
    The dependence between $Y$ and the two LEDs is small but significant (testing a reduced linear model without either $L_{31}$ or $L_{32}$ against the full model results in $p$-values smaller than $10^{-4}$, both for historical and test data).
    }
    \label{fig:real_data}
\end{figure}

We apply \isd{} on the historical data to find an invariant component $\hat{\beta}^{\operatorname{inv}}$. For comparison, we also compute the maximin $\hat{\beta}^{\operatorname{mm}}$\citep{buhlmann2015magging} and the OLS solution $\hat{\beta}^{\operatorname{OLS}}$ on historical data. We then use these estimated parameters for zero-shot prediction on test data. Table~\ref{tab:zero_shot} shows the $R^2$ coefficient, defined as in Section~\ref{sec:invariant_exp} as the fraction of explained variance $R^2 = \frac{\sum_{t=1}^n (\widehat{\operatorname{Var}}(Y_t)-\widehat{\operatorname{Var}}(Y_t-X_t^{\top}\hat{\beta}))}{\sum_{t=1}^n \widehat{\operatorname{Var}}(Y_t)}$, on historical and test data.
\begin{table}[t]
    \centering
    \begin{tabular}{c|c|c}
        $R^2$ & historical data & test data (zero-shot prediction) \\
        \hline
        $\hat{\beta}^{\operatorname{inv}}$   & 0.019 & 0.062  \\
        $\hat{\beta}^{\operatorname{OLS}}$  & 0.558 &  -10.703 \\
        $\hat{\beta}^{\operatorname{mm}}$ & 0.477 & -2.194
    \end{tabular}
    \caption{Normalized explained variance ($R^2$) by $\hat{\beta}^{\operatorname{inv}}$ and comparison with $\hat{\beta}^{\operatorname{OLS}}$ and $\hat{\beta}^{\operatorname{mm}}$: training (historical data) and zero-shot generalization (test data); see Section~\ref{sec:light_tunnel}.}
    \label{tab:zero_shot}
\end{table}
The invariant component $\hat{\beta}^{\operatorname{inv}}$ is the only estimator that achieves positive explained variance on test data. This is because $\hat{\beta}^{\operatorname{inv}}$ only captures the parts of the variance that can be transferred to the test data, as can be seen from the lower explained variance compared to $\hat{\beta}^{\operatorname{OLS}}$ on the historical data.
The estimated invariant subspace $\hat{\mathcal{S}}^{\operatorname{inv}}=\operatorname{span}\{[-0.13910168, -0.95836843, -0.24936055]^{\top}\}$ has dimension $1$ and shows in particular that most of the invariant information is encoded in the two LEDs, with highest weight given to $X_{L_{31}}$. This is expected since the LEDs intensities are not affected by the changing angle between the two polarizers. 
Also the relatively small $R^2$ is expected: since the RGB light source is stronger than the two LEDs at the end of the tunnel, most of the variance in $Y$ is explained by the non-invariant component $X_{\operatorname{RGB}}$ (see Figure~\ref{fig:real_data}).
However, it is not the case that the whole subspace spanned by the two covariates corresponding to the LEDs is invariant, as we would have assumed from the knowledge of the physical system. An explanation is that due to the data collection process there is nonzero observed correlation between $X_{\operatorname{RGB}}$ and $X_{L_{32}}$, but not between $X_{\operatorname{RGB}}$ and $X_{L_{31}}$ (on the historical data, we have $\operatorname{corr(X_{\operatorname{RGB}}, X_{L_{31}})}=0.004$ with p-value $0.723$ and $\operatorname{corr(X_{\operatorname{RGB}}, X_{L_{32}})}=-0.107$ with p-value smaller than $10^{-15}$). 

We also run the adaptation step considering as adaptation data a rolling window of size $m=8$ shifting through the test data. We show in Figure~\ref{fig:cum_xv_rd} the cumulative explained variance obtained by \isd{}, by the OLS solution computed on the same adaptation data and by the OLS solution and the invariant component computed on historical data. The plot shows that $\hat{\gamma}^{\isd{}}$ achieves the highest explained variance,  with a small improvement on the rolling window OLS $\hat{\gamma}^{OLS}$. 
Indeed, in this particular example the size of the invariant subspace is small ($\operatorname{dim}(\mathcal{S}^{\operatorname{inv}})=1$) and by Theorem~\ref{thm:empirical_explained_variance}, 
we expect only a small improvement.
\begin{figure}[t]
    \centering
    \includegraphics[width=0.65\linewidth]{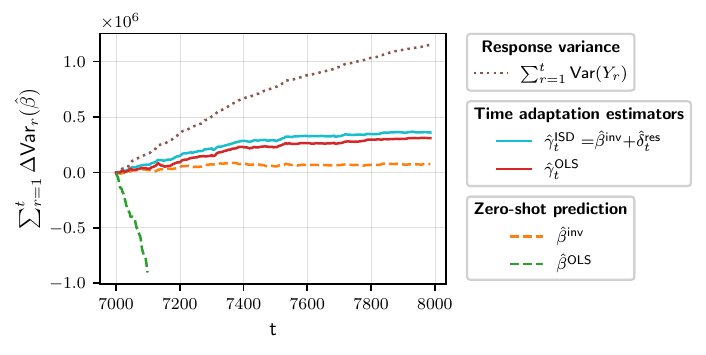}
    \caption{
    Cumulative explained variance on the adaptation data by the time adaptation estimators $\hat{\gamma}^{\operatorname{\isd}}$ and $\hat{\gamma}^{\operatorname{OLS}}$ and by the zero-shot predictors $\hat{\beta}^{\operatorname{inv}}$ and $\hat{\beta}^{\operatorname{OLS}}$; see Section~\ref{sec:light_tunnel}.
    }
    \label{fig:cum_xv_rd}
\end{figure}
% }

\section{Summary}
We propose Invariant Subspace Decomposition (\isd),
a framework for invariance-based time adaptation. Our method relies on the orthogonal decomposition of the parameter space into an invariant subspace $\mathcal{S}^{\operatorname{inv}}$ and a residual subspace $\mathcal{S}^{\operatorname{res}}$, such that the maximizer of the explained variance over $\mathcal{S}^{\operatorname{inv}}$ is time-invariant. The estimation of the invariant component $\beta^{\operatorname{inv}}$ on a large historical dataset and the reduced dimensionality of $\mathcal{S}^{\operatorname{res}}$ with respect to the original parameter space $\mathbb{R}^p$ allow the \isd{} estimator to improve on the prediction accuracy of existing estimation techniques. We provide finite sample guarantees for the proposed estimation method and additionally support the validity of our theoretical results through simulated experiments and one real world data experiment.

Future developments of this work may investigate the presented problem in the case of nonlinear time-varying models, and study how to incorporate the \isd{} framework in specific applied settings such as contextual bandits.

\acks{We thank Juan Gamella for generating the data for the presented real world example, and the three anonymous reviewers for the valuable comments.
NP and ML are supported by a research grant (0069071) from Novo Nordisk Fonden.}

\appendix

\bibliography{refs}

\section{Supporting examples and remarks}
\subsection{Example of non-uniqueness of an irreducible orthogonal partition}
    % \begin{example}
    \label{ex:non_unique_orthogonal_partition}
    Assume that for all $t\in[n]$ the covariance matrix $\Sigma_t$ of $X_t$ takes one of the following two values (and each value is taken at least once in $[n]$)
        \begin{align*}
            \Sigma^1 \coloneqq \begin{bmatrix}
                2 & 0 & 0\\
                0 & 2 & 0 \\
                0 & 0 & 1
            \end{bmatrix},
            \Sigma^2 \coloneqq \begin{bmatrix}
                3 & 0 & 0\\
                0 & 3 & 0 \\
                0 & 0 & 2
            \end{bmatrix}.
        \end{align*}
    Define for all $j\in\{1,2,3\}$ the linear spaces $\mathcal{S}_j=\langle e_j\rangle$, where $e_j$ is the $j$-th vector of the canonical basis. Then since $\Sigma^1$ and $\Sigma^2$ are (block) diagonal the partition $\mathcal{S}_1, \mathcal{S}_2, \mathcal{S}_3$ is an irreducible orthogonal and $(X_t)_{t\in[n]}$-decorrelating partition.
    Consider now the orthonormal matrix
    \begin{align*}
        U \coloneqq \begin{bmatrix}
                1/\sqrt{2} & 1/\sqrt{2} & 0\\
                -1/\sqrt{2} & 1/\sqrt{2} & 0 \\
                0 & 0 & 1
            \end{bmatrix}.
    \end{align*}
    It holds that 
    \begin{align*}
        U^{\top}\Sigma^1 U & = \Sigma^1 \\
         U^{\top}\Sigma^2 U & = \Sigma^2.
    \end{align*}
    Therefore, the spaces 
    \begin{align*}
        \tilde{\mathcal{S}}_1 =\langle \begin{bmatrix}
            1/\sqrt{2} \\ -1/\sqrt{2} \\0
        \end{bmatrix}\rangle, \quad\tilde{\mathcal{S}}_2 =\langle \begin{bmatrix}
            1/\sqrt{2} \\ 1/\sqrt{2} \\0
        \end{bmatrix}\rangle,\quad\tilde{\mathcal{S}}_3 =\langle \begin{bmatrix}
            0 \\ 0 \\ 1
        \end{bmatrix}\rangle
    \end{align*}
    also form an irreducible orthogonal and $(X_t)_{t\in[n]}$-decorrelating partition 
   (this follows, for example, from Proposition~\ref{prop:irreducible_orthogonal_partition})
    but $\{\mathcal{S}_1, \mathcal{S}_2, \mathcal{S}_3\} 
    \neq 
    \{\tilde{\mathcal{S}}_1, \tilde{\mathcal{S}}_2, \tilde{\mathcal{S}}_3\}$.

    Then, if we assume for example that $\gamma_{0,t}=\begin{bmatrix}
        1, & t, & 1
    \end{bmatrix}^{\top}$, given the first partition we obtain $\mathcal{S}^{\operatorname{inv}}=\mathcal{S}_1\oplus\mathcal{S}_3$, whereas given the second partition it holds that 
    $\Pi_{\tilde{\mathcal{S}}_1}\gamma_{0,t} =\begin{bmatrix}
        \frac{1-t}{2}, & \frac{t-1}{2}, & 0
    \end{bmatrix}^{\top}$ 
    and 
    $\Pi_{\tilde{\mathcal{S}}_2}\gamma_{0,t} =\begin{bmatrix}
        \frac{1+t}{2}, & \frac{1+t}{2}, & 0
    \end{bmatrix}^{\top}$ 
    and therefore $\tilde{\mathcal{S}}^{\operatorname{inv}}= \tilde{\mathcal{S}}_3$, leading to Assumption~\ref{ass:uniqueness_time_inv_sp} not being satisfied.
    % \end{example}

    \subsection{Example of quickly varying $\gamma_{0,t}$ and zero-shot generalization}
\label{sec:fast_gamma_changes}
We repeat the same simulation described in Section~\ref{sec:invariant_exp} but now consider non-smooth variations of $\gamma_{0,t}$. More specifically, the only difference from the simulation described in Section~\ref{sec:invariant_exp} is that we let the $3$ time-varying entries of $\gamma_{0,t}$ (prior to its rotation by $U$) vary quickly with $t$.
\begin{figure}[t]
    \centering
    \includegraphics[width=0.5\textwidth]{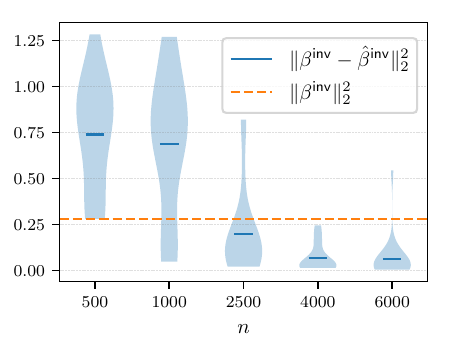}\\
    \caption{MSE of $\hat{\beta}^{\operatorname{inv}}$ for increasing size of the historical data $n$. For larger values of $n$, the estimation of the invariant subspace decomposition becomes more precise and leads to smaller errors in the estimated invariant component $\hat{\beta}^{\operatorname{inv}}$. 
    }
    \label{fig:mse_inv_randwalk}
\end{figure}
\begin{figure}[btp]
    \centering
    \includegraphics[width=\textwidth]{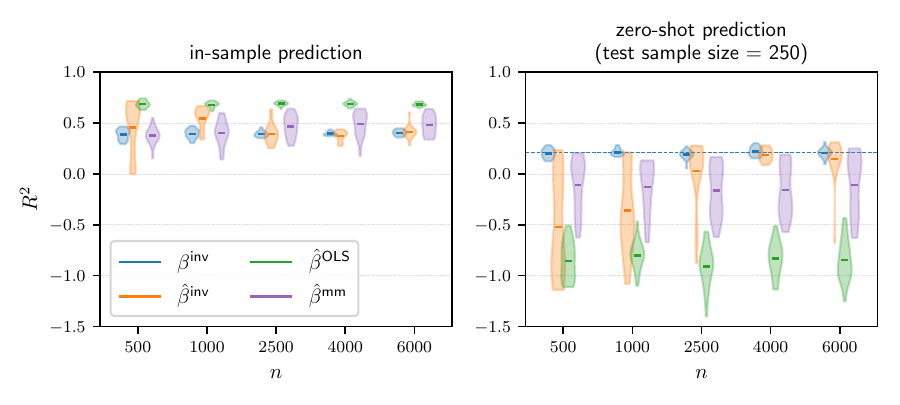}
    \caption{Normalized explained variance ($R^2$) by $\hat{\beta}^{\operatorname{inv}}$ and comparison with $\beta^{\operatorname{inv}}$, $\hat{\beta}^{\operatorname{mm}}$ and $\hat{\beta}^{\operatorname{OLS}}$: training (historical data, left) and zero-shot generalization (test data, right), for different sizes $n$ of the historical data. The dashed line indicates the population value of the (normalized) explained variance by $\beta^{\operatorname{inv}}$. In the historical data, $\gamma_{0,t}$ is quickly varying.}
    \label{fig:sim_inv_randwalk}
\end{figure}
More specifically, at each time point each one of the $3$ entries is sampled uniformly in an interval of width $1$ centered around a value which changes $20$ times in the observed time horizon $n$. These centers are randomly sampled in $[0, 1.2]$. 
These choices of the sampling intervals ensure that the experiments are comparable for different sizes $n$ of the historical data and that there is a shift in $\gamma_{0,t}$ outside of the observed support in the test data (which is generated as in Section~\ref{sec:invariant_exp}). This experiment supports Remark~\ref{rem:time_changes}, showing that we do not need assumptions on the type of changes in $\gamma_{0,t}$ in the historical data. 

In particular, Figure~\ref{fig:mse_inv_randwalk} shows that the MSE of $\hat{\beta}^{\operatorname{inv}}$ converges to zero for increasing values of $n$, that is, we are able to estimate the invariant component even when $\gamma_{0,t}$ is quickly varying in the historical data. 

Moreover, Figure~\ref{fig:sim_inv_randwalk} shows the results in terms of the $R^2$ coefficient. In particular, for $n$ large enough, the $R^2$ coefficient of the estimated invariant component $\hat{\beta}^{\operatorname{inv}}$ remains positive even for values of the test data that lie outside of the observed support.

\subsection{Methods and computational complexity of joint block diagonalization}
\label{rem:ajbd}
In the context of joint block diagonalization, we can differentiate between methods that solve the exact problem (JBD), i.e., are such that the transformed matrices have exactly zero off-block diagonal entries, and approximate methods (AJBD), which assume the presence of noise and aim to minimize the off-block diagonal entries, without necessarily setting them to zero. 

JBD is in general an easier problem, and algorithms that solve it have been shown to achieve
polynomial complexity 
\citep[see, e.g.,][]{murota2010numerical, tichavsky2012computation}. Many of these methods, e.g., the one presented by \citet{murota2010numerical}, are based on eigenvalue decompositions. Alternatively, as shown by~\citet{tichavsky2012computation}, some algorithms that solve the problem of approximate joint diagonalization (AJD)
of a set of matrices, such as \texttt{uwedge} developed by \citet{tichavsky2008fast}, can also be used for JBD. More in detail, a solution to AJD is a matrix that maximally jointly diagonalizes a set of matrices by minimizing the average value of the off-diagonal entries; if the matrices in the set cannot be exactly jointly diagonalized, the transformed matrices will have some non-zero off-diagonal elements. Adding an appropriate permutation of the columns of the joint diagonalizer allows to reorganize the non-zero off-diagonal elements into blocks, leading to jointly block diagonal matrices: in Section $4.3$ of their work \citet{gutch2012uniqueness} argue that if the matrices to be block diagonalized are symmetric, then the solution found in this way is also an optimal JBD solution. 

However, in general, methods for JBD cannot be directly applied to solve AJBD.
 Algorithms that solve AJBD directly---based on the iterative optimization of a cost function via matrix rotations---have been developed, for example, by \citet{tichavsky2012algorithms} and \citet{fevotte2007pivot}, but require the number of diagonal blocks to be known in advance. Alternatively and similarly to how JBD can be solved by AJD methods, one can also, with some slight modifications, use AJD methods to solve AJBD. More specifically, some heuristics need to be used to determine the size of the blocks: these can consist, for example, in setting a threshold for the non-zero off-block-diagonal elements in the transformed matrices. 

In the (orthogonal) settings considered here, we have found the last approach to work effectively. 
More specifically,
in Section~\ref{sec:approx_jbd_estimation}, we have denoted by $V$ the AJD solution for the set of estimated covariance matrices 
%if the matrices
$\{\hat{\Sigma}_k\}_{k=1}^K$.
Similarly to how \citet{tichavsky2012algorithms} suggest to determine a permutation of the AJD result, we proceed in the following way. To discriminate the non-zero off-block-diagonal entries in these matrices, we start by computing the following auxiliary matrix using $V$ 
 \begin{equation*}
     \Sigma \coloneqq \operatorname{max}_{k\in\{1,\dots,K\}} |V^{\top}\hat{\Sigma}_kV|,
 \end{equation*}
 where the maximum is taken element-wise. 
The matrix $\Sigma$ captures in its off-diagonal entries the residual correlation among the components identified by AJD, for all the $K$ jointly diagonalized matrices. 
For all thresholds $\tau\in\mathbb{R}$, we let $P(\tau)\in\mathbb{R}^{p\times p}$ denote one of the permutation matrices satisfying that $P(\tau)^{\top}\Sigma P(\tau)$ is a block diagonal matrix if all entries smaller than $\tau$ are considered zero. We then define the optimal threshold $\tau^*$ by
\begin{equation*}
    \tau^* \in \argmin_{\tau} \frac{1}{K}\sum_{k=1}^K \operatorname{OBD}_{\tau}(|(VP(\tau))^{\top}\hat{\Sigma}_k(VP(\tau))|) + \nu \frac{p_{bd}(\tau)}{p^2}
\end{equation*}
where $\operatorname{OBD}_{\tau}(\cdot)$ denotes the average value of off-block-diagonal entries (determined by the threshold $\tau$) of a matrix, $p_{bd}(\tau)$ is the total number of entries in the blocks induced by $\tau$ and $\nu\in\mathbb{R}$ is a regularization parameter. The penalization term $\frac{p_{bd}(\tau)}{p^2}$ is introduced to avoid always selecting a zero threshold, and the regularization parameter is set to $\nu = \frac{1}{K}\sum_{k=1}^{K}\lambda_{\min}(\hat{\Sigma}_k)$ where $\lambda_{\min}(\cdot)$ denotes the minimum eigenvalue. The optimal permutation is then $P^* \coloneqq P(\tau^*)$ and the estimated irreducible joint block diagonalizer is $\hat{U} = V P^* $.

\subsection{Threshold selection for opt-invariant subspaces}
    \label{rem:constant_threshold_selection}
    In the simulations, we select the threshold $\lambda$ in \eqref{eq:opt-invariant_test} by cross-validation. More in detail, we define the grid of possible thresholds $\lambda$ by
    \begin{equation*}
        \Lambda\coloneqq\left\{0, \tfrac{1}{K}\textstyle\sum_{k=1}^K |\hat{c}_k^1|,\ldots, \tfrac{1}{K}\textstyle\sum_{k=1}^K |\hat{c}_k^{q^{\hat{U}}_{\max}}|\right\},
    \end{equation*}
    where, for all $j\in\{1,\dots, q^{\hat{U}}_{\max}\}$ and $k\in\{1,\dots,K\}$, $\hat{c}_k^j \coloneqq \widehat{\operatorname{Corr}}(\mathbf{Y}_k - \mathbf{X}_k(\Pi_{\hat{\mathcal{S}}_j}\hat{\gamma}), \mathbf{X}_k(\Pi_{\hat{\mathcal{S}}_j}\hat{\gamma}))$.
     We then split the historical data into $L=10$ disjoint blocks of observations, and for all $j\in\{1,\dots,J\}$ denote by $\mathbf{X}_{\ell}$ and $\mathbf{Y}_{\ell}$ the observations in the $\ell$-th block and by $\mathbf{X}_{-\ell}$ and $\mathbf{Y}_{-\ell}$ the remaining historical data. 
    For all possible thresholds $\lambda\in\Lambda$ 
    we proceed in the following way. For all folds $\ell\in\{1,\dots,L\}$, we compute an estimate for the invariant component as in \eqref{eq:time-invariant_subspace_effect_estimator} using $\mathbf{X}_{-\ell}$ and $\mathbf{Y}_{-\ell}$, which we denote by $\hat{\beta}^{\operatorname{inv},-\ell}(\lambda)$. Inside the left-out $\ell$-th block of observations, we then consider a rolling window of length $d=2p$ and the observation at $t^*$ immediately following the rolling window: we compute the residual parameter $\hat{\delta}^{\operatorname{res}}_{t^*}(\lambda)$ as in \eqref{eq:residual_effect_estimator} using the $d$ observations in the rolling window, and evaluate the empirical explained variance by $\hat{\beta}^{\operatorname{inv},-\ell}(\lambda)+\hat{\delta}^{\operatorname{res}}_{t^*}(\lambda)$ on the observation at $t^*$, i.e.,
    \begin{equation*}
    \widehat{\operatorname{\Delta Var}}_{t^*}(\lambda)\coloneqq
         (Y_{t^*})^2- (Y_{t^*} - X_{t^*}^{\top}(\hat{\beta}^{\operatorname{inv},-\ell}(\lambda)+\hat{\delta}^{\operatorname{res}}_{t^*}(\lambda)))^2.
    \end{equation*}
    We repeat this computation for all possible $t^*\in\mathcal{I}_{\ell}$, where $\mathcal{I}_{\ell}$ denotes the time points in the $\ell$-th block of observations excluding the first $d$ observations, and define $\overline{\operatorname{\Delta Var}}_\ell^{\lambda}\coloneqq \frac{1}{|\mathcal{I}_{\ell}|}\sum_{t^*\in\mathcal{I}_{\ell}}\widehat{\operatorname{\Delta Var}}_{t^*}(\lambda)$. For all $\lambda\in\Lambda$, we denote the average explained variance over the $L$ folds by $\overline{\operatorname{\Delta Var}}^{\lambda}\coloneqq \frac{1}{L} \sum_{\ell=1}^L \overline{\operatorname{\Delta Var}}_{\ell}^{\lambda}$ and the standard error (across the $L$ folds) of such explained variance as $\operatorname{se}(\overline{\operatorname{\Delta Var}}^{\lambda})\coloneqq \frac{1}{L}\sqrt{\sum_{\ell=1}^L (\overline{\operatorname{\Delta Var}}_{\ell}^{\lambda}-\overline{\operatorname{\Delta Var}}^{\lambda})^2}$. Moreover, let $\lambda^{\max}\coloneqq \argmax_{\lambda\in\Lambda} \overline{\operatorname{\Delta Var}}^{\lambda}$. 
Then, we choose the optimal threshold as
    \begin{equation*}
        \lambda^* \coloneqq \min \left\{\lambda\in\Lambda  \,\Big\vert\, \overline{\operatorname{\Delta Var}}^{\lambda} > \overline{\operatorname{\Delta Var}}^{\lambda^{\max}} - t_{\operatorname{se}}\operatorname{se}(\overline{\operatorname{\Delta Var}}^{\lambda^{\max}})\right\},
    \end{equation*}
    which is the most conservative (lowest) threshold such that the corresponding explained variance is within $t_{\operatorname{se}}$ (in our simulations, we choose $t_{\operatorname{se}}=1$) standard errors (computed across the folds) of the maximal explained variance. 

\subsection{Further simulation details: MSPE comparison}
\label{sec:simulations_time_adaptation_appendix}
In Section~\ref{sec:simulations_time_adaptation} we  present a simulated experiment in which we compare the \isd{} estimator and the OLS estimator on the time adaptation task. Figure~\ref{fig:sim_time} shows that the difference in the MSPE for $\hat{\gamma}^{\operatorname{OLS}}_t$ and $\hat{\gamma}^{\operatorname{\isd}}_t$ is positive and decreases for increasing values of $m$. To further support the statement of Theorem~\ref{thm:empirical_explained_variance}, we show in Figure~\ref{fig:mspe_diff_fit} the value of such difference against $\sigma_{\operatorname{ad}}^2\frac{\operatorname{dim}(\mathcal{S}^{\operatorname{inv}})}{m}$, when computing $\hat{\gamma}^{\operatorname{\isd}}_{t}$ both with the estimated and oracle invariant component. The figure shows that the difference in the MSPEs indeed satisfies the bound stated in Theorem~\ref{thm:empirical_explained_variance}, i.e., it is always greater than $\sigma_{\operatorname{ad}}^2\frac{\operatorname{dim}(\mathcal{S}^{\operatorname{inv}})}{m}$. Moreover, it shows that for small values of $m$ the gain in using the \isd{} estimator over the OLS is even higher than what the theoretical bound suggests, indicating that it is not sharp for small $m$. 
\begin{figure}[!ht]
    \centering
    \includegraphics[width=0.46\textwidth]{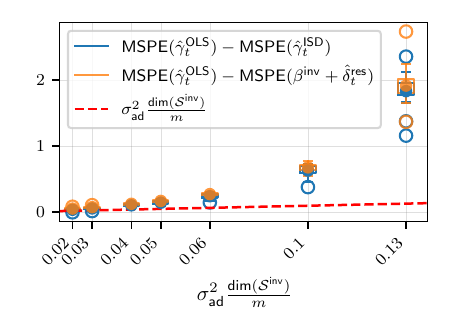}
    \caption{Difference in the MSPE of $\hat{\gamma}^{\operatorname{OLS}}_t$ and $\hat{\gamma}^{\operatorname{\isd}}_t$ (and an oracle version of $\hat{\gamma}^{\operatorname{\isd}}_t$ based on the true $\beta^{\operatorname{inv}}$) with respect to $\sigma_{\operatorname{ad}}^2\frac{\operatorname{dim}(\mathcal{S}^{\operatorname{inv}})}{m}$ for different values of $m$. The computed difference is larger than $\sigma_{\operatorname{ad}}^2\frac{\operatorname{dim}(\mathcal{S}^{\operatorname{inv}})}{m}$ for all values of $m$ (in the oracle case), empirically confirming the lower bound obtained in Theorem~\ref{thm:empirical_explained_variance}. The filled dots in the boxplots show the mean over $20$ runs (while the empty dots represent the outliers).}
    \label{fig:mspe_diff_fit}
\end{figure}
We further show in Figure~\ref{fig:mspe_isd_bound} the MSPE of $\hat{\gamma}^{\operatorname{\isd}}_t$ (again computed both with the estimated and oracle invariant component) against $\sigma_{\operatorname{ad}}^2\frac{\operatorname{dim}(\mathcal{S}^{\operatorname{res}})}{m}$, obtaining in this case an empirical confirmation of the first bound presented in Theorem~\ref{thm:empirical_explained_variance}.
\begin{figure}[t]
    \centering
    \includegraphics[width=0.46\textwidth]{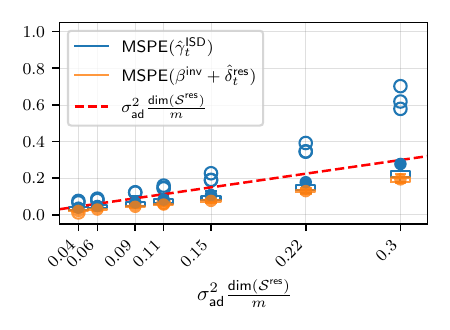}
\caption{MSPE of $\hat{\gamma}^{\operatorname{\isd}}_t$ (and an oracle version of $\hat{\gamma}^{\operatorname{\isd}}_t$ based on the true $\beta^{\operatorname{inv}}$) with respect to $\sigma_{\operatorname{ad}}^2\frac{\operatorname{dim}(\mathcal{S}^{\operatorname{res}})}{m}$ for different values of $m$. The computed MSPE is smaller than $\sigma_{\operatorname{ad}}^2\frac{\operatorname{dim}(\mathcal{S}^{\operatorname{res}})}{m}$ for all values of $m$ (in the oracle case), empirically confirming the upper bound obtained in Theorem~\ref{thm:empirical_explained_variance}. The filled dots in the boxplots show the mean over $20$ runs (while the empty dots represent the outliers).}
    \label{fig:mspe_isd_bound}
\end{figure}

\FloatBarrier

\section{ISD estimation algorithm}
\label{sec:isd_algorithm_pseudocode}
We provide here the pseudocode summarizing the ISD procedure described in Section~\ref{sec:estimation}. 
The algorithm includes the estimation of the intercept, that is, it considers for all $t\in\mathbb{N}$ the model
\begin{equation}
\label{eq:model_def_intercept}
    Y_t = \gamma_{0,t}^0 + X_t^{\top}\gamma_{0,t}+\epsilon_t
\end{equation}
satisfying the assumptions of Setting~\ref{set:definition} but
with the addition of $\gamma_{0,t}^0\in\mathbb{R}$. 
\begin{algorithm}[!ht]
    \caption{\isd{}: estimation} \label{alg:estimator}
    \begin{algorithmic}[1]
    \Require observations $(X_t, Y_t)_{t\in[n]\cup\mathcal{I}_{\operatorname{ad}}}$, $X_{t^*}$, number of windows $K$, $\lambda\in[0, 1]$
    \Ensure $\hat{\beta}^{\operatorname{inv}}$, $\hat{\delta}^{\operatorname{res}}_{t^*}, \hat{\gamma}_{t^*},\hat{\gamma}_{t^*}^0$
    \State $(\mathbf{X}_k)_{k\in[K]}\gets ([X_{\frac{(k-1)n}{K}}, \dots, X_{\frac{kn}{K}}]^{\top})_{k\in[K]}$
    \State $(\mathbf{Y}_k)_{k\in[K]}\gets ([Y_{\frac{(k-1)n}{K}}, \dots, Y_{\frac{kn}{K}}]^{\top})_{k\in[K]}$
    \State $(\hat{\Sigma}_k)_{k\in[K]} \gets \{\widehat{\operatorname{Var}}(\mathbf{X}_k)\}_{k\in[K]}$ 
    \State $\left\{\begin{bmatrix}
        \hat{\gamma}^{0}_k\\
        \hat{\gamma}_{k}
    \end{bmatrix}\right\}_{k\in[K]}\gets \left\{\operatorname{OLS}(\mathbf{Y}_k, \begin{bmatrix}
        \mathbf{1}_{\frac{n}{K}} & \mathbf{X}_k
    \end{bmatrix}
        )\right\}_{k\in[K]}$
    \State $\begin{bmatrix}
        \hat{\gamma}^{0}\\
        \hat{\gamma}
    \end{bmatrix}\gets \frac{1}{K}\sum_{k=1}^K \begin{bmatrix}
        \hat{\gamma}^{0}_k\\
        \hat{\gamma}_{k}
    \end{bmatrix}$
    \State $\hat{U},\{\hat{S}_j\}_{j=1}^q \gets \operatorname{approxIrreducibleJointBlockDiagonalizer}((\hat{\Sigma}_t)_{t\in[n]})$ \Comment{{\color{gray}see Sec.~\ref{sec:approx_jbd_estimation}}}
    \State $\rhd$ {\color{gray}Find the opt-invariant subspaces to estimate $\mathcal{S}^{\operatorname{inv}}, \mathcal{S}^{\operatorname{res}}$:}
    \State $\hat{S}^{\operatorname{inv}}, \hat{S}^{\operatorname{res}} \gets \emptyset$
    \For{$j=1,\dots,q$} 
        \State $\Pi_{\hat{\mathcal{S}}_j} \gets \hat{U}^{S_j}(\hat{U}^{S_j})^{\top}$
        \State $(\hat{c}_k^j)_{k\in[K]} \gets (\widehat{\operatorname{Corr}}(\mathbf{Y}_k - \mathbf{X}_k(\Pi_{\hat{\mathcal{S}}_j}\hat{\gamma}), \mathbf{X}_k(\Pi_{\hat{\mathcal{S}}_j}\hat{\gamma})))_{k\in[K]}$
        \If{$\frac{1}{K}\sum_{k=1}^K \left| \hat{c}_k^j \right| \le \lambda$ } 
        $\hat{S}^{\operatorname{inv}}\gets \hat{S}^{\operatorname{inv}}\cup \hat{S}_j$ \Comment{{\color{gray}see Eq.~\eqref{eq:opt-invariant_test}}}
        \Else  $\quad \hat{S}^{\operatorname{res}}\gets \hat{S}^{\operatorname{res}}\cup \hat{S}_j$
        \EndIf
    \EndFor
    \State $\rhd$ {\color{gray}Invariant component estimation:}
    \State $\mathbf{X}\gets [X_1\dots X_n]^{\top}$, $\mathbf{Y}\gets [Y_1\dots Y_n]^{\top}$
    \State $\hat{\beta}^{\operatorname{inv}} \gets \hat{U}^{\operatorname{inv}} \widehat{\operatorname{Var}}(
    \mathbf{X}\hat{U}^{\operatorname{inv}}) ^{-1} \widehat{\operatorname{Cov}}(\mathbf{X}\hat{U}^{\operatorname{inv}},\mathbf{Y})$ \Comment{{\color{gray}see Eq.~\eqref{eq:invariant_eff_estimator}}}
    \State $\rhd$ {\color{gray}Adaptation step:}
    \State $\mathbf{X}^{\operatorname{ad}}\gets [(X_t)_{t\in\mathcal{I}^{\operatorname{ad}}}]^{\top}$, $\mathbf{Y}^{\operatorname{ad}}\gets [(Y_t)_{t\in\mathcal{I}^{\operatorname{ad}}}]^{\top}$ 
    \State $ \hat{\delta}^{\operatorname{res}}_{t^*} 
      \gets \hat{U}^{\operatorname{res}}
      \widehat{\operatorname{Var}}(
      %(\hat{U}^{\operatorname{res}})^{\top}(\mathbf{X}^{\operatorname{ad}})^{\top}
      \mathbf{X}^{\operatorname{ad}}\hat{U}^{\operatorname{res}})^{-1} \widehat{\operatorname{Cov}}(\mathbf{X}^{\operatorname{ad}}\hat{U}^{\operatorname{res}},\mathbf{Y}^{\operatorname{ad}}-\mathbf{X}^{\operatorname{ad}}\hat{\beta}^{\operatorname{inv}})$  \Comment{{\color{gray}see Eq.~\eqref{eq:residual_effect_estimator}}}
      \State $\rhd$ {\color{gray}Intercept estimation:}
    \If{$\{\hat{\gamma}^0_k\}_{k\in[K]}$ approximately constant}
        $\hat{\gamma}_{t^*}^0\gets \hat{\gamma}^0$
    \Else $\quad \hat{\gamma}_{t^*}^0\gets \hat{\mathbb{E}}[\mathbf{Y}^{\operatorname{ad}}-\mathbf{X}^{\operatorname{ad}}\hat{\beta}^{\operatorname{inv}}]-\hat{\mathbb{E}}[\mathbf{X}^{\operatorname{ad}}]\hat{\delta}^{\operatorname{res}}_{t^*}$
    \EndIf
    \State $\hat{\gamma}_{t^*}\gets \hat{\beta}^{\operatorname{inv}} + \hat{\delta}^{\operatorname{res}}_{t^*} $ \Comment{{\color{gray}see Eq.~\eqref{eq:true_parameter_empirical_estimator}}}
    \end{algorithmic}
\end{algorithm}
In the estimation of the intercept, the algorithm distinguishes between two cases: (a) the intercept remains approximately constant in $[n]\cup\mathcal{I}_{\operatorname{ad}}$ and (b) the intercept changes with time but is assumed to be approximately constant in $\mathcal{I}_{\operatorname{ad}}$. 
In case (a), the computation of $\hat{\gamma}_{t^*}^0$ in line 21 of Algorithm $2$ is done by averaging the estimated values of the (approximately constant) intercept on historical data. Alternatively, in case (a) we could estimate the intercept simultaneously to the invariant component $\beta^{\operatorname{inv}}$: the whole vector (including the intercept) can be computed
by taking the OLS solution of regressing $\mathbf{Y}$ on $\begin{bmatrix}
    \mathbf{1}_n & \mathbf{X}\hat{U}^{\operatorname{inv}}
\end{bmatrix}$ and premultiplying it by $\begin{bmatrix}
    \mathbf{e}_{1} & \hat{U}^{\operatorname{inv}}
\end{bmatrix}$ (where $\mathbf{e}_1$ is the $p$-dimensional vector with the first entry equal to $1$ and the remaining equal to zero) in place of lines 16 and 21 of Algorithm~\ref{alg:estimator} (see Remark~\ref{rem:intercept}).
In case (b), the intercept can instead be estimated simultaneously to the residual component $\hat{\delta}^{\operatorname{res}}_t$: the computation  in line 22 of the algorithm is equivalent to taking the first component of the OLS solution of regressing $\mathbf{Y}^{\operatorname{ad}}-\mathbf{X}^{\operatorname{ad}}\hat{\beta}^{\operatorname{inv}}$ on $\begin{bmatrix}
        \mathbf{1}_m & \mathbf{X}^{\operatorname{ad}}\hat{U}^{\operatorname{res}}
    \end{bmatrix}$, premultiplied by $\begin{bmatrix}
    \mathbf{e}_{1} & \hat{U}^{\operatorname{res}}
\end{bmatrix}$. 
\begin{remark}
\label{rem:intercept}
    Recall that the population OLS solution for the linear model \eqref{eq:model_def_intercept} at time $t$ can be found by adding a $1$ to the vector $X_t$ and solving 
    \begin{equation*}
        \argmin_{\gamma^0, \gamma}\mathbb{E}\left[Y_t - [1\; X_t^{\top}]\begin{bmatrix}
        \gamma^0\\ \gamma
    \end{bmatrix}\right]^2.
    \end{equation*}
The problem has closed form solution 
    \begin{equation*}
        \begin{bmatrix}
            \gamma^0 \\
            \gamma
        \end{bmatrix} = \begin{bmatrix}
            1 & \mathbb{E}[X_t]^{\top} \\
            \mathbb{E}[X_t] & \mathbb{E}[X_tX_t^{\top}]
        \end{bmatrix}^{-1}\begin{bmatrix}
            \mathbb{E}[Y_t]\\
            \mathbb{E}[X_tY_t]
        \end{bmatrix} = \begin{bmatrix}
            \mathbb{E}[Y_t] -  \mathbb{E}[X_t]^{\top} \operatorname{Var}(X_t)^{-1}\operatorname{Cov}(X_t, Y_t) \\
            \operatorname{Var}(X_t)^{-1}\operatorname{Cov}(X_t, Y_t)
        \end{bmatrix}.
    \end{equation*}
\end{remark}

\section{Extension to non-orthogonal subspaces}
\label{sec:non_orthogonal_subspaces}
In Section~\ref{sec:time_inv_sub} we have defined an orthogonal and $(X_t)_{t\in[n]}$-decorrelating partition as a collection $\{\mathcal{S}_j\}_{j\in\{1,\dots,q\}}$ of pairwise orthogonal linear subspaces of $\mathbb{R}^p$ satisfying \eqref{eq:orthogonal_partition}. Finding the orthogonal subspaces that form the partition means in particular finding a rotation of the original $X$-space such that each subspace is spanned by a subset of the rotated axes, and the coordinates of the projected predictors onto one subspace are uncorrelated with the ones in the remaining subspaces. 
In this section, we show that orthogonality of the subspaces in the partition is not strictly required to obtain a separation of the true time-varying parameter of the form~\eqref{eq:separation_orth_partition1}, 
that is, more general invertible linear transformations can be considered besides rotations.
In particular, we briefly present results similar to the ones obtained throughout Section~\ref{sec:time_inv_sep} but where the subspaces in the partition are not necessarily orthogonal.
To do so, we define a collection of (not necessarily orthogonal) linear subspaces $\mathcal{S}_1,\ldots,\mathcal{S}_q\subseteq\mathbb{R}^p$ with $\bigoplus_{j=1}^q \mathcal{S}_j=\mathbb{R}^p$ and satisfying \eqref{eq:orthogonal_partition} a \emph{$(X_t)_{t\in [n]}$-decorrelating partition (of cardinality $q$)}, and further call it \emph{irreducible} if it is of maximal cardinality. 
A $(X_t)_{t\in[n]}$-decorrelating partition can still be identified through a joint block diagonalization of the covariance matrices $(\Sigma_t)_{t\in[n]}$ as described in Section~\ref{sec:orthogonal_joint_block_diagonalization} but with an adjustment. More specifically, instead of assuming that the joint diagonalizer $U$ is an orthogonal matrix, we only assume it is invertible and for all $j\in\{1,\dots,q\}$, the columns of $U^{S_j}$ are orthonormal vectors.
A version of Proposition~\ref{prop:irreducible_orthogonal_partition} in which the resulting partition is not necessarily orthogonal follows with the same proof.
Moreover, similarly to the orthogonal case, the uniqueness of an irreducible $(X_t)_{t\in[n]}$-decorrelating partition is implied by the uniqueness of an irreducible non-orthogonal joint block diagonalizer for $(\Sigma_t)_{t\in[n]}$; explicit conditions under which such uniqueness holds can be found for example in the work by \citet{nion2011tensor}. In the results presented in the remainder of this section we adopt the same notation introduced in Section~\ref{sec:orthogonal_joint_block_diagonalization}, and we additionally define the matrix $W\coloneqq U^{-\top}$. 

A $(X_t)_{t\in[n]}$-decorrelating partition of cardinality $q$ allows us to decompose the true time-varying parameter into the sum of $q$ components. 
To obtain such a decomposition via non-orthogonal subspaces, oblique projections need to be considered in place of orthogonal ones. Oblique projections are defined \citep[see, e.g.,][]{schott2016matrix} for two subspaces $\mathcal{S}_1,\mathcal{S}_2\subseteq\mathbb{R}^p$ such that $\mathcal{S}_1\bigoplus \mathcal{S}_2 = \mathbb{R}^p$ and a vector $x\in\mathbb{R}^p$ as the vectors $x_1\in\mathcal{S}_1$ and $x_2\in\mathcal{S}_2$ such that $x=x_1+x_2$: $x_1$ is called the projection of $x$ onto $\mathcal{S}_1$ along $\mathcal{S}_2$, and $x_2$ the projection of $x$ onto $\mathcal{S}_2$ along $\mathcal{S}_1$. For a $(X_t)_{t\in[n]}$-decorrelating partition $\{\mathcal{S}_j\}_{j=1}^q$, we denote by $P_{\mathcal{S}_j\mid\mathcal{S}_{-j}}$ the oblique projection matrix onto $\mathcal{S}_j$ along $\bigoplus_{i\in\{1,\dots,q\}\setminus\{j\}} \mathcal{S}_i$: this can be expressed in terms of a (non-orthogonal) joint block diagonalizer $U$ corresponding to the partition as $ P_{\mathcal{S}_j\mid\mathcal{S}_{-j}} = U^{S_j}(W^{S_j})^{\top}$. 
Orthogonal and $(X_t)_{t\in[n]}$-decorrelating partitions $\{\mathcal{S}_j\}_{j=1}^q$ are a special case of $(X_t)_{t\in[n]}$-decorrelating partitions. In particular, if the subspaces are pairwise orthogonal, it holds for all $j\in\{1,\dots,q\}$ that $P_{\mathcal{S}_j\mid\mathcal{S}^{-j}} = \Pi_{\mathcal{S}_j}$.

By definition of oblique projections, for all $t\in[n]$, we can express $\gamma_{0,t}$ as 
\begin{equation*}
    \gamma_{0,t} = \sum_{j=1}^q P_{\mathcal{S}_j\mid\mathcal{S}_{-j}}\gamma_{0,t}.
\end{equation*}
Similarly to how we define opt-invariance on $[n]$ for orthogonal subspaces in Section~\ref{sec:time_inv_sub}, we say that a subspace $\mathcal{S}_j$ in a $(X_t)_{t\in[n]}$-decorrelating partition is \emph{proj-invariant} on $[n]$ if it satisfies for all $t,s\in[n]$ that
\begin{equation*}
    P_{\mathcal{S}_j\mid\mathcal{S}_{-j}}\gamma_{0,t} = P_{\mathcal{S}_j\mid\mathcal{S}_{-j}}\gamma_{0,s}.
\end{equation*}
By Lemma~\ref{lem:time_var_parameter_separation} it follows that for orthogonal partitions proj-invariance is equivalent to opt-invariance.
For an irreducible $(X_t)_{t\in[n]}$-decorrelating partition, we now define the invariant subspace $ \mathcal{S}^{\operatorname{inv}}$ and residual subspace  $\mathcal{S}^{\operatorname{res}}$ as
\begin{equation*}
    \mathcal{S}^{\operatorname{inv}} \coloneqq\textstyle\bigoplus_{\substack{j\in\{1,\dots,q\}:\\ \mathcal{S}_j \text{ proj-invariant on }[n]}}\mathcal{S}_j
    \quad\text{and}\quad
    \mathcal{S}^{\operatorname{res}}\coloneqq \textstyle\bigoplus_{\substack{j\in\{1,\dots,q\}:\\ \mathcal{S}_j\, \text{not proj-invariant on }[n]}}\mathcal{S}_j.
\end{equation*}
\sloppy It follows directly by the definition of partitions that $\{\mathcal{S}^{\operatorname{inv}}, \mathcal{S}^{\operatorname{res}}\}$ is a $(X_t)_{t\in[n]}$-decorrelating partition. Moreover, $\mathcal{S}^{\operatorname{inv}}$ is proj-invariant on $[n]$ since $P_{\mathcal{S}^{\operatorname{inv}}\mid\mathcal{S}^{\operatorname{res}}}\gamma_{0,t} = \sum_{\substack{j\in\{1,\dots,q\}:\\\mathcal{S}_j \text{ proj-invariant on }[n]}} P_{\mathcal{S}_j\mid\mathcal{S}_{-j}}\gamma_{0,t}$.
We finally define the invariant and residual components by
\begin{equation*}
    \beta^{\operatorname{inv}} \coloneqq  P_{\mathcal{S}^{\operatorname{inv}}\mid\mathcal{S}^{\operatorname{res}}}\bar{\gamma}_{0} \quad\text{and}\quad \delta^{\operatorname{res}}_t \coloneqq P_{\mathcal{S}^{\operatorname{res}}\mid\mathcal{S}^{\operatorname{inv}}}\gamma_{0,t}.
\end{equation*}
We show in the following proposition that the expressions \eqref{eq:invariant_component_pop_estimator} and \eqref{eq:adaptation_parameter_estimator} used to construct our estimators for the invariant and residual component in the case of orthogonal partitions, remain valid in the case of non-orthogonal partitions. 

\begin{proposition}
\label{prop:non_orthogonal_subspaces}
    Let $\{\mathcal{S}_j\}_{j=1}^q$ be a $(X_t)_{t\in[n]}$-decorrelating partition and let $U$ be a joint block diagonalizer corresponding to that partition. 
    Then, the following results hold.
    \begin{enumerate}[(i)]
        \item For all $t\in[n]$ and for all $j\in\{1,\dots,q\}$ 
                \begin{align}
                     P_{\mathcal{S}_j\mid\mathcal{S}_{-j}}\gamma_{0,t}
                    & =  U^{S_j}\operatorname{Var}((U^{S_j})^{\top}X_t)^{-1} \operatorname{Cov}((U^{S_j})^{\top}X_t, Y_t). \label{eq:non_orthogonal_decomposition_opt}
                \end{align}
            \item $\beta^{\operatorname{inv}}$ is time-invariant over $[n]$ and 
                    \begin{equation*}
                       \beta^{\operatorname{inv}} = U^{\operatorname{inv}} ((U^{\operatorname{inv}})^{\top}\overline{\operatorname{Var}}(X)U^{\operatorname{inv}})^{-1}(U^{\operatorname{inv}})^{\top}\overline{\operatorname{Cov}}(X, Y).
                    \end{equation*}
            \item \begin{equation*}
                       \delta^{\operatorname{res}}_t = U^{\operatorname{res}} ((U^{\operatorname{res}})^{\top}\operatorname{Var}(X_t)U^{\operatorname{res}})^{-1}(U^{\operatorname{res}})^{\top}{\operatorname{Cov}}(X_t, Y_t-X_t^{\top}\beta^{\operatorname{inv}}).
                    \end{equation*}
    \end{enumerate}
\end{proposition}

Proposition~\ref{prop:non_orthogonal_subspaces} implies in particular that, apart from the joint block diagonalization differences, estimating $\beta^{\operatorname{inv}}$ and $\delta^{\operatorname{res}}_t$ in the non-orthogonal case can be done as described in Section~\ref{sec:estimation}. Moreover, it  holds that the parameter $\beta^{\operatorname{inv}}$ defined for non-orthogonal $(X_t)_{t\in[n]}$-decorrelating partitions is still a time-invariant parameter. Under a generalization assumption analogous to Assumption~\ref{ass:generalization}, $\beta^{\operatorname{inv}}$ has positive explained variance at all time points $t\in\mathbb{N}$ and can be used to at least partially predict $\gamma_{0,t}$ (we have not added this result explicitly).
In addition, also in the non-orthogonal case the estimation of the residual component only requires to estimate a reduced number of parameters, that is, $\operatorname{dim}(\mathcal{S}^{\operatorname{res}})$ .

In the case of non-orthogonal partitions, however, we cannot directly interpret \isd{} as separating the true time-varying parameter into two separate optimizations of the explained variance over $\mathcal{S}^{\operatorname{inv}}$ and $\mathcal{S}^{\operatorname{res}}$. 

\begin{example}[non-orthogonal irreducible partition]
\label{ex:non_orthogonal_partition}
    Let $X_t\in\mathbb{R}^3$ with covariance matrix $\Sigma_t$ that for all $t\in[n]$ takes one of the following values
    \begin{equation*}
         \begin{bmatrix}
            1 & -0.5 & 0\\
            -0.5 & 1 & 0 \\
            0 & 0 & 1
        \end{bmatrix}  \quad \text{or} \quad
         \begin{bmatrix}
            4 & -2 & 0\\
            -2 & 7/2 & 0\\
            0 & 0 & 1
        \end{bmatrix}.
    \end{equation*}
These matrices are in block diagonal form and in particular the $2\times2$ submatrices forming the first diagonal block do not commute. This implies that an irreducible (orthogonal) joint block diagonalizer for these matrices is the three-dimensional identity matrix $I_3$, and the diagonal blocks cannot be further reduced into smaller blocks using an orthogonal transformation. It also implies that an irreducible orthogonal and $(X_t)_{t\in[n]}$-decorrelating partition is given by $\{\mathcal{S}_j\}_{j=1}^2$ with $\mathcal{S}_1\coloneqq \langle e_1, e_2\rangle$ and $\mathcal{S}_2=\langle e_3 \rangle$, where $e_j$ is the $j$-th vector of the canonical basis for $\mathbb{R}^3$. 

There exists, however, a non-orthogonal joint diagonalizer for these matrices, i.e., a non-orthogonal matrix $U$ such that $\tilde{\Sigma}_t=U^{\top}\Sigma_t U$ is diagonal. It is given by 
\begin{equation*}
    U = \begin{bmatrix}
        1/\sqrt{5} & 1 & 0 \\
        2/\sqrt{5} & 0 & 0\\
        0 & 0 & 1
    \end{bmatrix}.
\end{equation*}
$U$ induces an irreducible $(X_t)_{t\in[n]}$-decorrelating partition by
\begin{equation*}
    \bar{\mathcal{S}}_1 \coloneqq \langle \begin{bmatrix}
        1/\sqrt{5}  \\
        2/\sqrt{5}  \\
        0
    \end{bmatrix}\rangle = \langle U^{S_1}\rangle, \quad  \bar{\mathcal{S}}_2 \coloneqq \langle \begin{bmatrix}
        1  \\
        0 \\
        0 \\
    \end{bmatrix}\rangle = \langle U^{S_2}\rangle, \quad
    \bar{\mathcal{S}}_3 \coloneqq \langle \begin{bmatrix}
        0  \\
        0 \\
        1 \\
    \end{bmatrix}\rangle = \langle U^{S_3}\rangle.
\end{equation*}
Let $\mathcal{S}^{\operatorname{inv}}$ and $\bar{\mathcal{S}}^{\operatorname{inv}}$ be the invariant subspaces associated with the irreducible orthogonal partition $\{\mathcal{S}_j\}_{j=1}^3$ and the irreducible partition $\{\Bar{\mathcal{S}}_j\}_{j=1}^3$, respectively. 
As any irreducible orthogonal and $(X_t)_{t\in[n]}$-decorrelating partition is also a $(X_t)_{t\in[n]}$-decorrelating partition, it in general holds that 
$$\operatorname{dim}(\mathcal{S}^{\operatorname{inv}})\leq \operatorname{dim}(\bar{\mathcal{S}}^{\operatorname{inv}}).$$
Moreover, as in the explicit example above, the inequality can be strict.
\end{example}

\section{Auxiliary results}

\begin{lemma}
    \label{lem:orthogonal_inverse}
    Let $B\in\mathbb{R}^{m\times m}$ be a symmetric invertible matrix, $U\in\mathbb{R}^{m\times m}$ an orthogonal block diagonalizer of $B$ and $S_1,\ldots,S_q\subseteq\{1,\dots,m\}$ disjoint subsets satisfying
    \begin{equation*}
        U^{\top}B U=
        \begin{bmatrix}
            (U^{S_1})^{\top}B U^{S_1} & \mathbf{0} &\cdots&\mathbf{0}\\
            \mathbf{0}& \ddots & \ddots & \vdots \\
            \vdots & \ddots & \ddots & \mathbf{0}\\
            \mathbf{0} &\cdots & \mathbf{0} & (U^{S_q})^{\top}B U^{S_q}
        \end{bmatrix}.
    \end{equation*}
    Then it holds for all $j\in\{1,\ldots,q\}$ that
    \begin{equation*}
        (\Pi_{\mathcal{S}_j} B \Pi_{\mathcal{S}_j})^{\dagger}=\Pi_{\mathcal{S}_j} B^{-1} \Pi_{\mathcal{S}_j},
    \end{equation*}
    where $\Pi_{\mathcal{S}_j}\coloneqq U^{S_j}(U^{S_j})^{\top}$.
\end{lemma}

\begin{proof}
The pseudo-inverse $A^{\dagger}$ of a matrix $A$ is defined as the unique matrix satisfying: (i) $AA^{\dagger}A=A$, (ii) $A^{\dagger}AA^{\dagger}=A^{\dagger}$, (iii) $(AA^{\dagger})^{\top} = AA^{\dagger}$ and  (iv) $(A^{\dagger}A)^{\top} = A^{\dagger}A$.

Fix $j\in\{1\ldots,q\}$ and define $A\coloneqq\Pi_{\mathcal{S}_j}B\Pi_{\mathcal{S}_j}$ and $A^{\dagger}\coloneqq\Pi_{\mathcal{S}_j}B^{-1}\Pi_{\mathcal{S}_j}$. Moreover, define $\widetilde{B}\coloneqq U^{\top}B U$ and for all $k\in\{1,\ldots,q\}$, $\widetilde{B}_k \coloneqq (U^{S_k})^{\top}B U^{S_k}$.
We now verify that conditions (i)-(iv) hold for $A$ and $A^{\dagger}$ and hence $A^{\dagger}$ is indeed the pseudo-inverse. Conditions (ii) and (iv) hold by symmetry of $B$ and $\Pi_S$. For (i) and (iii), first observe that by the properties of orthogonal matrices it holds that $\widetilde{B}^{-1}=U^{\top}B^{-1}U$, and, due to the block diagonal structure of $\widetilde{B}$, 
\begin{equation}
    \label{eq:block_diagonal_inverse_first}
    \widetilde{B}^{-1} = \begin{bmatrix}
    \widetilde{B}_{1}^{-1} & & \\
    & \ddots & \\
    & & \widetilde{B}_{q}^{-1} 
    \end{bmatrix} = \begin{bmatrix}
        (U^{S_1})^{\top} \\ \vdots \\
        (U^{S_q})^{\top}
    \end{bmatrix} B^{-1}
    \begin{bmatrix}
        U^{S_1} & \dots &
        U^{S_q}
    \end{bmatrix}.
\end{equation}
Hence we get that $\widetilde{B}_j^{-1}=(U^{S_j})^{\top}B^{-1}U^{S_j}$. For (i), we now get
\begin{align*}
    \Pi_{\mathcal{S}_j}B\Pi_{\mathcal{S}_j}\Pi_{\mathcal{S}_j}B^{-1}\Pi_{\mathcal{S}_j} \Pi_{\mathcal{S}_j}B\Pi_{\mathcal{S}_j}
    &=  U^{S_j}(U^{S_j})^{\top}B U^{S_j}(U^{S_j})^{\top}B^{-1}U^{S_j}(U^{S_j})^{\top}B U^{S_j}(U^{S_j})^{\top} \\
    &=  U^{S_j} \widetilde{B}_{j} \widetilde{B}_{j}^{-1} \widetilde{B}_{j}(U^{S_j})^{\top} \\
    &=  \Pi_{\mathcal{S}_j}B\Pi_{\mathcal{S}_j}.
    \end{align*}
    Similarly, for (iii) we get
    \begin{align*}
          (\Pi_{\mathcal{S}_j}B\Pi_{\mathcal{S}_j}\Pi_{\mathcal{S}_j}B^{-1}\Pi_{\mathcal{S}_j} )^{\top}
        &=  (U^{S_j}(U^{S_j})^{\top}B U^{S_j}(U^{S_j})^{\top}B^{-1}U^{S_j}(U^{S_j})^{\top})^{\top} \\
        &=  ( U^{S_j} 
        \widetilde{B}_{j}
        \widetilde{B}_{j}^{-1} (U^{S_j})^{\top})^{\top} \\
        &=  (\Pi_{\mathcal{S}_j})^{\top} \\
        &=  \Pi_{\mathcal{S}_j}.
    \end{align*}
This completes the proof of Lemma~\ref{lem:orthogonal_inverse}.

\end{proof}

\begin{lemma}
\label{lem:orthogonal_partition_diagonalizer}
    Let 
    $\mathcal{N}\subseteq \mathbb{N}$ and let
    $\{\mathcal{S}_j\}_{j=1}^q$ be an orthogonal and $(X_t)_{t\in \mathcal{N}}$-decorrelating partition.
    Then, there exists a joint block diagonalizer of $(\Sigma_t)_{t\in \mathcal{N}}$. More precisely, there exists an orthonormal matrix $U\in\mathbb{R}^{p\times p}$ such that for all $t\in\mathcal{N}$ the matrix 
    $\tilde{\Sigma}_t\coloneqq U^{\top}\Sigma_t U$ is block diagonal with $q$ diagonal blocks $\widetilde{\Sigma}_{t,j}=(U^{S_j})^{\top}\Sigma_t U^{S_j}$, $j\in\{1,\dots,q\}$ and of dimension $|S_j|=\operatorname{dim}(\mathcal{S}_j)$, where $S_j\subseteq \{1,\dots,p\}$ indexes a subset of the columns of $U$. Moreover, $\Pi_{\mathcal{S}_j}=U^{S_j}(U^{S_j})^{\top}$.
\end{lemma}

\begin{proof}
    Let $U=(u_1,\ldots,u_p)\in\mathbb{R}^{p\times p}$ be an orthonormal matrix with columns $u_1,\ldots,u_p$ such that for all $j\in\{1,\ldots,q\}$ there exists $S_j\subseteq\{1,\ldots,p\}$ such that $\mathcal{S}_j=\operatorname{span}(\{u_k\mid k\in S_j\})$. Such a matrix $U$ can be constructed by selecting an orthonormal basis for each of the disjoint subspaces $\{\mathcal{S}_j\}_{j=1}^q$. Furthermore, assume that the columns of $U$ are ordered in such a way that for all $i,j\in\{1,\ldots,q\}$ with $i<j$ it holds for all $k\in S_i$ and $\ell\in S_j$ that $k<\ell$.
    As the matrix $U$ is orthogonal, it holds for all $j\in\{1,\ldots,q\}$ that the projection matrix $\Pi_{\mathcal{S}_j}$ can be expressed as $\Pi_{\mathcal{S}_j} = U^{S_j}(U^{S_j})^{\top}$ and hence using the definition of orthogonal partition (see \eqref{eq:orthogonal_partition}) it holds that, for all $t \in \mathcal{N}$,
    \begin{equation*}
        \widetilde{\Sigma}_t\coloneqq U^{\top}\Sigma_t U =
        \begin{bmatrix}
            \widetilde{\Sigma}_{t,1} & \mathbf{0} &\cdots&\mathbf{0}\\
            \mathbf{0}& \ddots & \ddots & \vdots \\
            \vdots & \ddots & \ddots & \mathbf{0}\\
            \mathbf{0} &\cdots & \mathbf{0} & \widetilde{\Sigma}_{t,q}
        \end{bmatrix},
    \end{equation*}
    where for all $j\in\{1,\ldots,q\}$ we defined $\widetilde{\Sigma}_{t,j}\coloneqq (U^{S_j})^{\top}\Sigma_t U^{S_j}$. 
\end{proof}

\begin{lemma}
    \label{lem:pseudo_inv_projection}
    Let $\{\mathcal{S}_j\}_{j=1}^q$ be an orthogonal and $(X_t)_{t\in[n]}$-decorrelating partition. Then it holds for all $t\in[n]$ and, for all $j \in \{1, \ldots, q\}$, that
\begin{equation}
\label{eq:pseudo_inverse_proof}
\operatorname{Var}\left(\Pi_{\mathcal{S}_j}X_t\right)^{\dagger} = \Pi_{\mathcal{S}_j}\operatorname{Var}\left(X_t\right)^{-1}\Pi_{\mathcal{S}_j}.
\end{equation}
\end{lemma}
\begin{proof}
Let $U\in\mathbb{R}^{p\times p}$ be an orthonormal matrix such that, for all $t\in[n]$, $\tilde{\Sigma}_t \coloneqq U^{\top}\Sigma_t U$ is block diagonal with $q$ diagonal blocks $\widetilde{\Sigma}_{t,1}, \dots, \widetilde{\Sigma}_{t,q}$ of dimensions $\operatorname{dim}(\mathcal{S}_1), \dots, \operatorname{dim}(\mathcal{S}_q)$. Such a matrix exists by Lemma~\ref{lem:orthogonal_partition_diagonalizer} and each diagonal block is given by $\widetilde{\Sigma}_{t,j}\coloneqq (U^{S_j})^{\top}\Sigma_t U^{S_j}$ and the projection matrix $\Pi_{\mathcal{S}_j}$ can be expressed as $\Pi_{\mathcal{S}_j} = U^{S_j}(U^{S_j})^{\top}$. The statement then follows from Lemma~\ref{lem:orthogonal_inverse}.
\end{proof}

\paragraph{Proof of Lemma~\ref{lem:time_var_parameter_separation}}
\begin{proof}
    Let $U\in\mathbb{R}^{p\times p}$ be an orthonormal matrix such that, for all $t\in[n]$, $\tilde{\Sigma}_t \coloneqq U^{\top}\Sigma_t U$ is block diagonal with $q$ diagonal blocks $\widetilde{\Sigma}_{t,1}, \dots, \widetilde{\Sigma}_{t,q}$ of dimensions $\operatorname{dim}(\mathcal{S}_1), \dots, \operatorname{dim}(\mathcal{S}_q)$. Such a matrix exists by Lemma~\ref{lem:orthogonal_partition_diagonalizer} and each diagonal block is given by $\widetilde{\Sigma}_{t,j}\coloneqq (U^{S_j})^{\top}\Sigma_t U^{S_j}$ and the projection matrix $\Pi_{\mathcal{S}_j}$ can be expressed as $\Pi_{\mathcal{S}_j} = U^{S_j}(U^{S_j})^{\top}$.
    
    Now, for an arbitrary $t\in[n]$ it holds using the linear model \eqref{eq:model_def} that $\gamma_{0,t}=\Sigma_t^{-1}\operatorname{Cov}(X_t,Y_t)$ and hence we use $U$ to get the following expansion
    \begin{align*}
        \gamma_{0,t} & = \Sigma_t^{-1}\operatorname{Cov}(X_t,Y_t) \\
        & = U \tilde{\Sigma}_{t}^{-1} U^{\top} \operatorname{Cov}(X_t,Y_t) \\
        & = U \begin{bmatrix}
            \tilde{\Sigma}_{t,1}^{-1} & & \\
            & \ddots & \\
            & & \tilde{\Sigma}_{t,q}^{-1} 
        \end{bmatrix} U^{\top} \operatorname{Cov}(X_t,Y_t) \\
        & =U\begin{bmatrix}
             \tilde{\Sigma}_{t,1}^{-1}(U^{S_1})^{\top}\operatorname{Cov}(X_t,Y_t)  \\
            \vdots \\
             \tilde{\Sigma}_{t,q}^{-1}(U^{S_q})^{\top}\operatorname{Cov}(X_t,Y_t) 
        \end{bmatrix}.
        \end{align*}
        By the properties of orthogonal matrices it holds that $\tilde{\Sigma}_{t}^{-1}=U^{\top}\Sigma_t^{-1}U$, and, due to the block diagonal structure of $\tilde{\Sigma}_t$, 
        \begin{equation}
        \label{eq:block_diagonal_inverse}
            \tilde{\Sigma}_{t}^{-1} = \begin{bmatrix}
            \tilde{\Sigma}_{t,1}^{-1} & & \\
            & \ddots & \\
            & & \tilde{\Sigma}_{t,q}^{-1} 
            \end{bmatrix} = \begin{bmatrix}
                (U^{S_1})^{\top} \\ \vdots \\
                (U^{S_q})^{\top}
            \end{bmatrix}\Sigma_t^{-1} \begin{bmatrix}
                U^{S_1} & \dots &
                U^{S_q}
            \end{bmatrix}.
        \end{equation}
        This implies that
        $\tilde{\Sigma}_{t,j}^{-1}=(U^{S_j})^{\top}\Sigma_t^{-1}U^{S_j}$ and therefore
        \begin{align*}
        \gamma_{0,t} & = U  \begin{bmatrix}
            (U^{S_1})^{\top}\Sigma_t^{-1}U^{S_1}(U^{S_1})^{\top}\operatorname{Cov}(X_t,Y_t)  \\
            \vdots \\
            (U^{S_q})^{\top}\Sigma_t^{-1}U^{S_q}(U^{S_q})^{\top}\operatorname{Cov}(X_t,Y_t) 
        \end{bmatrix} \\
         & =  \begin{bmatrix}
            U^{S_1} & \cdots & U^{S_q}
        \end{bmatrix}  \begin{bmatrix}
            (U^{S_1})^{\top}\Sigma_t^{-1}U^{S_1}(U^{S_1})^{\top}\operatorname{Cov}(X_t,Y_t)  \\
            \vdots \\
            (U^{S_q})^{\top}\Sigma_t^{-1}U^{S_q}(U^{S_q})^{\top}\operatorname{Cov}(X_t,Y_t) 
        \end{bmatrix} \\
         & = \sum_{j=1}^q  U^{S_j}(U^{S_j})^{\top}\Sigma_t^{-1}U^{S_j}(U^{S_j})^{\top}\operatorname{Cov}(X_t,Y_t) \\
        & =  \sum_{j=1}^q  \Pi_{\mathcal{S}_j}\Sigma_t^{-1}\Pi_{\mathcal{S}_j}\operatorname{Cov}(X_t,Y_t) \\
        & = \sum_{j=1}^q  \Pi_{\mathcal{S}_j}\Sigma_t^{-1}\Pi_{\mathcal{S}_j}\Pi_{\mathcal{S}_j}\operatorname{Cov}(X_t,Y_t) \\
         & = \sum_{j=1}^q \operatorname{Var}(\Pi_{\mathcal{S}_j}X_t)^{\dagger} \operatorname{Cov}(\Pi_{\mathcal{S}_j}X_t,Y_t).
    \end{align*}
    In the second to last equality, we have used that $\Pi_{\mathcal{S}_j}$ is a projection matrix and thus idempotent. In the last equality we have used that $\operatorname{Var}(\Pi_{\mathcal{S}_j}X_t)^{\dagger}=\Pi_{\mathcal{S}_j}\Sigma_t^{-1}\Pi_{\mathcal{S}_j}$, by Lemma~\ref{lem:pseudo_inv_projection}. 
    It now suffices to show that $\Pi_{\mathcal{S}_j}\gamma_{0,t}= \operatorname{Var}(\Pi_{\mathcal{S}_j}X_t)^{\dagger}\operatorname{Cov}(\Pi_{\mathcal{S}_j}X_t, Y_t )$, which follows from the following computation
    \begin{align*}
    \Pi_{\mathcal{S}_j}\gamma_{0,t} 
    & = \Pi_{\mathcal{S}_j} \sum_{i=1}^q \Pi_{\mathcal{S}_i}\operatorname{Var}(X_t)^{-1}\Pi_{\mathcal{S}_i}\operatorname{Cov}(\Pi_{\mathcal{S}_i}X_t,Y_t) \\
    & = \Pi_{\mathcal{S}_j}\operatorname{Var}(X_t)^{-1}\Pi_{\mathcal{S}_j}\operatorname{Cov}(\Pi_{\mathcal{S}_j}X_t,Y_t) \\
    & = \operatorname{Var}(\Pi_{\mathcal{S}_j}X_t)^{\dagger}\operatorname{Cov}(\Pi_{\mathcal{S}_j}X_t, Y_t ),    
\end{align*}
where the first equality follows from the first part of this proof and and the third equality follows from Lemma~\ref{lem:pseudo_inv_projection}.
\end{proof}

\paragraph{Proof of Lemma~\ref{lem:subspace_maximization}}
\begin{proof}
      For all $j\in\{1,\dots,q\}$, for all $\beta\in\mathcal{S}_j$ and for all $t\in\mathcal{N}$ it holds that
    \begin{align}
        \operatorname{\Delta Var}_t(\beta) & = 2\operatorname{Cov}(Y_t, X_t)\beta - \beta^{\top}\operatorname{Var}(X_t)\beta \nonumber \\
        & =  2\operatorname{Cov}(Y_t, (\textstyle\sum_{k=1}^q\Pi_{\mathcal{S}_k}X_t))\Pi_{\mathcal{S}_j}\beta - \beta^{\top}\Pi_{\mathcal{S}_j}\operatorname{Var}(\textstyle\sum_{k=1}^q\Pi_{\mathcal{S}_k}X_t)\Pi_{\mathcal{S}_j}\beta  \nonumber\\
        & = 2\operatorname{Cov}(Y_t, (\Pi_{\mathcal{S}_j}X_t))\Pi_{\mathcal{S}_j}\beta - \beta^{\top}\Pi_{\mathcal{S}_j}\operatorname{Var}(\Pi_{\mathcal{S}_j}X_t)\Pi_{\mathcal{S}_j}\beta \nonumber\\
        & = 2\operatorname{Cov}(Y_t, (\Pi_{\mathcal{S}_j}X_t))\beta - \beta^{\top}\operatorname{Var}(\Pi_{\mathcal{S}_j}X_t)\beta,\label{eq:without_sum_expansion}
    \end{align}
    where the first equality follows by \eqref{eq:explained_var} and the third equality follows from the definition of an orthogonal partition. It follows that
    \begin{align*}
        \nabla( \operatorname{\Delta Var}_t)(\beta) = 2\operatorname{Cov}(\Pi_{\mathcal{S}_j}X_t, Y_t) - 2\operatorname{Var}(\Pi_{\mathcal{S}_j}X_t)\beta,
    \end{align*}
    where $\nabla$ denotes the gradient.
    The equation $\nabla (\operatorname{\Delta Var}_t)(\beta)=0$ has a unique solution in $\mathcal{S}_j$ given by $\operatorname{Var}(\Pi_{\mathcal{S}_j}X_t)^{\dagger}\operatorname{Cov}(\Pi_{\mathcal{S}_j}X_t, Y_t )$. To see this, observe that all other solutions in $\mathbb{R}^p$ are given, for an arbitrary vector $w\in\mathbb{R}^p$, by
    \begin{equation}
        \operatorname{Var}(\Pi_{\mathcal{S}_j}X_t)^{\dagger}\operatorname{Cov}(\Pi_{\mathcal{S}_j}X_t, Y_t ) + (I_p-\operatorname{Var}(\Pi_{\mathcal{S}_j}X_t)^{\dagger}\operatorname{Var}(\Pi_{\mathcal{S}_j}X_t))w
        \label{eq:pseudo_inv_regression_solutions}
    \end{equation} 
    where $I_p$ denotes the identity matrix. 
    Let $U\in\mathbb{R}^{p\times p}$ and $(\tilde{\Sigma}_t)_{t\in\mathcal{N}}$ be defined for the orthogonal and $(X_t)_{t\in\mathcal{N}}$-decorrelating partition as in Lemma~\ref{lem:orthogonal_partition_diagonalizer}.
    We now observe that
    \begin{align}
    \label{eq:pseudoinverse_identity}
        \operatorname{Var}(\Pi_{\mathcal{S}_j}X_t)^{\dagger}\operatorname{Var}(\Pi_{\mathcal{S}_j}X_t) & = \Pi_{\mathcal{S}_j}\Sigma_t^{-1}\Pi_{\mathcal{S}_j}\Pi_{\mathcal{S}_j}\Sigma_t\Pi_{\mathcal{S}_j}  \\
        & = U^{S_j}(U^{S_j})^{\top}\Sigma_t^{-1} U^{S_j}(U^{S_j})^{\top}\Sigma_t U^{S_j}(U^{S_j})^{\top} \notag\\
        & = U^{S_j} \tilde{\Sigma}_{t,j}^{-1} \tilde{\Sigma}_{t,j} (U^{S_j})^{\top} \notag\\
        & = \Pi_{\mathcal{S}_j}. \notag
    \end{align}
    where the first equality follows from Lemma~\ref{lem:orthogonal_inverse}, 
    since $U$ jointly block diagonalizes the matrices $(\Sigma_t)_{t\in\mathcal{N}}$ by Lemma~\ref{lem:orthogonal_partition_diagonalizer}.
    We can therefore rewrite \eqref{eq:pseudo_inv_regression_solutions} as
    \begin{equation*}
        \operatorname{Var}(\Pi_{\mathcal{S}_j}X_t)^{\dagger}\operatorname{Cov}(\Pi_{\mathcal{S}_j}X_t, Y_t ) + (I_p-\Pi_{\mathcal{S}_j})w.
    \end{equation*}
    For all $w \in \mathcal{S}_j$, this expression equals $\operatorname{Var}(\Pi_{\mathcal{S}_j}X_t)^{\dagger}\operatorname{Cov}(\Pi_{\mathcal{S}_j}X_t, Y_t )$. For all $w \not \in \mathcal{S}_j$, it is not in $\mathcal{S}_j$.
    This concludes the proof of Lemma~\ref{lem:subspace_maximization}.
\end{proof}

\begin{lemma}
    \label{lem:inv_res_partition}
    $\{\mathcal{S}^{\operatorname{inv}}, \mathcal{S}^{\operatorname{res}}\}$ is an orthogonal and $(X_t)_{t\in[n]}$-decorrelating partition.
\end{lemma}
\begin{proof}
    Let $\{\mathcal{S}_j\}_{j=1}^q$ be a fixed irreducible orthogonal and $(X_t)_{t\in[n]}$-decorrelating partition according to which $\mathcal{S}^{\operatorname{inv}}$ and $\mathcal{S}^{\operatorname{res}}$ are defined. By orthogonality of the subspaces in the partition and by definition of $\mathcal{S}^{\operatorname{inv}}$ and $\mathcal{S}^{\operatorname{res}}$, $\mathcal{S}^{\operatorname{res}}=(\mathcal{S}^{\operatorname{inv}})^{\perp}$. Moreover, by definition of orthogonal and $(X_t)_{t\in[n]}$-decorrelating partition, it holds that 
    \begin{align*}
        \operatorname{Cov}(\Pi_{\mathcal{S}^{\operatorname{inv}}}X_t, \Pi_{\mathcal{S}^{\operatorname{res}}}X_t) =  \operatorname{Cov}(\sum_{\substack{j\in\{1,\dots,q\}:\\ \mathcal{S}_j \text{ opt-invariant on } [n]}}\Pi_{\mathcal{S}_j}X_t, \sum_{\substack{j\in\{1,\dots,q\}:\\ \mathcal{S}_j \text{ not opt-invariant on } [n]}}\Pi_{\mathcal{S}_j}X_t) = 0.
    \end{align*}
\end{proof}

\begin{lemma}
\label{lem:invariant_subspace}
$\mathcal{S}^{\operatorname{inv}}$ is opt-invariant on $[n]$.
\end{lemma}
\begin{proof}
That $\mathcal{S}^{\operatorname{inv}}$ is opt-invariant on $[n]$ can be seen from the following computations. It holds for all $t,s\in[n]$ that 
\begin{align*}
    \argmax_{\beta\in\mathcal{S}^{\operatorname{inv}}} \operatorname{\Delta Var}_t(\beta) & = \operatorname{Var}(\Pi_{\mathcal{S}^{\operatorname{inv}}}X_t)^{\dagger}\operatorname{Cov}(\Pi_{\mathcal{S}^{\operatorname{inv}}}X_t, Y_t) \\
    & = \sum_{\substack{j\in\{1,\dots,q\}:\\\mathcal{S}_j \text{ opt-invariant on }[n]}} \operatorname{Var}(\Pi_{\mathcal{S}_j}X_t)^{\dagger}\operatorname{Cov}(\Pi_{\mathcal{S}_j}X_t, Y_t) \\
    & = \sum_{\substack{j\in\{1,\dots,q\}:\\\mathcal{S}_j \text{ opt-invariant on }[n]}} \argmax_{\beta\in\mathcal{S}_j} \operatorname{\Delta Var}_t(\beta) \\
    & = \sum_{\substack{j\in\{1,\dots,q\}:\\\mathcal{S}_j \text{ opt-invariant on }[n]}} \argmax_{\beta\in\mathcal{S}_j} \operatorname{\Delta Var}_s(\beta) \\
    & = \argmax_{\beta\in\mathcal{S}^{\operatorname{inv}}} \operatorname{\Delta Var}_s(\beta).
\end{align*}
The first equality holds by Lemma~\ref{lem:subspace_maximization}, since $\{\mathcal{S}^{\operatorname{inv}}, \mathcal{S}^{\operatorname{res}}\}$ is indeed an orthogonal and $(X_t)_{t\in[n]}$-decorrelating partition, see Lemma~\ref{lem:inv_res_partition}. The second equality follows from the definition of $\mathcal{S}^{\operatorname{inv}}$ and can be proved by Lemma~\ref{lem:time_var_parameter_separation} 
and observing that the set of subspaces $\{\mathcal{S}_j\mid j\in\{1,\dots,q\}: \mathcal{S}_j \text{ opt-invariant on }[n]\}$ is an orthogonal and $(\Pi_{\mathcal{S}^{\operatorname{inv}}}X_t)_{t\in[n]}$-decorrelating partition.
The third equality holds by Lemma~\ref{lem:subspace_maximization} and the fourth by definition of an opt-invariant subspace on $[n]$. 
\end{proof}

\begin{lemma}
\label{lem:jbd_uniqueness}
    Let $\{A_t\}_{t=1}^n$ be a set of $n$ symmetric strictly positive definite matrices. If there exists a matrix $A\in\{A_t\}_{t=1}^n$ that has all distinct eigenvalues, then any two irreducible joint block diagonalizers  $U, \tilde{U}$ for the set $\{A_t\}_{t=1}^n$ are equal up to block permutations and block-wise isometric transformations. 
\end{lemma}
\begin{proof}
We start by observing that if a matrix is symmetric and has all distinct eigenvalues, then its eigenvectors are orthogonal to each other and unique up to scaling. We define $Q\in\mathbb{R}^{p\times p}$ as the orthonormal matrix whose columns are eigenvectors for $A$: $Q$ is then uniquely defined up to permutations of its columns. 

We now exploit the results by \citet{murota2010numerical} used in the construction of an irreducible orthogonal joint block diagonalizer for a set of symmetric matrices $\{A_t\}_{t=1}^n$ (not necessarily containing a matrix with all distinct eigenvalues). 
In the following, we translate all the useful results by~\citet{murota2010numerical} in our notation introduced for joint block diagonalizers in Section~\ref{sec:orthogonal_joint_block_diagonalization}.
\citet{murota2010numerical} show in Theorem~$1$ that there exists an irreducible orthogonal joint block diagonalizer $U$ such that, for all $t\in[n]$, $U^{\top}A_tU=\bigoplus_{j=1}^q(I_{\bar{m}_j}\otimes A_{t,j})$, where $q,\bar{m}_j\in\mathbb{N}$, $0<q, \bar{m}_j\le p$ are such that $q^U_{\max}=\sum_{j=1}^q\bar{m}_j$, and $A_{t,j}$ are square matrices (common diagonal blocks). Here $\oplus$ denotes the direct sum operator for matrices and $\otimes$ the Kronecker product. 
They further propose to partition the columns of the matrix $U$ into $q$ subsets, each denoted by $U^{S_j}$, with $j\in\{1,\dots,q\}$  indexing the diagonal blocks and $S_j\subseteq\{1,\dots,p\}$ denoting the subset of indexes corresponding to the selected columns in $U^{S_j}$, such that, for all $t\in[n]$, $(U^{S_j})^{\top}A_tU^{S_j}=I_{\bar{m}_j}\otimes A_{t,j}$. They then argue that, as a consequence of Theorem~$1$, the spaces spanned by such subsets of columns, i.e., 
\begin{equation*}
    \mathcal{U}_j\coloneqq \operatorname{span}\{u^k\mid k\in S_j\}
\end{equation*}are uniquely defined.  
We therefore only need to prove that, if at least one matrix in the set $\{A_t\}_{t=1}^n$
has all distinct eigenvalues, then for all $j\in\{1,\dots,q\}$, $\Bar{m}_j=1$. Such condition implies that the diagonal blocks indexed by $S_j$, $j\in\{1,\dots,q\}$, are irreducible, and $q^U_{\max}=q$.

The uniqueness of the spaces $\mathcal{U}_j$ then implies the uniqueness of the irreducible joint block diagonalizer $U$ up to block permutations and block-wise isometric transformations.
To show this result, we now use \citet[Proposition~1]{murota2010numerical}. 
In particular, if the set $\{A_t\}_{t=1}^n$ contains at least one matrix $A$ with all distinct eigenvalues then the assumptions of Proposition~$1$ are satisfied. Let $\{\mathcal{Q}_1,\dots, \mathcal{Q}_p\}$ be the set of eigenspaces of $A$ and for all $i\in\{1,\dots,p\}$ let $m_i\coloneqq\operatorname{dim}(\mathcal{Q}_i)=1$.
The proposition then implies that for all $i\in\{1,\dots,p\}$ there exists $j\in\{1,\dots,q\}$ such that $\mathcal{Q}_i\subseteq\mathcal{U}_j$. Moreover, for all $i$ such that $\mathcal{Q}_i\subseteq\mathcal{U}_j$ it holds that $m_i=\Bar{m}_j$. As a consequence, we obtain that, since the eigenvalues of $A$ are distinct, $m_1=\dots=m_p=1$ and therefore for all $j\in\{1,\dots,q\}$, $\bar{m}_j=1$. This concludes the proof for Lemma~\ref{lem:jbd_uniqueness}.
\end{proof}

\begin{lemma}
\label{lem:opt_invariant_projection}
    Let $\{\mathcal{S}_j\}_{j=1}^q$ be an irreducible orthogonal and $(X_t)_{t\in[n]}$-decorrelating partition. Then, for all $j\in\{1,\ldots,q\}$ it holds that
\begin{equation*}
\argmax_{\beta\in\mathcal{S}_j}\overline{\operatorname{\Delta Var}}(\beta)
    =\left(\frac{1}{n}\sum_{t=1}^n\operatorname{Var}(\Pi_{\mathcal{S}_j}X_t)\right)^{\dagger}\frac{1}{n}\sum_{t=1}^n\operatorname{Cov}(\Pi_{\mathcal{S}_j}X_t, Y_t ).
\end{equation*}
In addition,
    for all $j\in\{1,\dots,q\}$ such that $\mathcal{S}_j$ is opt-invariant on $[n]$, it holds for all $t\in[n]$ that 
    \begin{equation*}
        \argmax_{\beta\in\mathcal{S}_j}\overline{\operatorname{\Delta Var}}(\beta) = \Pi_{\mathcal{S}_j}\gamma_{0,t}=\Pi_{\mathcal{S}_j}\bar{\gamma}_{0}.
    \end{equation*}
    Moreover, $\Pi_{\mathcal{S}_j}\bar{\gamma}_{0}$ is time-invariant on $[n]$.
    
\end{lemma}
\begin{proof}
Let $U\in\mathbb{R}^{p\times p}$ and $(\tilde{\Sigma}_t)_{t\in[n]}$ be defined for the orthogonal and $(X_t)_{t\in[n]}$-decorrelating partition as in Lemma~\ref{lem:orthogonal_partition_diagonalizer}.
For all $j\in\{1,\dots, q\}$ 
and for all $\beta\in\mathcal{S}_j$ it follows from \eqref{eq:without_sum_expansion} that
\begin{equation*}
    \overline{\operatorname{\Delta Var}}(\beta) =  \frac{2}{n}\sum_{t=1}^n\operatorname{Cov}(Y_t, \Pi_{\mathcal{S}_j}X_t)\beta - \beta^{\top}\frac{1}{n}\sum_{t=1}^n\operatorname{Var}(\Pi_{\mathcal{S}_j}X_t)\beta 
\end{equation*}
which has gradient
\begin{equation*}
    \nabla ( \overline{\operatorname{\Delta Var}})(\beta) =  \frac{2}{n}\sum_{t=1}^n\operatorname{Cov}(\Pi_{\mathcal{S}_j}X_t, Y_t) - \frac{2}{n}\sum_{t=1}^n\operatorname{Var}(\Pi_{\mathcal{S}_j}X_t)\beta.
\end{equation*}
The equation $ \nabla (\overline{\operatorname{\Delta Var}})(\beta) = 0$ has solution equal to
\begin{equation*}
    \beta^* = \left(\frac{1}{n}\sum_{t=1}^n\operatorname{Var}(\Pi_{\mathcal{S}_j}X_t)\right)^{\dagger}\frac{1}{n}\sum_{t=1}^n\operatorname{Cov}(\Pi_{\mathcal{S}_j}X_t, Y_t ).
\end{equation*}
By Lemma~\ref{lem:orthogonal_inverse}, it holds that $\left(\frac{1}{n}\sum_{t=1}^n\operatorname{Var}(\Pi_{\mathcal{S}_j}X_t)\right)^{\dagger} = \Pi_{\mathcal{S}_j}\left(\frac{1}{n}\sum_{t=1}^n\operatorname{Var}(X_t)\right)^{-1}\Pi_{\mathcal{S}_j}$ and therefore $\beta^*\in\mathcal{S}_j$. In order to apply Lemma~\ref{lem:orthogonal_inverse}, we in particular use that $U$ is such 
that, for all $t\in [n]$, $\tilde{\Sigma}_t = U^{\top}\Sigma_tU$ is block diagonal with diagonal blocks given, for all $j\in\{1,\dots,q\}$, by $\tilde{\Sigma}_{t,j}=(U^{S_j})^{\top}\Sigma_tU^{S_j}$. This implies that $U^{\top}\left(\sum_{t=1}^n\Sigma_t\right)U$ is also block diagonal with diagonal blocks $(U^{S_j})^{\top}\left(\sum_{t=1}^n\Sigma_t\right)U^{S_j}$. Moreover, its inverse is block diagonal and, by the properties of orthogonal matrices, its diagonal blocks are $(U^{S_j})^{\top}\left(\sum_{t=1}^n\Sigma_t\right)^{-1}U^{S_j}$.
It now remains to show that this is also the only solution in $\mathcal{S}_j$. All other solutions in $\mathbb{R}^p$ are given, for an arbitrary vector $w\in\mathbb{R}^p$, by 
\begin{align*}
   &\beta^* + \left(I_p-\left(\frac{1}{n}\sum_{t=1}^n\operatorname{Var}(\Pi_{\mathcal{S}_j}X_t)\right)^{\dagger}\left(\frac{1}{n}\sum_{t=1}^n\operatorname{Var}(\Pi_{\mathcal{S}_j}X_t)\right)\right)w  \\
   = & \beta^* + \left(I_p-\Pi_{\mathcal{S}_j}\right)w.
\end{align*}
The equality follows from the following computation
\begin{align*}
    &\quad\left(\frac{1}{n}\sum_{t=1}^n\operatorname{Var}(\Pi_{\mathcal{S}_j}X_t)\right)^{\dagger}\left(\frac{1}{n}\sum_{t=1}^n\operatorname{Var}(\Pi_{\mathcal{S}_j}X_t)\right)\\ & = \Pi_{\mathcal{S}_j}\left(\sum_{t=1}^n\operatorname{Var}(X_t)\right)^{-1}\Pi_{\mathcal{S}_j} \Pi_{\mathcal{S}_j}\left(\sum_{t=1}^n\operatorname{Var}(X_t)\right)\Pi_{\mathcal{S}_j} \\
    & = U^{S_j}(U^{S_j})^{\top}\left(\sum_{t=1}^n\Sigma_t\right)^{-1}U^{S_j}(U^{S_j})^{\top}\left(\sum_{t=1}^n\Sigma_t\right)U^{S_j}(U^{S_j})^{\top} \\
    & = U^{S_j}\left(\sum_{t=1}^n\tilde{\Sigma}_t\right)^{-1}\left(\sum_{t=1}^n\tilde{\Sigma}_t\right)(U^{S_j})^{\top} \\
    & = \Pi_{\mathcal{S}_j}.
\end{align*}
The first equality follows again from Lemma~\ref{lem:orthogonal_inverse}.
For all $w \in \mathcal{S}_j$, $\beta^* + (I_p-\Pi_{\mathcal{S}_j})w$ equals $\beta^*$. For all $w \not \in \mathcal{S}_j$, $\beta^* + (I_p-\Pi_{\mathcal{S}_j})w$
is not in $\mathcal{S}_j$.

    For all $\mathcal{S}_j$ opt-invariant on $[n]$ and for all $t\in[n]$ it holds that
    \begin{align*}
        & \quad \argmax_{\beta\in\mathcal{S}_j}\overline{\operatorname{\Delta Var}}(\beta)\\
        & = \argmax_{\beta\in\mathcal{S}_j}\frac{1}{n}\sum_{s=1}^n\operatorname{\Delta Var}_s(\beta)\\
        & = \argmax_{\beta\in\mathcal{S}_j}\operatorname{\Delta Var}_{t}(\beta)\\
        & = \operatorname{Var}(\Pi_{\mathcal{S}_j}X_{t})^{\dagger}\operatorname{Cov}(\Pi_{\mathcal{S}_j}X_{t}, Y_{t})\\
        & = \Pi_{\mathcal{S}_j}\gamma_{0,t},
    \end{align*}
    where we used the definition of opt-invariance on $[n]$ for the second equality, Lemma~\ref{lem:subspace_maximization} for the third equality and Lemma~\ref{lem:time_var_parameter_separation} for the  fourth equality. Since the result holds for all $t\in[n]$, it also follows that 
    \begin{equation*}
        \Pi_{\mathcal{S}_j}\gamma_{0,t}=\Pi_{\mathcal{S}_j}\bar{\gamma}_0.
    \end{equation*}
 We now need to prove for all $t\in[n]$ that $\operatorname{Cov}(Y_t-X_t^{\top}\Pi_{\mathcal{S}_j}\bar{\gamma}_{0}, X_t^{\top}\Pi_{\mathcal{S}_j}\bar{\gamma}_{0})=0$. To see this, fix $t\in[n]$. Then
    \begin{align*}
        \operatorname{Cov}(Y_t-X_t^{\top}\Pi_{\mathcal{S}_j}\bar{\gamma}_{0}, X_t^{\top}\Pi_{\mathcal{S}_j}\bar{\gamma}_{0}) & = \operatorname{Cov}(X_t^{\top}(\gamma_{0,t}-\Pi_{\mathcal{S}_j}\gamma_{0,t}), X_t^{\top}\Pi_{\mathcal{S}_j}\gamma_{0,t}) \\
        & = \operatorname{Cov}(X_t^{\top}(\sum_{i=1}^q\Pi_{\mathcal{S}_i}\gamma_{0,t}-\Pi_{\mathcal{S}_j}\gamma_{0,t}), X_t^{\top}\Pi_{\mathcal{S}_j}\gamma_{0,t}) \\
        & = \operatorname{Cov}(X_t^{\top}\sum_{i\in\{1,\dots,q\}:i\neq j}\Pi_{\mathcal{S}_i}\gamma_{0,t}, X_t^{\top}\Pi_{\mathcal{S}_j}\gamma_{0,t}) \\
        & = \sum_{i\in\{1,\dots,q\}:i\neq j} \gamma_{0,t}^{\top} \operatorname{Cov}(\Pi_{\mathcal{S}_i}X_t, \Pi_{\mathcal{S}_j}X_t)\gamma_{0,t} \\
        & = 0.
    \end{align*}
    The last equality follows from the definition of an orthogonal and $(X_t)_{t\in[n]}$-decorrelating partition.
    
\end{proof}

\begin{lemma} 
\label{lem:invariant_component_pop_estimator}
Let $\{\mathcal{S}_j\}_{j=1}^q$ be an irreducible orthogonal and $(X_t)_{t\in[n]}$-decorrelating partition and let $\mathcal{S}^{\operatorname{inv}}$ be the corresponding invariant subspace. Moreover, 
let $\overline{\operatorname{Var}}(X)\coloneqq\frac{1}{n}\sum_{t=1}^n\operatorname{Var}(X_t)$ and $\overline{\operatorname{Cov}}(X, Y)\coloneqq\frac{1}{n}\sum_{t=1}^n\operatorname{Cov}(X_t, Y_t)$. Finally, let $U^{\operatorname{inv}}$ be the submatrix of an arbitrary irreducible joint block diagonalizer $U$ corresponding to the irreducible orthogonal and $(X_t)_{t\in[n]}$-decorrelating partition whose columns span $\mathcal{S}^{\operatorname{inv}}$. Then,
    \begin{align*}
        \beta^{\operatorname{inv}} 
        & = U^{\operatorname{inv}} ((U^{\operatorname{inv}})^{\top}\overline{\operatorname{Var}}(X)U^{\operatorname{inv}})^{-1}(U^{\operatorname{inv}})^{\top}\overline{\operatorname{Cov}}(X, Y).
    \end{align*}
\end{lemma}
\begin{proof}
Expanding the definition of $\beta^{\operatorname{inv}}$, we obtain that
    \begin{align*}
        \beta^{\operatorname{inv}}
        & = \argmax_{\beta\in\mathcal{S}^{\operatorname{inv}}}\overline{\operatorname{\Delta Var}}(\beta)\\
        & = \argmax_{\beta\in\mathcal{S}^{\operatorname{inv}}}\frac{1}{n}\sum_{t=1}^n\operatorname{\Delta Var}_t(\beta)\\
        & = \argmax_{\beta\in\mathcal{S}^{\operatorname{inv}}}\frac{1}{n}\sum_{s=1}^n (2\operatorname{Cov}(Y_t, \Pi_{\mathcal{S}^{\operatorname{inv}}}X_t)\beta - \beta^{\top}\operatorname{Var}(\Pi_{\mathcal{S}^{\operatorname{inv}}}X_t)\beta) \\
        & = \left(\frac{1}{n}\sum_{t=1}^n \operatorname{Var}(\Pi_{\mathcal{S}^{\operatorname{inv}}}X_{t})\right)^{\dagger}\frac{1}{n}\sum_{t=1}^n\operatorname{Cov}(\Pi_{\mathcal{S}^{\operatorname{inv}}}X_{t}, Y_{t})\\
        &= \Pi_{\mathcal{S}^{\operatorname{inv}}}\overline{\operatorname{Var}}(X)^{-1}\Pi_{\mathcal{S}^{\operatorname{inv}}}\overline{\operatorname{Cov}}(X, Y) \\
        & = U^{\operatorname{inv}} ((U^{\operatorname{inv}})^{\top}\overline{\operatorname{Var}}(X)U^{\operatorname{inv}})^{-1}(U^{\operatorname{inv}})^{\top}\overline{\operatorname{Cov}}(X, Y).
    \end{align*} 
    The fourth equality follows from Lemma~\ref{lem:opt_invariant_projection}. The fifth equality follows by Lemma~\ref{lem:orthogonal_inverse} (see the proof of Lemma~\ref{lem:opt_invariant_projection}).
    In the last equality we used that $(U^{\operatorname{inv}})^{\top}\overline{\operatorname{Var}}(X)^{-1}U^{\operatorname{inv}}= ((U^{\operatorname{inv}})^{\top}\overline{\operatorname{Var}}(X)U^{\operatorname{inv}})^{-1}$, which follows from the properties of orthogonal matrices and the block diagonal structure of $U^{\top}\overline{\operatorname{Var}}(X)U$: these imply that  $( U^{\top}\overline{\operatorname{Var}}(X)U)^{-1} = U^{\top}\overline{\operatorname{Var}}(X)^{-1}U $ and
     \begin{align*}
           & \quad \begin{bmatrix}
            ((U^{S_1})^{\top}\overline{\operatorname{Var}}(X)U^{S_1})^{-1} & & \\
            & \ddots & \\
            & & ((U^{S_q})^{\top}\overline{\operatorname{Var}}(X)U^{S_q})^{-1}
            \end{bmatrix} \\ &=  \begin{bmatrix}
            (U^{S_1})^{\top}\overline{\operatorname{Var}}(X)^{-1}U^{S_1} & & \\
            & \ddots & \\
            & & (U^{S_q})^{\top}\overline{\operatorname{Var}}(X)^{-1}U^{S_q}
            \end{bmatrix}.
        \end{align*}
\end{proof}

\section{Proofs}

\subsection{Proof of Theorem~\ref{thm:subspace_partition}}
\label{pf:subspace_partition}

\begin{proof}
Let $\{\mathcal{S}_j\}_{j=1}^q$ be an irreducible orthogonal and $(X_t)_{t\in[n]}$-decorrelating partition and define $\mathcal{S}^{\operatorname{inv}}$ and $\mathcal{S}^{\operatorname{res}}$ as in \eqref{eq:invariant_residual_spaces}. By Assumption~\ref{ass:generalization}, $\{\mathcal{S}^{\operatorname{inv}}, \mathcal{S}^{\operatorname{res}}\}$ form an orthogonal and $(X_t)_{t\in\mathbb{N}}$-decorrelating partition and $\mathcal{S}^{\operatorname{inv}}$ is opt-invariant on $\mathbb{N}$. Therefore, using Lemma~\ref{lem:time_var_parameter_separation} and Lemma~\ref{lem:subspace_maximization}, we get for all $t\in\mathbb{N}$ that
\begin{equation*}
    \gamma_{0,t} = \argmax_{\beta\in\mathcal{S}^{\operatorname{inv}}}\operatorname{\Delta Var}_t(\beta) + \argmax_{\beta\in\mathcal{S}^{\operatorname{res}}}\operatorname{\Delta Var}_t(\beta).
\end{equation*}
Furthermore, $\mathcal{S}^{\operatorname{inv}}$ opt-invariant on $\mathbb{N}$
implies that the first term, $\argmax_{\beta\in\mathcal{S}^{\operatorname{inv}}}\operatorname{\Delta Var}_t(\beta) $, does not depend on $t$ and hence it holds for all $t\in\mathbb{N}$ that
\begin{align*}
    \gamma_{0,t} & =  \argmax_{\beta\in\mathcal{S}^{\operatorname{inv}}}\overline{\operatorname{\Delta Var}}(\beta) + \argmax_{\beta\in\mathcal{S}^{\operatorname{res}}}\operatorname{\Delta Var}_t(\beta).
\end{align*}

Moreover, Assumption~\ref{ass:uniqueness_time_inv_sp} ensures that the invariant and residual subspaces can be uniquely identified from an arbitrary irreducible orthogonal and $(X_t)_{t\in[n]}$-decorrelating partition, and therefore all the above results do not depend on the considered partition. This concludes the proof of Theorem~\ref{thm:subspace_partition}.
\end{proof}

\subsection{Proof of Proposition~\ref{prop:irreducible_orthogonal_partition}}
\begin{proof}
   (i) For all $j\in\{1,\dots,q^U_{\max}\}$, let $U^{S_j}$ denote the submatrix of $U$ formed by the columns indexed by $S_j$. Then, the orthogonal projection matrix onto the subspace $\mathcal{S}_j$ can be expressed as $\Pi_{\mathcal{S}_j}= U^{S_j}(U^{S_j})^{\top}$. It follows for all $t\in[n]$ that 
    \begin{align*}
         \operatorname{Cov}(\Pi_{\mathcal{S}_i}X_t, \Pi_{\mathcal{S}_j}X_t) & = \Pi_{\mathcal{S}_i} \Sigma_t \Pi_{\mathcal{S}_j} \\
         & = U^{S_i}(U^{S_i})^{\top} \Sigma_t U^{S_j}(U^{S_j})^{\top} \\
         & = 0,
    \end{align*}
    where the last equality holds since $(U^{S_i})^{\top} \Sigma_t U^{S_j}$ is the $(i,j)$-th (off-diagonal) block of $\tilde{\Sigma}_t \coloneqq U^{\top} \Sigma_t U$, which is zero due to the block diagonal structure of $\tilde{\Sigma}_t$. Irreducibility of the orthogonal partition follows from the irreducibility of the joint block diagonal decomposition. \\
    
    (ii)  The statement follows from Lemma~\ref{lem:orthogonal_partition_diagonalizer}, taking $\mathcal{N}=[n]$. Irreducibility of the joint block diagonalizer follows from the irreducibility of the orthogonal partition.
\end{proof}

\subsection{Proof of Proposition~\ref{prop:invariance}}

\begin{proof}
    (i) First, observe that, by Lemma~\ref{lem:inv_res_partition},
    $\{\mathcal{S}^{\operatorname{inv}}, \mathcal{S}^{\operatorname{res}}\}$ is an orthogonal and $(X_t)_{t\in[n]}$-decorrelating partition and that, by Lemma~\ref{lem:invariant_subspace}, $\mathcal{S}^{\operatorname{inv}}$ is opt-invariant on $[n]$. Then, by definition of $\beta^{\operatorname{inv}}$ and by Lemma~\ref{lem:opt_invariant_projection} we get for all $t\in[n]$ that
    \begin{align*}
        \beta^{\operatorname{inv}}
         = \argmax_{\beta\in\mathcal{S}^{\operatorname{inv}}}\overline{\operatorname{\Delta Var}}(\beta) = \Pi_{\mathcal{S}^{\operatorname{inv}}}\gamma_{0,t} = \Pi_{\mathcal{S}^{\operatorname{inv}}}\bar{\gamma}_0.
    \end{align*}

  (ii) We need to prove for all $t\in[n]$ that $\operatorname{Cov}(Y_t-X_t^{\top}\beta^{\operatorname{inv}}, X_t^{\top}\beta^{\operatorname{inv}})=0$. To see this, fix $t\in[n]$. Then
    \begin{align*}
        \operatorname{Cov}(Y_t-X_t^{\top}\beta^{\operatorname{inv}}, X_t^{\top}\beta^{\operatorname{inv}})
        &= \operatorname{Cov}(X_t^{\top}(\gamma_{0,t}-\beta^{\operatorname{inv}}), X_t^{\top}\beta^{\operatorname{inv}}) \\
        &= \operatorname{Cov}(X_t^{\top}(\Pi_{\mathcal{S}^{\operatorname{inv}}}\gamma_{0,t}+\Pi_{\mathcal{S}^{\operatorname{res}}}\gamma_{0,t}-\beta^{\operatorname{inv}}), X_t^{\top}(\Pi_{\mathcal{S}^{\operatorname{inv}}}\beta^{\operatorname{inv}}))\\
        &= \operatorname{Cov}(X_t^{\top}\Pi_{\mathcal{S}^{\operatorname{res}}}\gamma_{0,t}, X_t^{\top}\Pi_{\mathcal{S}^{\operatorname{inv}}}\beta^{\operatorname{inv}}) \\
        &= \gamma_{0,t}^{\top}\operatorname{Cov}(\Pi_{\mathcal{S}^{\operatorname{res}}}X_t, \Pi_{\mathcal{S}^{\operatorname{inv}}}X_t)\beta^{\operatorname{inv}} \\
        &= 0,
    \end{align*}
    where the third equality uses $\beta^{\operatorname{inv}}=\Pi_{\mathcal{S}^{\operatorname{inv}}}\gamma_{0,t}$ by Proposition~\ref{prop:invariance}~$(i)$ and the last equality follows from the fact that $\mathcal{S}^{\operatorname{inv}}$, 
    $\mathcal{S}^{\operatorname{res}}$ are an orthogonal and $(X_t)_{t \in [n]}$-decorrelating partition by Lemma~\ref{lem:inv_res_partition}.
 
(iii) Let $\mathcal{B}^n$ denote the set of all time-invariant parameters over $[n]$. By assumption, we have that $\mathcal{B}^n\subseteq \mathcal{S}^{\operatorname{inv}}$. Moreover, by point $(ii)$, we have that $\beta^{\operatorname{inv}}\in\mathcal{B}^n$. Therefore, by definition of $\beta^{\operatorname{inv}}$ we obtain that $\beta^{\operatorname{inv}}=\argmax_{\beta\in\mathcal{B}^n}\overline{\operatorname{\Delta Var}}(\beta)$.

     This completes the proof of Proposition~\ref{prop:invariance}.
\end{proof}

\subsection{Proof of Theorem~\ref{thm:inv_eff_maximin_optimal}}
\begin{proof}
    Under Assumptions~\ref{ass:uniqueness_time_inv_sp} and \ref{ass:generalization} and using the definitions of $\beta^{\operatorname{inv}}$ and $\delta^{\operatorname{res}}_t$, we can write the explained variance of $\beta\in\mathbb{R}^p$ at time $t\in\mathbb{N}$ for the true time varying parameter $\gamma_{0,t}\in\mathbb{R}^p$ as
    \begin{align*}
        \operatorname{\Delta Var}_t(\beta) & = 2\gamma_{0,t}^{\top}\operatorname{Var}(X_t)\beta - \beta^{\top}\operatorname{Var}(X_t)\beta \\
        & =  2(\beta^{\operatorname{inv}}+\delta^{\operatorname{res}}_t)^{\top}\operatorname{Var}((\Pi_{\mathcal{S}^{\operatorname{inv}}}+\Pi_{\mathcal{S}^{\operatorname{res}}})X_t)\beta - \beta^{\top}\operatorname{Var}(X_t)\beta\\
        & = 2(\beta^{\operatorname{inv}})^{\top}\operatorname{Var}(\Pi_{\mathcal{S}^{\operatorname{inv}}}X_t) \Pi_{\mathcal{S}^{\operatorname{inv}}}\beta + 2(\delta^{\operatorname{res}}_t)^{\top}\operatorname{Var}(\Pi_{\mathcal{S}^{\operatorname{res}}}X_t) \Pi_{\mathcal{S}^{\operatorname{res}}}\beta - \beta^{\top}\operatorname{Var}(X_t)\beta.
    \end{align*}
Using this expansion, we get, since $\delta^{\operatorname{res}}_t=\gamma_{0,t}-\beta^{\operatorname{inv}}$, for all $\beta\notin\mathcal{S}^{\operatorname{inv}}$ that 
    \begin{equation*}
        \inf_{\substack{\gamma_{0,t}\in\mathbb{R}^p:\\ \gamma_{0,t}-\beta^{\operatorname{inv}}\in\mathcal{S}^{\operatorname{res}}}} \operatorname{\Delta Var}_t(\beta) = -\infty.
    \end{equation*}
    Therefore, 
    \begin{align*}
    \argmax_{\beta\in\mathbb{R}^p} \inf_{\substack{\gamma_{0,t}\in\mathbb{R}^p:\\ \gamma_{0,t}-\beta^{\operatorname{inv}}\in\mathcal{S}^{\operatorname{res}}}} \operatorname{\Delta Var}_t(\beta) &= \argmax_{\beta\in\mathcal{S}^{\operatorname{inv}}} \inf_{\substack{\gamma_{0,t}\in\mathbb{R}^p:\\ \gamma_{0,t}-\beta^{\operatorname{inv}}\in\mathcal{S}^{\operatorname{res}}}} \operatorname{\Delta Var}_t(\beta)\\
    & = \argmax_{\beta\in\mathcal{S}^{\operatorname{inv}}} \operatorname{\Delta Var}_t(\beta).
\end{align*}
Since by assumption it holds that $\mathcal{S}^{\operatorname{inv}}$ is opt-invariant on $\mathbb{N}$, it further holds that for all $t\in\mathbb{N}$
\begin{equation*}
     \argmax_{\beta\in\mathcal{S}^{\operatorname{inv}}} \operatorname{\Delta Var}_t(\beta) = \argmax_{\beta\in\mathcal{S}^{\operatorname{inv}}} \overline{\operatorname{\Delta Var}(\beta)}.
\end{equation*}
The claim follows from the definition of $\beta^{\operatorname{inv}}$. 
\end{proof}

\subsection{Proof of Theorem~\ref{thm:empirical_explained_variance}}
\begin{proof}
    We assume without loss of generality that the observed predictors $X_t$ have zero mean in $t\in\mathcal{I}^{\operatorname{ad}}\cup t^*$. Alternatively, as mentioned in Section~\ref{sec:subspaces_estimation}, a constant term could be added to $X_t$ to account for the mean. We also observe that $\hat{\gamma}_{t^*}^{\operatorname{OLS}}$ and $\hat{\delta}_{t^*}^{\operatorname{res}}$ are both unbiased estimators for $\gamma_{0,t^*}$ and $\delta^{\operatorname{res}}_{t^*}$, respectively (this can be checked using standard OLS analysis).
    We now start by computing the out-of-sample MSPE  for $\hat{\gamma}^{\operatorname{\isd}}_{t^*}$.
    \begin{align*}
        \operatorname{MSPE}(\hat{\gamma}^{\operatorname{\isd}}_{t^*}) & = \mathbb{E}[(X_{t^*}^{\top}(\hat{\gamma}^{\operatorname{\isd}}_{t^*}-\gamma_{0,t^*}))^2] \\
        & = \mathbb{E}[(X_{t^*}^{\top}(\Pi_{\mathcal{S}^{\operatorname{inv}}}+ \Pi_{\mathcal{S}^{\operatorname{res}}})(\hat{\beta}^{\operatorname{inv}}+ \hat{\delta}^{\operatorname{res}}_{t^*}-\beta^{\operatorname{inv}} - \delta^{\operatorname{res}}_{t^*}))^2] \\
        & = \operatorname{trace}(\mathbb{E}[(\hat{\beta}^{\operatorname{inv}}-\beta^{\operatorname{inv}})(\hat{\beta}^{\operatorname{inv}}-\beta^{\operatorname{inv}})^{\top}]\operatorname{Var}(\Pi_{\mathcal{S}^{\operatorname{inv}}}X_{t^*})) \\
        &\quad + \operatorname{trace}(\mathbb{E}[(\hat{\delta}^{\operatorname{res}}_{t^*}-\delta^{\operatorname{res}}_{t^*})(\hat{\delta}^{\operatorname{res}}_{t^*}-\delta^{\operatorname{res}}_{t^*})^{\top}]\operatorname{Var}(\Pi_{\mathcal{S}^{\operatorname{res}}}X_{t^*}))\\
          & =  \operatorname{trace}(\mathbb{E}[\operatorname{Var}(\hat{\beta}^{\operatorname{inv}}\mid \mathbf{X})] \operatorname{Var}(\Pi_{\mathcal{S}^{\operatorname{inv}}}X_{t^*}))\\
        &\quad + \operatorname{trace}(\mathbb{E}[\operatorname{Var}(\hat{\delta}^{\operatorname{res}}_{t^*}\mid \mathbf{X}^{\operatorname{ad}})] \operatorname{Var}(\Pi_{\mathcal{S}^{\operatorname{res}}}X_{t^*})) \\
        & = \operatorname{trace}(\mathbb{E}[\operatorname{Var}(\hat{\beta}^{\operatorname{inv}}\mid \mathbf{X})] \operatorname{Var}(\Pi_{\mathcal{S}^{\operatorname{inv}}}X_{t^*}))\\
        &\quad + \tfrac{\sigma^2_{\operatorname{ad}}}{m} \operatorname{trace}(\Pi_{\mathcal{S}^{\operatorname{res}}}\mathbb{E}[(\tfrac{1}{m}(\mathbf{X}^{\operatorname{ad}})^{\top}\mathbf{X}^{\operatorname{ad}})^{-1}]\Pi_{\mathcal{S}^{\operatorname{res}}}\operatorname{Var}(\Pi_{\mathcal{S}^{\operatorname{res}}}X_{t^*})),
    \end{align*}
    where we have used that $\operatorname{Var}({\hat{\delta}^{\operatorname{res}}_{t^*}\mid \mathbf{X}^{\operatorname{ad}}}) = \sigma_{\operatorname{ad}}^{2} \Pi_{\mathcal{S}^{\operatorname{res}}}((\mathbf{X}^{\operatorname{ad}})^{\top}\mathbf{X}^{\operatorname{ad}})^{-1}\Pi_{\mathcal{S}^{\operatorname{res}}}$. We further observe that 
    \begin{equation*}
        \operatorname{Var}(\hat{\beta}^{\operatorname{inv}}\mid \mathbf{X}) = \Pi_{\mathcal{S}^{\operatorname{inv}}} (\mathbf{X}^{\top}\mathbf{X})^{-1}\Pi_{\mathcal{S}^{\operatorname{inv}}} \mathbf{X}^{\top}\operatorname{Var}(\mathbf{\epsilon})\mathbf{X}\Pi_{\mathcal{S}^{\operatorname{inv}}} (\mathbf{X}^{\top}\mathbf{X})^{-1}\Pi_{\mathcal{S}^{\operatorname{inv}}},
    \end{equation*}
    where $\operatorname{Var}(\mathbf{\epsilon})$ is a $n\times n$ diagonal matrix whose diagonal elements are the error variances at each observed time steps, $(\operatorname{Var}(\epsilon_t))_{t\in [n]}$. Let $\sigma_{\epsilon,\min}^{2}\coloneqq \min_{t\in[n]}\operatorname{Var}(\epsilon_t)$ and $\sigma_{\epsilon,\max}^{2}\ge \max_{t\in[n]}\operatorname{Var}(\epsilon_t)$. Then, 
    \begin{equation*}
        \sigma_{\epsilon,\min}^{2}\Pi_{\mathcal{S}^{\operatorname{inv}}} (\mathbf{X}^{\top}\mathbf{X})^{-1}\Pi_{\mathcal{S}^{\operatorname{inv}}}\preceq \operatorname{Var}(\hat{\beta}^{\operatorname{inv}}\mid\mathbf{X})\preceq\sigma_{\epsilon,\max}^{2}\Pi_{\mathcal{S}^{\operatorname{inv}}} (\mathbf{X}^{\top}\mathbf{X})^{-1}\Pi_{\mathcal{S}^{\operatorname{inv}}},
    \end{equation*}
    where $\preceq$ denotes the Loewner order, and it follows from $\operatorname{diag}_n(\sigma_{\epsilon,\min}^2)\preceq \operatorname{Var}(\epsilon_t)\preceq \operatorname{diag}_n(\sigma_{\epsilon,\max}^2)$, with $\operatorname{diag}_n(\cdot)$ denoting an $n$-dimensional diagonal matrix with diagonal elements all equal. We have used in particular that for two symmetric matrices $A, B\in\mathbb{R}^{n\times n}$ such that $A\succeq B$ and a matrix $S\in\mathbb{R}^{n\times n}$ it holds that $S^{\top}AS\succeq S^{\top}BS$ \citep[see, for example, Theorem 7.7.2.\ by][]{horn2012matrix}.
    Using this relation and Jensen's inequality, we obtain that the first term in $\operatorname{MSPE}(\hat{\gamma}^{\operatorname{\isd}}_{t^*})$ is lower bounded by
    \begin{align*}
        & \operatorname{trace}(\mathbb{E}[\operatorname{Var}(\hat{\beta}^{\operatorname{inv}}\mid \mathbf{X})] \operatorname{Var}(\Pi_{\mathcal{S}^{\operatorname{inv}}}X_{t^*})) \\
         \ge & \sigma_{\epsilon,\min}^2 \operatorname{trace}(\Pi_{\mathcal{S}^{\operatorname{inv}}}\mathbb{E}[(\mathbf{X}^{\top}\mathbf{X})^{-1}]\Pi_{\mathcal{S}^{\operatorname{inv}}}\operatorname{Var}(\Pi_{\mathcal{S}^{\operatorname{inv}}}X_{t^*})) \\
         \ge & \tfrac{\sigma_{\epsilon,\min}^2}{n} \operatorname{trace}(\Pi_{\mathcal{S}^{\operatorname{inv}}}\mathbf{\Sigma}^{-1}\Pi_{\mathcal{S}^{\operatorname{inv}}}\Sigma_{t^*}\Pi_{\mathcal{S}^{\operatorname{inv}}}),
         \end{align*}
         where $\mathbf{\Sigma}\coloneqq \mathbb{E}[\frac{1}{n}\mathbf{X}^{\top}\mathbf{X}]$. Using now that, by assumption, $\{\mathcal{S}^{\operatorname{inv}}, \mathcal{S}^{\operatorname{res}}\}$ is an orthogonal and $(X_t)_{t\in\mathbb{N}}$-decorrelating partition, let $U\in\mathbb{R}^{p\times p}$ be defined for such a partition as in Lemma~\ref{lem:orthogonal_partition_diagonalizer} and let
          $U^{\operatorname{inv}}\in\mathbb{R}^{p\times\operatorname{dim}(\mathcal{S}^{\operatorname{inv}})}$ be the submatrix of $U$ whose columns form an orthonormal basis for $\mathcal{S}^{\operatorname{inv}}$, such that $\Pi_{\mathcal{S}^{\operatorname{inv}}}=U^{\operatorname{inv}}(U^{\operatorname{inv}})^{\top}$. Moreover, let  $\tilde{\mathbf{\Sigma}}^{\operatorname{inv}}\coloneqq (U^{\operatorname{inv}})^{\top}\mathbf{\Sigma}U^{\operatorname{inv}}$ and $\tilde{\Sigma}_{t^*}^{\operatorname{inv}}\coloneqq (U^{\operatorname{inv}})^{\top}\Sigma_{t^*}U^{\operatorname{inv}}$ denote, as in Lemma~\ref{lem:orthogonal_partition_diagonalizer}, the diagonal block corresponding to $\mathcal{S}^{\operatorname{inv}}$ of the block diagonal matrices $U^{\top}\mathbf{\Sigma}U$ and $U^{\top}\Sigma_{t^*}U$, respectively. By the properties of orthogonal matrices, $(\tilde{\mathbf{\Sigma}}^{\operatorname{inv}})^{-1} =  (U^{\operatorname{inv}})^{\top}\mathbf{\Sigma}^{-1}U^{\operatorname{inv}}$. For all $i\in\{1,\dots,\operatorname{dim}(\mathcal{S}^{\operatorname{inv}})\}$, let $\lambda_i$ denote the $i$-th eigenvalue in decreasing order. Then, we can further express the above lower bound as
         \begin{align*}
         & \tfrac{\sigma_{\epsilon,\min}^2}{n} \operatorname{trace}(\Pi_{\mathcal{S}^{\operatorname{inv}}}\mathbf{\Sigma}^{-1}\Pi_{\mathcal{S}^{\operatorname{inv}}}\Sigma_{t^*}\Pi_{\mathcal{S}^{\operatorname{inv}}}) \\
         = &  \tfrac{\sigma_{\epsilon,\min}^2}{n} \operatorname{trace} (U^{\operatorname{inv}}(U^{\operatorname{inv}})^{\top}\mathbf{\Sigma}^{-1}U^{\operatorname{inv}}(U^{\operatorname{inv}})^{\top}\Sigma_{t^*}U^{\operatorname{inv}}(U^{\operatorname{inv}})^{\top}) \\
         = & \tfrac{\sigma_{\epsilon,\min}^2}{n} \operatorname{trace} (U^{\operatorname{inv}}(\tilde{\mathbf{\Sigma}}^{\operatorname{inv}})^{-1}\tilde{\Sigma}_{t^*}^{\operatorname{inv}}(U^{\operatorname{inv}})^{\top})\\
         \ge &\sigma_{\epsilon,\min}^2\tfrac{\operatorname{dim}(\mathcal{S}^{\operatorname{inv}})}{n} \tfrac{\lambda_{\min}(\tilde{\Sigma}_{t^*}^{\operatorname{inv}})}{\lambda_{\max}(\tilde{\mathbf{\Sigma}}^{\operatorname{inv}})}\\
         = & \sigma_{\epsilon,\min}^2\tfrac{\operatorname{dim}(\mathcal{S}^{\operatorname{inv}})}{n} c_{\operatorname{inv}} ,
    \end{align*}
    where we define $c_{\operatorname{inv}}\coloneqq \lambda_{\min}(\tilde{\Sigma}_{t^*}^{\operatorname{inv}})/\lambda_{\max}(\tilde{\mathbf{\Sigma}}^{\operatorname{inv}})$. 
    The same term in $\operatorname{MSPE}(\hat{\gamma}^{\operatorname{\isd}}_{t^*})$ is upper bounded by
    \begin{align*}
         & \operatorname{trace}(\mathbb{E}[\operatorname{Var}(\hat{\beta}^{\operatorname{inv}}\mid \mathbf{X})] \operatorname{Var}(\Pi_{\mathcal{S}^{\operatorname{inv}}}X_{t^*}))  \\
         \le & \sigma_{\epsilon,\max}^2 \operatorname{trace}(\Pi_{\mathcal{S}^{\operatorname{inv}}}\mathbb{E}[(\mathbf{X}^{\top}\mathbf{X})^{-1}]\Pi_{\mathcal{S}^{\operatorname{inv}}}\operatorname{Var}(\Pi_{\mathcal{S}^{\operatorname{inv}}}X_{t^*})) \\
         \le & \sigma_{\epsilon,\max}^2\tfrac{\operatorname{dim}(\mathcal{S}^{\operatorname{inv}})}{n} C_{\operatorname{inv}},
    \end{align*}
    where $C_{\operatorname{inv}}\coloneqq 1+ \operatorname{trace}(\Pi_{\mathcal{S}^{\operatorname{inv}}}(\mathbb{E}[(\tfrac{1}{n}\mathbf{X}^{\top}\mathbf{X})^{-1}]-\Sigma_{t^*}^{-1})\Pi_{\mathcal{S}^{\operatorname{inv}}}\operatorname{Var}(\Pi_{\mathcal{S}^{\operatorname{inv}}}X_{t^*}))$.
Proceeding in a similar way, we can express the second term in $\operatorname{MSPE}(\hat{\gamma}^{\operatorname{\isd}}_{t^*})$ as 
\begin{align*}
   \tfrac{\sigma^2_{\operatorname{ad}}}{m} \operatorname{trace}(\Pi_{\mathcal{S}^{\operatorname{res}}}\mathbb{E}[(\tfrac{1}{m}(\mathbf{X}^{\operatorname{ad}})^{\top}\mathbf{X}^{\operatorname{ad}})^{-1}]\Pi_{\mathcal{S}^{\operatorname{res}}}\operatorname{Var}(\Pi_{\mathcal{S}^{\operatorname{res}}}X_{t^*})) =   \tfrac{\sigma^2_{\operatorname{ad}}}{m} h_{\operatorname{res}}(m),
\end{align*}
where $h_{\operatorname{res}}(m)\coloneqq \operatorname{trace}(\Pi_{\mathcal{S}^{\operatorname{res}}}\mathbb{E}[(\tfrac{1}{m}(\mathbf{X}^{\operatorname{ad}})^{\top}\mathbf{X}^{\operatorname{ad}})^{-1}]\Pi_{\mathcal{S}^{\operatorname{res}}}\operatorname{Var}(\Pi_{\mathcal{S}^{\operatorname{res}}}X_{t^*}))$. Since we have assumed that for all $t\in\mathcal{I}^{\operatorname{ad}}\cup\{t^*\}$ the distribution of $X_t$ does not change, by Jensen's inequality and by the fact that $\Pi_{\mathcal{S}^{\operatorname{res}}}\Sigma_{t^*}^{-1}\Pi_{\mathcal{S}^{\operatorname{res}}}\Sigma_{t^*}\Pi_{\mathcal{S}^{\operatorname{res}}}=\Pi_{\mathcal{S}^{\operatorname{res}}}$ (see Lemma~\ref{lem:inv_res_partition} and \eqref{eq:pseudoinverse_identity}) we have that $h_{\operatorname{res}}(m)\ge \operatorname{dim}(\mathcal{S}^{\operatorname{res}})$, and $\lim_{m\rightarrow\infty}h_{\operatorname{res}}(m)= \operatorname{dim}(\mathcal{S}^{\operatorname{res}})$. 
Moreover, we can find an upper bound for $h_{\operatorname{res}}(m)$ in the following way. 
\begin{align*}
    h_{\operatorname{res}}(m) = & \operatorname{trace}(\Pi_{\mathcal{S}^{\operatorname{res}}}\mathbb{E}[(\tfrac{1}{m}(\mathbf{X}^{\operatorname{ad}})^{\top}\mathbf{X}^{\operatorname{ad}})^{-1}] \Pi_{\mathcal{S}^{\operatorname{res}}}\operatorname{Var}(\Pi_{\mathcal{S}^{\operatorname{res}}}X_{t^*})) \\%+ \operatorname{dim}(\mathcal{S}^{\operatorname{res}}) \\
    = & \operatorname{trace}(\mathbb{E}[(\tfrac{1}{m}(\mathbf{X}^{\operatorname{ad}})^{\top}\mathbf{X}^{\operatorname{ad}})^{-1}]\operatorname{Var}(\Pi_{\mathcal{S}^{\operatorname{res}}}X_{t^*})) \\
     \le & \operatorname{trace}(\mathbb{E}[(\tfrac{1}{m}(\mathbf{X}^{\operatorname{ad}})^{\top}\mathbf{X}^{\operatorname{ad}})^{-1}])\operatorname{trace}(\operatorname{Var}(\Pi_{\mathcal{S}^{\operatorname{res}}}X_{t^*}))\\
     \le & \|\mathbb{E}[(\tfrac{1}{m}(\mathbf{X}^{\operatorname{ad}})^{\top}\mathbf{X}^{\operatorname{ad}})^{-1}]\|_{\operatorname{op}}\operatorname{trace}(\operatorname{Var}(\Pi_{\mathcal{S}^{\operatorname{res}}}X_{t^*}))\\
     \le & \mathbb{E}[\|(\tfrac{1}{m}(\mathbf{X}^{\operatorname{ad}})^{\top}\mathbf{X}^{\operatorname{ad}})^{-1}\|_{\operatorname{op}}]\operatorname{trace}(\operatorname{Var}(\Pi_{\mathcal{S}^{\operatorname{res}}}X_{t^*}))\\
     \le & c^{-1}\lambda_{\max}(\Sigma_{t^*})\operatorname{dim}(\mathcal{S}^{\operatorname{res}}).
\end{align*}
Summarizing, we obtain that 
\begin{align*}
    \operatorname{MSPE}(\hat{\gamma}^{\operatorname{\isd}}_{t^*}) \ge \sigma_{\epsilon,\min}^2\tfrac{\operatorname{dim}(\mathcal{S}^{\operatorname{inv}})}{n}c_{\operatorname{inv}}  + \sigma^2_{\operatorname{ad}}\tfrac{\operatorname{dim}(\mathcal{S}^{\operatorname{res}})}{m}  
\end{align*}
and 
\begin{equation*}
    \operatorname{MSPE}(\hat{\gamma}^{\operatorname{\isd}}_{t^*}) \le \sigma_{\epsilon,\max}^2\tfrac{\operatorname{dim}(\mathcal{S}^{\operatorname{inv}}) }{n} C_{\operatorname{inv}} + \sigma_{\operatorname{ad}}^2\tfrac{\operatorname{dim}(\mathcal{S}^{\operatorname{res}}) }{m} C_{\operatorname{res}},
\end{equation*}
where $C_{\operatorname{res}}\coloneqq c^{-1}\lambda_{\max}(\Sigma_{t^*})$ is the constant introduced in the theorem statement.

We now compute the MSPE of $\hat{\gamma}_{t^*}^{\operatorname{OLS}}$.
\begin{align*}
         \operatorname{MSPE}(\hat{\gamma}_{t^*}^{\operatorname{OLS}}) & =  \mathbb{E}[(X_{t^*}^{\top}(\hat{\gamma}^{\operatorname{OLS}}_{t^*}-\gamma_{0,t^*}))^2] \\
         & = \operatorname{trace}(\mathbb{E}[(\hat{\gamma}^{\operatorname{OLS}}_{t^*}-\gamma_{0,t^*})(\hat{\gamma}^{\operatorname{OLS}}_{t^*}-\gamma_{0,t^*})^{\top}X_{t^*}X_{t^*}^{\top}]) \\
         & = \operatorname{trace}(\mathbb{E}[(\hat{\gamma}^{\operatorname{OLS}}_{t^*}-\gamma_{0,t^*})(\hat{\gamma}^{\operatorname{OLS}}_{t^*}-\gamma_{0,t^*})^{\top}]\operatorname{Var}(X_{t^*})) \\
         & = \operatorname{trace}(\mathbb{E}[\operatorname{Var}(\hat{\gamma}^{\operatorname{OLS}}_{t^*}\mid \mathbf{X}^{\operatorname{ad}})]\Sigma_{t^*})) \\
         & = \sigma^2_{\operatorname{ad}}\operatorname{trace}(\mathbb{E}[((\mathbf{X}^{\operatorname{ad}})^{\top}\mathbf{X}^{\operatorname{ad}})^{-1}]\Sigma_{t^*})) \\
         & = \tfrac{\sigma^2_{\operatorname{ad}}}{m} \operatorname{trace}(\mathbb{E}[(\tfrac{1}{m}(\mathbf{X}^{\operatorname{ad}})^{\top}\mathbf{X}^{\operatorname{ad}})^{-1}]\Sigma_{t^*})) .
    \end{align*}
    We can further express $\operatorname{MSPE}(\hat{\gamma}_{t^*}^{\operatorname{OLS}})$ as
        \begin{align*}
        & \tfrac{\sigma^2_{\operatorname{ad}}}{m} \operatorname{trace}(\mathbb{E}[(\tfrac{1}{m}(\mathbf{X}^{\operatorname{ad}})^{\top}\mathbf{X}^{\operatorname{ad}})^{-1}]\Sigma_{t^*}) \\
         = &  \tfrac{\sigma^2_{\operatorname{ad}}}{m} \operatorname{trace}((\Pi_{\mathcal{S}^{\operatorname{inv}}}+\Pi_{\mathcal{S}^{\operatorname{res}}})\mathbb{E}[(\tfrac{1}{m}(\mathbf{X}^{\operatorname{ad}})^{\top}\mathbf{X}^{\operatorname{ad}})^{-1}](\Pi_{\mathcal{S}^{\operatorname{inv}}}+\Pi_{\mathcal{S}^{\operatorname{res}}})\operatorname{Var}((\Pi_{\mathcal{S}^{\operatorname{inv}}}+\Pi_{\mathcal{S}^{\operatorname{res}}})X_{t^*})) \\
         = & \tfrac{\sigma^2_{\operatorname{ad}}}{m} \operatorname{trace}(\Pi_{\mathcal{S}^{\operatorname{inv}}}\mathbb{E}[(\tfrac{1}{m}(\mathbf{X}^{\operatorname{ad}})^{\top}\mathbf{X}^{\operatorname{ad}})^{-1}]\Pi_{\mathcal{S}^{\operatorname{inv}}}\operatorname{Var}(\Pi_{\mathcal{S}^{\operatorname{inv}}}X_{t^*})) \\
          & + \tfrac{\sigma^2_{\operatorname{ad}}}{m} \operatorname{trace}(\Pi_{\mathcal{S}^{\operatorname{res}}}\mathbb{E}[(\tfrac{1}{m}(\mathbf{X}^{\operatorname{ad}})^{\top}\mathbf{X}^{\operatorname{ad}})^{-1}]\Pi_{\mathcal{S}^{\operatorname{res}}}\operatorname{Var}(\Pi_{\mathcal{S}^{\operatorname{res}}}X_{t^*})).
        \end{align*}
   In particular, the second term in the above sum also appears in $\operatorname{MSPE}(\hat{\gamma}^{\operatorname{\isd}}_{t^*})$. Taking now the difference between the MSPE of $\hat{\gamma}^{\operatorname{OLS}}_{t^*}$ and $\hat{\gamma}^{\operatorname{\isd}}_{t^*}$, we obtain that 
    \begin{align*}
        \operatorname{MSPE}(\gamma^{\operatorname{OLS}}_{t^*})-\operatorname{MSPE}(\gamma^{\operatorname{\isd}}_{t^*}) & = \tfrac{\sigma^2_{\operatorname{ad}}}{m} \operatorname{trace}(\Pi_{\mathcal{S}^{\operatorname{inv}}}\mathbb{E}[(\tfrac{1}{m}(\mathbf{X}^{\operatorname{ad}})^{\top}\mathbf{X}^{\operatorname{ad}})^{-1}]\Pi_{\mathcal{S}^{\operatorname{inv}}}\operatorname{Var}(\Pi_{\mathcal{S}^{\operatorname{inv}}}X_{t^*}))\\
        & -  \operatorname{trace}(\mathbb{E}[\operatorname{Var}(\hat{\beta}^{\operatorname{inv}}\mid \mathbf{X})] \operatorname{Var}(\Pi_{\mathcal{S}^{\operatorname{inv}}}X_{t^*}))
    \end{align*}
We have already obtained an upper bound for the second term in the difference, namely $\sigma_{\epsilon,\max}^2\frac{ \operatorname{dim}(\mathcal{S}^{\operatorname{inv}}) }{n}C_{\operatorname{inv}}$. For the first term, we can make the same considerations made above for $\tfrac{\sigma^2_{\operatorname{ad}}}{m} \operatorname{trace}(\Pi_{\mathcal{S}^{\operatorname{res}}}\mathbb{E}[(\tfrac{1}{m}(\mathbf{X}^{\operatorname{ad}})^{\top}\mathbf{X}^{\operatorname{ad}})^{-1}]\Pi_{\mathcal{S}^{\operatorname{res}}}\operatorname{Var}(\Pi_{\mathcal{S}^{\operatorname{res}}}X_{t^*}))$ and $h_{\operatorname{res}}(m)$. In particular,
\begin{equation*}
    \tfrac{\sigma^2_{\operatorname{ad}}}{m} \operatorname{trace}(\Pi_{\mathcal{S}^{\operatorname{inv}}}\mathbb{E}[(\tfrac{1}{m}(\mathbf{X}^{\operatorname{ad}})^{\top}\mathbf{X}^{\operatorname{ad}})^{-1}]\Pi_{\mathcal{S}^{\operatorname{inv}}}\operatorname{Var}(\Pi_{\mathcal{S}^{\operatorname{inv}}}X_{t^*})) = \tfrac{\sigma_{\operatorname{ad}}^2}{m}h_{\operatorname{ad,inv}}(m)
\end{equation*}
where
\begin{equation*}
   h_{\operatorname{ad,inv}}(m)\coloneqq \operatorname{trace}(\Pi_{\mathcal{S}^{\operatorname{inv}}}\mathbb{E}[(\tfrac{1}{m}(\mathbf{X}^{\operatorname{ad}})^{\top}\mathbf{X}^{\operatorname{ad}})^{-1}]\Pi_{\mathcal{S}^{\operatorname{inv}}}\operatorname{Var}(\Pi_{\mathcal{S}^{\operatorname{inv}}}X_{t^*}))
\end{equation*}
 satisfies
 $h_{\operatorname{ad,inv}}(m)\ge \operatorname{dim}(\mathcal{S}^{\operatorname{inv}})$ and 
$h_{\operatorname{ad,inv}}(m) \le  c^{-1}\lambda_{\max}(\Sigma_{t^*})\operatorname{dim}(\mathcal{S}^{\operatorname{inv}})$ (as for $h_{\operatorname{res}}(m)$, we observe that $\lim_{m\rightarrow\infty}h_{\operatorname{ad,inv}}(m)=\operatorname{dim}(\mathcal{S}^{\operatorname{inv}})$).
Therefore, we obtain that 
\begin{equation*}
    \operatorname{MSPE}(\gamma^{\operatorname{OLS}}_{t^*})-\operatorname{MSPE}(\gamma^{\operatorname{\isd}}_{t^*}) \ge \sigma_{\operatorname{ad}}^2\tfrac{\operatorname{dim}(\mathcal{S}^{\operatorname{inv}})}{m} - \sigma_{\epsilon,\max}^2\tfrac{ \operatorname{dim}(\mathcal{S}^{\operatorname{inv}}) }{n}C_{\operatorname{inv}}
\end{equation*}
and 
\begin{equation*}
     \operatorname{MSPE}(\gamma^{\operatorname{OLS}}_{t^*})-\operatorname{MSPE}(\gamma^{\operatorname{\isd}}_{t^*}) \le \sigma_{\operatorname{ad}}^2\tfrac{\operatorname{dim}(\mathcal{S}^{\operatorname{inv}})}{m}C_{\operatorname{ad,inv}}
     - \sigma_{\epsilon,\min}^2\tfrac{ \operatorname{dim}(\mathcal{S}^{\operatorname{inv}}) }{n}c_{\operatorname{inv}},
\end{equation*}
where $C_{\operatorname{ad,inv}}\coloneqq c^{-1}\lambda_{\max}(\Sigma_{t^*})$.
In particular, this difference is always positive if $n$ is sufficiently large. Finally, we can observe that the expected explained variance \eqref{eq:expected_explained_variance} for $\hat{\gamma}$, where $\hat{\gamma}=\hat{\gamma}^{\operatorname{\isd}}_{t^*}$ or $\hat{\gamma}=\hat{\gamma}^{\operatorname{OLS}}_{t^*}$, is
    \begin{align*}
        & \mathbb{E}[\operatorname{\Delta Var}_{t^*}(\hat{\gamma})] \\ 
        &\quad= \mathbb{E}[\operatorname{Var}(Y_{t^*})-\operatorname{Var}(Y_{t^*}-X_{t^*}^{\top}\hat{\gamma}\mid\hat{\gamma})] \\
         &\quad = \operatorname{Var}(X_{t^*}^{\top}\gamma_{0,t^*}+\epsilon_{t^*})-\mathbb{E}[\operatorname{Var}(X_{t^*}^{\top}(\gamma_{0,t^*}-\hat{\gamma})+\epsilon_{t^*}\mid\hat{\gamma})] \\
         &\quad=  \gamma_{0,t^*}^{\top}\operatorname{Var}(X_{t^*})\gamma_{0,t^*} + \sigma_{\epsilon^*}^2 - \mathbb{E}[(\gamma_{0,t^*}  -\hat{\gamma})^{\top}\operatorname{Var}(X_{t^*})(\gamma_{0,t^*}-\hat{\gamma})] - \sigma_{\epsilon^*}^2\\
          &\quad=  \gamma_{0,t^*}^{\top}\operatorname{Var}(X_{t^*})\gamma_{0,t^*} - \mathbb{E}[\operatorname{trace}((\gamma_{0,t^*}-\hat{\gamma})(\gamma_{0,t^*}  -\hat{\gamma})^{\top}\operatorname{Var}(X_{t^*}))]\\
          &\quad= \gamma_{0,t^*}^{\top}\Sigma_{t^*}\gamma_{0,t^*} - \operatorname{MSPE}(\hat{\gamma})
    \end{align*}
    and therefore the same inequalities found for the MSPEs difference equivalently hold for $\mathbb{E}[\operatorname{\Delta Var}_{t^*}(\hat{\gamma}^{\operatorname{\isd}}_{t^*})]-\mathbb{E}[\operatorname{\Delta Var}_{t^*}(\hat{\gamma}^{\operatorname{OLS}}_{t^*})] $.
\end{proof}

\subsection{Proof of Proposition~\ref{prop:non_orthogonal_subspaces}}
\begin{proof}
    (i) By definition of a (non-orthogonal) joint block diagonalizer, it holds for all $t\in[n]$ that the matrix $\widetilde{\Sigma}_t\coloneqq U^{\top}\Sigma_t U$ is block diagonal with $q$ diagonal blocks $\widetilde{\Sigma}_{t,j}\coloneqq (U^{S_j})^{\top}\Sigma_t U^{S_j}$, $j\in\{1,\ldots,q\}$. Define now the matrix $W\coloneqq U^{-\top}$, and observe that the following relations hold 
\begin{align*}
    & \Sigma_t = W \widetilde{\Sigma}_t W^{\top}, \quad \Sigma_t^{-1} = U \widetilde{\Sigma}_t^{-1} U^{\top}\quad \text{and}\\
    & \widetilde{\Sigma}_t^{-1} = W^{\top}\Sigma_t^{-1} W \quad \text{block diagonal with blocks } \widetilde{\Sigma}_{t,j}^{-1}\coloneqq (W^{S_j})^{\top}\Sigma_t^{-1} W^{S_j}.
\end{align*}
The last relation is obtained by observing that
\begin{equation*}
            \tilde{\Sigma}_{t}^{-1} = \begin{bmatrix}
            \tilde{\Sigma}_{t,1}^{-1} & & \\
            & \ddots & \\
            & & \tilde{\Sigma}_{t,q}^{-1} 
            \end{bmatrix} = \begin{bmatrix}
                (W^{S_1})^{\top} \\ \vdots \\
                (W^{S_q})^{\top}
            \end{bmatrix}\Sigma_t^{-1} \begin{bmatrix}
                W^{S_1} & \dots &
                W^{S_q}
            \end{bmatrix}.
        \end{equation*}
Moreover, $\tilde{\Sigma}_{t,j}^{-1}= \operatorname{Var}((U^{S_j})^{\top}X_t)^{-1}$. We observe that, because $W^{\top}U=I_p$, it holds for all $i,j\in\{1,\dots,q\}$ that $(W^{S_j})^{\top}U^{S_i}=0$, that is, the space spanned by the columns of $W^{S_j}$ is orthogonal to the space spanned by the columns of $U^{S_i}$. Moreover, the matrix $U^{S_j}(W^{S_j})^{\top}$ is an oblique projection matrix onto $\mathcal{S}_j$ along $\bigoplus_{i\in\{1,\dots,q\}\setminus\{j\}} \mathcal{S}_i$. 
Fix $j\in\{1,\ldots,q\}$. Then, using these relations we obtain that
\begin{align*}
    P_{\mathcal{S}_j\mid \mathcal{S}_{-j}}\gamma_{0,t} & =  U^{S_j}(W^{S_j})^{\top}\Sigma_t^{-1}\operatorname{Cov}(X_t,Y_t) \\
         & =   U^{S_j}(W^{S_j})^{\top}\Sigma_t^{-1}WU^{\top}\operatorname{Cov}(X_t,Y_t) \\
        & = U^{S_j}(W^{S_j})^{\top}\Sigma_t^{-1} \begin{bmatrix}
           W^{S_1} & \dots &  W^{S_q}
        \end{bmatrix} \begin{bmatrix}
           (U^{S_1})^{\top} \\ \vdots \\ (U^{S_q})^{\top}
        \end{bmatrix} \operatorname{Cov}(X_t,Y_t) \\
        & = U^{S_j}(W^{S_j})^{\top}\Sigma_t^{-1} 
           W^{S_j} 
           (U^{S_j})^{\top} \operatorname{Cov}(X_t,Y_t) \\
        & =  U^{S_j}\tilde{\Sigma}_{t,j}^{-1}\operatorname{Cov}((U^{S_j})^{\top}X_t,Y_t) \\
        & =   U^{S_j}  \operatorname{Var}((U^{S_j})^{\top}X_t)^{-1}\operatorname{Cov}((U^{S_j})^{\top}X_t,Y_t).
\end{align*}
The fourth equality follows from the block diagonal structure of $ W^{\top}\Sigma_t^{-1} W$, which implies that, for all $j\neq1$, $(W^{S_1})^{\top}\Sigma_t^{-1} W^{S_j}=0$.

(ii) By definition of $\beta^{\operatorname{inv}}$ and using point (i) of this proposition and that $\{\mathcal{S}^{\operatorname{inv}}, \mathcal{S}^{\operatorname{res}}\}$ form a $(X_t)_{t\in[n]}$-decorrelating partition, we obtain that, for all $t\in[n]$, 
\begin{equation*}
    \beta^{\operatorname{inv}} = U^{\operatorname{inv}} \operatorname{Var}((U^{S^{\operatorname{inv}}})^{\top}X_t)^{-1}(U^{\operatorname{inv}})^{\top}\operatorname{Cov}(X_t, Y_t).
\end{equation*}
We now observe that $\tilde{\beta}^{\operatorname{inv}}\coloneqq \operatorname{Var}((U^{S^{\operatorname{inv}}})^{\top}X_t)^{-1} (U^{\operatorname{inv}})^{\top}\operatorname{Cov}(X_t, Y_t)$ is the OLS solution of regressing $Y_t$ onto $(U^{\operatorname{inv}})^{\top}X_t$ (as we assume all variables have zero mean). Moreover, this quantity is constant by proj-invariance of $\mathcal{S}^{\operatorname{inv}}$. In particular, this means that $\tilde{\beta}^{\operatorname{inv}} = \argmin_{\beta\in\mathbb{R}^{\operatorname{dim}(\mathcal{S}^{\operatorname{inv}})}}\mathbb{E}[\frac{1}{n}\sum_{t=1}^n (Y_t - X_t^{\top}U^{\operatorname{inv}}\beta)^2]$, which implies that 
\begin{equation*}
    \tilde{\beta}^{\operatorname{inv}} = ((U^{\operatorname{inv}})^{\top}\overline{\operatorname{Var}}(X)U^{\operatorname{inv}})^{-1}(U^{\operatorname{inv}})^{\top}\overline{\operatorname{Cov}}(X, Y).
\end{equation*}
Therefore, $ \beta^{\operatorname{inv}} = U^{\operatorname{inv}} ((U^{\operatorname{inv}})^{\top}\overline{\operatorname{Var}}(X)U^{\operatorname{inv}})^{-1}(U^{\operatorname{inv}})^{\top}\overline{\operatorname{Cov}}(X, Y)$.

(iii) The claim follows from the definition of $\delta^{\operatorname{res}}_t$, from the fact that $\{\mathcal{S}^{\operatorname{inv}}, \mathcal{S}^{\operatorname{res}}\}$ forms a $(X_t)_{t\in[n]}$-decorrelating partition and from point (i) of this proposition.
\end{proof}

\end{document}